\documentclass{article}
\usepackage{graphicx} %
\usepackage{amsmath, amsthm, mathtools, dsfont,amssymb}
\usepackage{hyperref}
\usepackage{tcolorbox}
\usepackage{amsfonts}
\usepackage{csquotes}
\usepackage[margin=1in]{geometry}
\usepackage[linesnumbered, ruled, vlined, noend]{algorithm2e}
\usepackage[parfill]{parskip}
\usepackage{nicefrac}
\usepackage{cleveref}
\usepackage{booktabs}

\newcommand{\wtilde}{\widetilde}

\newcommand{\norm}[1]{\left\lVert #1 \right\rVert}

\newcommand{\EXP}{\mathbb{E}}
\newcommand{\PROB}{\textnormal{Pr}}
\newcommand{\IND}{\mathds{1}}
\renewcommand{\Pr}[1]{{\PROB \sbr{#1}}}
\newcommand{\Pru}[2]{\ensuremath{\PROB_{#1}\left[#2\right]}}
\newcommand{\Ex}[1]{{\EXP \sbr{#1}}}
\newcommand{\Exu}[2]{\ensuremath{\mathbb{E}_{#1}\left[#2\right]}}
\newcommand{\ind}[1]{{\IND \left\{ #1 \right\}}}
\newcommand{\eps}{\varepsilon}

\newcommand{\R}{\mathbb{R}}

\newcommand{\cA}{\mathcal{A}}
\newcommand{\cB}{\mathcal{B}}
\newcommand{\cC}{\mathcal{C}}
\newcommand{\cD}{\mathcal{D}}
\newcommand{\cE}{\mathcal{E}}
\newcommand{\cF}{\mathcal{F}}
\newcommand{\cG}{\mathcal{G}}
\newcommand{\cH}{\mathcal{H}}

\newcommand{\cO}{\mathcal{O}}

\newcommand{\cT}{\mathcal{T}}
\newcommand{\cP}{\mathcal{P}}
\newcommand{\cW}{\mathcal{W}}
\newcommand{\cX}{\mathcal{X}}
\newcommand{\cY}{\mathcal{Y}}
\newcommand{\cZ}{\mathcal{Z}}

\newcommand{\cHhat}{\widehat{\cH}}

\newcommand{\poly}{\operatorname{poly}} 
\newcommand{\polylog}{\operatorname{polylog}}

\newcommand{\rbr}[1]{\left(#1\right)}
\newcommand{\abs}[1]{\left|#1\right|}
\newcommand{\sbr}[1]{\left[#1\right]}
\newcommand{\cbr}[1]{\left\{#1\right\}}

\newtheorem{theorem}{Theorem}
\newtheorem{corollary}{Corollary}
\newtheorem{lemma}{Lemma}

\newtheorem{claim}{Claim}

\newtheorem{definition}{Definition}
\newtheorem{remark}{Remark}

\crefname{lemma}{Lemma}{Lemmas}
\Crefname{lemma}{Lemma}{Lemmas}
\crefname{corollary}{Corollary}{Corollaries}
\Crefname{corollary}{Corollary}{Corollaries}
\crefname{fact}{Fact}{Facts}
\Crefname{fact}{Fact}{Facts}
\crefname{observation}{Observation}{Observations}
\Crefname{observation}{Observation}{Observations}
\crefname{claim}{Claim}{Claims}
\Crefname{claim}{Claim}{Claims}
\crefname{proposition}{Proposition}{Propositions}
\Crefname{proposition}{Proposition}{Propositions}
\crefname{definition}{Definition}{Definitions}
\Crefname{definition}{Definition}{Definitions}
\crefname{remark}{Remark}{Remarks}
\Crefname{remark}{Remark}{Remarks}

\newcommand{\EqComment}[1]{\text{\emph{(#1)}}}
\newcommand{\term}{\texttt}

\newcommand{\cTcomp}{\ensuremath{\cT^*}}
\newcommand{\cAcomp}{\ensuremath{\cA^*}}
\newcommand{\cBcomp}{\ensuremath{\cB^*}}
\newcommand{\cYcomp}{\ensuremath{\cY^*}}
\newcommand{\cDcomp}{\ensuremath{\cD^*}}

\newcommand{\dcomp}{\ensuremath{d'}}
\newcommand{\Sim}{\term{sim}}

\newcommand{\err}{\term{err}}

\title{Replicable Composition}
\author{
Kiarash Banihashem\thanks{University of Maryland, College Park, MD, USA. \texttt{kiarash@umd.edu}}
\and
MohammadHossein Bateni\thanks{Google Research, New York City, NY, USA. \texttt{bateni@google.com}}
\and
Hossein Esfandiari\thanks{Google Research, London, UK. \texttt{esfandiari@google.com}}
\and
Samira Goudarzi\thanks{University of Maryland, College Park, MD, USA. \texttt{samirag@umd.edu}}
\and
MohammadTaghi Hajiaghayi\thanks{University of Maryland, College Park, MD, USA. \texttt{hajiagha@umd.edu}}
}
\date{}

\begin{document}\sloppy

\maketitle

\begin{abstract}
Replicability, the requirement that algorithmic conclusions remain consistent when run on independently drawn stochastic data, is a cornerstone of reliable data analysis and has become a central topic in theoretical computer science. A key structural question in this context is \emph{composition}. Specifically, given $k$ problems each admitting a $\rho$-replicable algorithm with sample complexity $n$, how many samples are needed to solve all of them jointly while maintaining replicability? The naive composition analysis yields $\widetilde{O}(nk^2)$ samples, and Bun et al.~(STOC’23) further observed that reductions through differential privacy imply an alternative $\widetilde{O}(n^2k)$ bound, leaving open whether the optimal $\widetilde{O}(nk)$ scaling could be achieved. In this work, we resolve this open problem and, more generally, show that problems with sample complexities $n_1,\ldots,n_k$ can be jointly solved with $\widetilde{O}(\sum_i n_i)$ samples while preserving constant replicability.

Our approach converts each replicable algorithm into a perfectly generalizing one, composes them in that space using a privacy-style analysis, and then maps the result back to a replicable algorithm via correlated sampling. The choice of perfectly generalizing algorithms and suitable parameters is crucial for obtaining a tight result. This yields the first \emph{advanced composition theorem for replicability}, achieving a nearly linear dependence on $n$ and $k$, which is approximately tight. As a direct application, we show that existing ad hoc multi-problem analyses from prior work can be unified under a single general framework. En route, we obtain new bounds for the composition of perfectly generalizing algorithms: the composition of $k$ such algorithms with parameters $(\varepsilon_i,\delta_i)$ is perfectly generalizing with parameters $(\varepsilon^*,\delta^*)$, where $\varepsilon^* \approx \sqrt{\sum_i \varepsilon_i^2}$ and $\delta^*$ can be made arbitrarily small at only a polylogarithmic cost in sample complexity.

As part of our study, we provide a new boosting theorem for increasing the success probability of a replicable algorithm which may be of independent interest. For a broad class of problems, we show that a low failure probability can be achieved by splitting the dataset into a replicable and a non-replicable part, with only the latter requiring low failure probability.
This makes the failure probability appear in a separate \emph{additive} term independent of $\rho$, immediately yielding improved bounds for several problems studied in the literature, including statistical query estimation, various PAC learning settings (such as finite classes, thresholds, and bounded Littlestone dimension), and the heavy-hitters problem.

Finally, we ask whether the linear scaling of our composition theorem extends to adaptive adversaries who can base the problems they pose, as well as the distribution of those problems, on previous answers.
We show that it does not: any algorithm solving $k$ problems in the adaptive setting requires $\Omega(nk^2)$ samples, establishing a quadratic separation between the adaptive and non-adaptive settings.
We use the key technique behind our lower bound, which we refer to as the
\emph{phantom run}, 
to show structural results for replicable algorithms which are of independent interest
\end{abstract}

\newpage
\section{Introduction}

\emph{Replicability} aims to ensure that algorithmic conclusions are stable under dataset resampling, which is central to the credibility of empirical science. Informally, an algorithm is replicable if, when rerun on a fresh i.i.d. sample from the same distribution but with the same internal random bits, it produces the same output with high probability (see \Cref{def:replic}).
The notion was introduced by Impagliazzo, Lei, Pitassi, and Sorrell~\cite{impagliazzo2022reproducibility} 
and has since motivated a growing line of research on stability in statistical estimation, learning theory, and data analysis.
Subsequent papers have explored its relationships with other notions of stability and applied the notion to different statistical problems~\cite{
bun2023stability,
ReplicFocs0001MY23,
LocalCMY24,
hopkins2024replicability,
banihashem2026replicablemedian,
aamand2025structure,
larsen2026sample,
blondal2025borsuk}.
Together, these results position replicability as a unifying stability concept that connects statistical reliability with algorithmic design.

Differential Privacy (DP) is a framework for protecting individual privacy, but it also implies a strong form of algorithmic stability. In DP, stability arises as a consequence of its indistinguishability guarantee—outputs remain nearly unchanged under small perturbations of the input dataset. In replicability, by contrast, stability is captured through invariance of outputs under resampling from the same distribution. This parallel motivates an analogous question for replicability that mirrors a central theme in DP: \emph{composition}. Suppose there are $k$ statistical problems, each solvable by a $\rho$-replicable algorithm using $n$ samples. What is the sample complexity required to solve all $k$ problems jointly while preserving replicability of the combined output?

\begin{center}
\begin{tcolorbox}[colback=gray!5,colframe=black!20]
\textbf{Open Question (Bun et
al., STOC'23~\cite{bun2023stability}).}
Given $k$ problems, each solvable with $n$ samples by a $\rho$-replicable algorithm, can one design a $\rho$-replicable algorithm that solves all $k$ problems simultaneously with sample complexity $\wtilde{O}(nk)$?
\end{tcolorbox}
\end{center}

The authors asked this question for both adaptive and non-adaptive composition.
We answer this open question in the affirmative for non-adaptive composition, and show that the picture changes fundamentally when adaptivity is allowed.
Specifically, for heterogeneous sample complexities $n_1,\ldots,n_k$, we give a non-adaptive composition scheme that is $\rho$-replicable and uses $\wtilde{O}\big(\sum_{i=1}^k n_i\big)$ samples, matching the natural linear scaling in $k$.
We believe this result is of fundamental importance: composition is a core structural question for any stability notion, and our theorem provides a simple, general principle that can be applied across problems.
Indeed, we show that existing problem-specific analyses from prior work can be replaced by our general framework, turning ad hoc multi-problem arguments into a standard, unified approach.

For adaptive composition, the situation is fundamentally different: we prove an $\Omega(nk^2)$ lower bound, showing that one cannot improve on naive composition.
The key distinction is how replicability errors --- the events where two runs on fresh samples diverge --- interact across algorithms.
In the adaptive setting, each algorithm runs without knowledge of future algorithms and must commit to its output before subsequent ones run, forcing the replicability errors of individual algorithms to behave independently; independent errors compound.
In the non-adaptive setting, algorithms run in coordination, allowing them to correlate their replicability errors rather than accumulate them independently --- which is precisely what enables the $\wtilde{O}(\sum_i n_i)$ bound.

\paragraph{Connection to differential privacy.}
In differential privacy, composition quantifies how privacy loss accumulates
when multiple analyses are performed using private algorithms. The naive
composition bound scales linearly in the number of analyses, but the
\emph{advanced composition theorem}~\cite{dwork2010boosting} (FOCS'10) shows a
sublinear $\sqrt{k}$ dependence—a property that is widely used in DP. This
shift serves as the guiding analogy for our work: we seek the replicability
counterpart of advanced composition. In replicability, if one simply runs the
$k$ algorithms together, the replicability parameter degrades linearly as
$O(\rho k)$, and boosting this parameter to maintain constant replicability
leads to a total sample complexity of $\wtilde{O}(nk^2)$ (see
\Cref{app:naive_is_tight}). Bun et al.~\cite{bun2023stability} further pointed
out that through reductions to DP, one can obtain an alternative bound of
$\wtilde{O}(n^2k)$, and left open whether the optimal $\wtilde{O}(nk)$ scaling
could be achieved. Our non-adaptive composition theorem resolves this question.
In this sense, our theorem can be viewed
as an \emph{advanced composition theorem for replicability} that also uncovers
a fundamental structural distinction from differential privacy.
Notably, however, the adaptive lower bound reveals a separation with no
counterpart in DP: while advanced composition in DP applies regardless of
adaptivity, in replicability the $\wtilde{O}(nk)$ bound is provably
unachievable in the adaptive setting.

\paragraph{Structural properties of replicable algorithms.}
The proof of our adaptive lower bound relies on a technique we call the \emph{phantom run}: a third independent execution of the algorithm used as an intermediate between the two standard runs.
Beyond the lower bound, this technique has broader applicability: any replicable algorithm can be converted into one whose output depends on the input only through a sufficient statistic of the distributional family (\Cref{sec:suff_stat}), with replicability preserved for distributions in $\cF$; this implies in particular that any replicable algorithm can be made \emph{order-invariant} (its output is invariant to the ordering of the input sample).
Extending our techniques, we further show that for symmetric problems, any replicable algorithm can be made \emph{label-invariant}, meaning its output distribution is invariant to relabelings of the domain.

These properties first became of interest to the community through the lower
bound of \cite{liu2024replicable}, who studied uniformity testing, proved lower bounds for label-invariant
testers, and asked whether the restriction to this class was
necessary.
\cite{aamand2025structure} established that order-invariance and label invariance are without loss of generality in more restricted settings, confined to binary hypothesis testing; they additionally showed that for hypothesis testing, any replicable algorithm can be converted to one that depends on a sufficient statistic, with replicability guaranteed for distributions within the parametric family.
Our approach establishes label invariance for any symmetric problem; our sufficient statistic result holds likewise, for any problem.

\paragraph{Boosting success probability.}
To complement our composition theorem, we further develop a framework for
boosting the success probability of replicable algorithms, which may be of
independent interest (see the discussion surrounding \Cref{thm:main} for the
motivation behind studying this question).
For a broad class of algorithms, we show that the success probability can be boosted at a very small cost: in most cases the dependence on the failure probability is no worse than the corresponding non-replicable version of the problem and appears as a separate additive term independent of $\rho$. Our framework immediately improves the sample complexity of several replicable algorithms obtained in prior work.

To summarize, our contributions are as follows:
\begin{itemize}
    \item We develop an \emph{advanced composition} theorem for replicability,
        showing the sample complexity of solving $k$ statistical problems with
        sample complexities $n_1, \dots, n_k$ simultaneously is at most $\wtilde{O}(\sum_{i} n_i)$.
        As part of our approach, we obtain new bounds for the adaptive composition of perfectly
        generalizing algorithms with heterogeneous parameters.
        As an immediate consequence, our framework gives a simpler and unified analysis
        for several results in the literature previously obtained via ad hoc techniques.
    \item We prove a lower bound showing that \emph{adaptive} composition of 
        replicable algorithms requires $\Omega(nk^2)$ samples, establishing a
        quadratic separation from the $\wtilde{O}(nk)$ upper bound of non-adaptive composition.
        The \emph{phantom run} technique underlying the proof is of independent interest:
        we use it to show that any replicable algorithm can be converted, with at most a factor-of-$2$ loss in
        replicability and no change in sample complexity or accuracy, to one that depends on the input only through a
        sufficient statistic of the distributional family
        (with replicability guaranteed for distributions within the family).
        When the problem is \emph{symmetric} (see \Cref{def:symmetric_problem}), we further
        show that the algorithm can be made label-invariant (see \Cref{def:label_inv}).
    \item We obtain the first boosting theorem for increasing the success probability of replicable algorithms, directly improving the sample complexity of several problems from prior work including
        statistical query estimation, various PAC learning settings (such as finite classes, thresholds, and bounded Littlestone dimension), and the heavy-hitters problem.
\end{itemize}

\section{Our Results}
\subsection{Replicable Composition}
\label{sec:intro_compose}

Our main result is the following theorem which bounds the sample complexity of the composition of $k$ statistical problems; i.e., solving all problems simultaneously (see \Cref{sec:prelim} for a formal definition of statistical problem and composition). 
\begin{theorem}[Main Theorem] \label{thm:main}
     Suppose algorithms $\cA_1, \dots, \cA_k$ are each $0.0001$-replicable and solve statistical problems $\cT_1, \dots, \cT_k$, respectively, where $\cA_i$ has sample complexity $n_i$, failure probability at most $\beta_0$, and its output is supported over a finite space. Then, for any $\rho > 0$, there exists a $\rho$-replicable algorithm that solves the composed problem $\cTcomp = (\cT_1, \dots, \cT_k)$ with failure probability at most $\beta$ and sample complexity $n$, where
     \begin{align}
         n =
        O\left(
        \frac{\sum_{j} n_j}{\rho^2} \polylog\left(\frac{k}{\rho \beta_0}\right)
        \right),
        \qquad
        \beta =
        O\left(k\sqrt{\beta_0}\log\frac{k}{\rho \beta_0}\right).
        \label{eq:bound_n_main_thm}
     \end{align}
\end{theorem}

This theorem matches the natural linear target in $k$ and resolves the open question of Bun et al.~\cite{bun2023stability}.
We note that linear dependence on $n$ and $k$ is essentially necessary and one cannot, e.g., obtain $O(n^{0.99} k)$ or $O(nk^{0.99})$ dependence; see \Cref{sec:tight} for a full proof.
We note that the assumption that the output of $\cA_i$ is supported over a finite space is mild as, in most cases, one can discretize the space and round the output to the closest answer to ensure the guarantee (see \Cref{sec:prelim} for a more detailed discussion).

The increase in failure probability is relatively mild and broadly consistent with what occurs in the context of differential privacy: when composing $k$ differentially private algorithms—each with failure probability $\beta_0$—the overall failure probability increases to $O(k\beta_0)$ under independence. In our setting, the bound in \Cref{thm:main} is slightly worse, scaling as $O(k\sqrt{\beta_0}\log(\cdot))$, but this is still a gentle growth in most regimes.
The increase motivates a natural question however: 
how does the sample complexity scale with the success probability. 
We study this question in \Cref{sec:intro_boost}, providing a general framework for boosting the success probability that applies for a broad class of problems. In particular, for many problems we show a logarithmic dependence on $1/\beta_0$.
This in turn means that boosting $\beta_0$ by replacing it with $(\beta_0/k)^2$ causes only a logarithmic increase in sample complexity. 

\paragraph{Why naive composition fails.}
A natural starting point is to simply run all $k$ algorithms in parallel with shared randomness. 
Since each component is $\rho$-replicable, a naive union bound implies a joint replicability parameter grows linearly as $O(\rho k)$.
In general, this bound is tight up to constants.
Intuitively, if the algorithms are solving \enquote{independent} problems, then their replicability errors are also independent and as such, the probability that at least one algorithm changes its output under resampling scales with $\Omega(k\rho)$ (see \Cref{app:naive_is_tight} for a formal proof).
This behavior is in sharp contrast to differential privacy, where advanced composition follows from a refined \emph{analysis} of the same combined algorithm.  
For replicability, no such analysis improvement is possible—the naive construction is inherently limited—so we must instead design a new composition scheme.

\paragraph{Our approach.}
Our construction proceeds in three steps.  
First, we convert each replicable algorithm $\cA_i$ into a \emph{perfectly generalizing} (PG) algorithm $\cB_i$ with carefully chosen parameters $(\eps_i,\delta_i)$.  
Second, we compose the PG algorithms.  
Finally, we transform the result back into a replicable algorithm using correlated sampling.  
An algorithm $\cA$ is $(\gamma, \eps,\delta)$-PG if for every distribution $D$ there exists a \enquote{simulated} distribution $\Sim_D$ such that,
with probability $1-\gamma$ 
over the sample $S\sim D^n$,
the distribution of $\cA(S)$ satisfies
the following
for all measurable $O$:
\begin{align*}
    e^{-\eps}(\Pru{\Sim_D}{O}-\delta)
    \le \Pr{\cA(S)\in O}
    \le e^{\eps}\Pru{\Sim_D}{O}+\delta.
\end{align*}
This notion parallels differential privacy: both constrain multiplicative deviations in output probabilities with a possibility for a small additive error.
However, while differential privacy guarantees stability under pointwise perturbations of the dataset, perfect generalization ensures stability under resampling from the underlying distribution.

\paragraph{Why perfect generalization is the right bridge.}
While the idea of transforming to an intermediate stability notion is crucial, the choice of this notion is even more so
as other natural alternatives are not suitable.
Differential privacy, though powerful, is too strong: converting between DP and replicability introduces a quadratic blow-up in sample complexity ($n\mapsto n^2$).  
Intuitively, DP constrains log-probability ratios for \emph{neighboring datasets}, while replicability concerns resampling under shared randomness—bridging the two loses a quadratic factor.  
Moreover, in general there is no reverse improvement from $n$ to $\sqrt{n}$ when going from replicability to DP, so DP cannot yield an optimal composition result.

A seemingly closer notion is \emph{TV-indistinguishability}.  
In a sense, it captures replicability without requiring the outputs to be synchronized: it demands that for two independent samples $S^{(1)},S^{(2)}$, the expected total-variation distance between $\cA(S^{(1)})$ and $\cA(S^{(2)})$ be at most $\rho$.  
This notion can be thought of as a special case of perfect generalization with $\eps = 0$ and is very close to replicability; indeed, one can \enquote{sync} its outputs through correlated sampling. 
However, it is unclear how to bound the accumulation of errors of composed TV-indistinguishable algorithms beyond the naive analysis and as such we need to use a different approach. 

Perfect generalization occupies a well-defined position precisely between these two extremes.
Like differential privacy, it constrains the \emph{log-probability ratio} of
output events, thereby permitting a concentration-based analysis akin to that
used in the advanced composition theorem for DP: instead of a standard linear
scaling of $\sum_{i} \eps_i$ we obtain an improved scaling of $\sqrt{\sum_{i}
\eps_i^2}$ for the multiplicative factor (see \Cref{thm:pg_compose_intro}). Importantly, the existing
transformation from replicability to perfect generalization allows us to choose
the parameter $\delta_i$ to be very small at only a polylogarithmic cost in the
sample complexity.
Since the parameters $\delta_i$ add up, this choice ensures that even cumulatively, they do not have a large effect.
Concretely, setting $\eps_i \approx \rho\sqrt{n_i/\sum_j n_j}$ (see \Cref{eq:def_eps_i_delta_i}) ensures $\sum_i \eps_i^2 \approx \rho^2$, and by \Cref{lm:replic_to_pg}, the cost of converting $\cA_i$ to a PG algorithm scales as $\wtilde{O}(n_i/\eps_i^2) \approx \wtilde{O}(\sum_j n_j / \rho^2)$, yielding the linear dependence in \Cref{thm:main}.

\paragraph{Composition of perfectly generalizing algorithms.}
As part of our approach, we generalize existing results for the composition of perfectly generalizing algorithms.
We state an informal version of our result below and refer to \Cref{lm:pg_compose} for more details. On a high level, ignoring $\delta_i$, our result shows that the composition of $k$ algorithms with parameters $\eps_i$ is also perfectly generalizing with parameter $\sqrt{\sum_{i} \eps_i^2}$.

\begin{theorem}[Informal composition theorem]\label{thm:pg_compose_intro}
    Assume each algorithm $\cA_i$ is $(\delta_i,\eps_i,\delta_i)$-PG with small enough $\eps_i$ and $\delta_i \ll \eps_i$.   
    Then for any $\delta' \in (0, 1/2)$, their composition is $(\delta^*,\eps^*,\delta^*)$-PG, where
    \begin{align*}
        \eps^* = O\Big(\sqrt{\log(1/\delta') \sum_i \eps_i^2} + \sum_i \eps_i^2\Big), \quad
        \delta^* = O\Big(\sqrt{\frac{1}{\eps^*}\Big(k\delta' + \sum_i \frac{\delta_i}{\eps_i}\Big)}\Big).
    \end{align*}
\end{theorem}
 
The above theorem also holds in the adaptive setting; see \Cref{sec:compose} for a formal definition of the setting and \Cref{lm:pg_compose_het} for the precise statement.
Our proof builds on the existing result of \cite{bassily2016typical} who study the more general concept of typical stability but focus on the special case of homogeneous parameters.\footnote{
    The proof of \cite{bassily2016typical} has a subtle issue in the adaptive setting---specifically, certain objects are named in a way that is inconsistent with what the proof requires, making the argument difficult to verify as written.
    We identify and correct this issue in \Cref{sec:pg_error_explain} and prove our extension to heterogeneous parameters in \Cref{lm:pg_compose} and \Cref{lm:pg_compose_het}.

    We note that the existence of an adaptive result is perhaps surprising: the closely related concept of max-information under product distributions does not compose adaptively (see Theorem~4.1 in \cite{RRST16}), and conditioning on the output of a previous algorithm means the input distribution is no longer i.i.d., making it a priori unclear whether PG composition can handle adaptivity.
    A high-level sketch of how this conditioning issue is sidestepped appears after \Cref{lm:pg_compose_het}.

    We emphasize that the non-adaptive version of \Cref{thm:pg_compose_intro}
    already suffices for \Cref{thm:main}; the adaptive result is included
    because it is a natural question one may ask about PG composition, and---as
    the above discussion suggests---the answer is not obvious.
    While we do not pursue this here, one may obtain a simpler proof for the non-adaptive case.
    We also note that the adaptive PG composition theorem does not make \Cref{thm:main} adaptive in any way: even though the PG algorithms themselves can be adaptively composed, the transformation back to a replicable algorithm requires knowledge of all algorithms up-front.
}
On a high level, the main intuition behind the proof is similar to that of advanced composition in DP.
Specifically, when composing algorithms with parameters $(\eps_i,\delta_i)$, we look at the logarithm of the probability ratio of a value under the algorithm compared to the simulator. 
The expected contribution of each algorithm to the overall instability scales as $\eps_i^2$, rather than linearly in $\eps_i$. While the \enquote{worst-case} contribution is still $\eps_i$, using concentration one can show that
the cumulative $\eps$ parameter grows as $\sqrt{\sum_i \eps_i^2}$---exactly paralleling the behavior established by advanced composition in differential privacy.
Two additional complications arise. First, the PG condition only guarantees closeness with probability $1-\gamma_i$, so the bad events where the bound fails must be tracked carefully as they accumulate across rounds. Second, in the adaptive case, conditioning on previous outputs means the distribution of samples is no longer i.i.d., introducing dependencies that require a more involved inductive argument.
Concretely, we first show that (with high probability) the values generated by the simulator satisfy the desired probability ratio bounds; since the simulator is independent of the input, it does not suffer from the same i.i.d.\ issue.
We then use the fact that the simulator and the algorithm are close to each
other---a fact that follows inductively from the argument applied to the first
$i-1$ rounds---to show that the values of the algorithm satisfy the probability ratio bounds as well.
We refer to the discussion following \Cref{lm:pg_compose_het} for a more detailed overview with the key definitions and notation.

\paragraph{Applications.}
Beyond the general composition theorem, our framework yields immediate applications to canonical problems studied in prior work. Specifically, we can re-obtain existing results from the literature using a simplified and unified analysis.
We briefly discuss these applications here and refer to \Cref{sec:app_sq} and \Cref{sec:app_best_arm} for a more formal treatment.

Our first application concerns the problem of estimating \emph{multiple statistical queries} on the same distribution.
Given $k$ functions $\phi_1, \dots, \phi_k : \cX \to [0,1]$, the goal is to (simultaneously) estimate the mean values $\mu_i = \Exu{X\sim D}{\phi_i(X)}$ with additive error $\alpha$ while ensuring replicability.
A canonical example is the $k$-coin problem, where we aim to estimate the mean of $k$ Bernoulli random variables simultaneously.\footnote{Prior work consider $N$ coins and use the term \emph{$N$-coin problem} instead. For consistency with \Cref{thm:main}, we use $k$ instead.} A naive solution estimates each query separately with a smaller replicability parameter $\rho/k$, but this leads to quadratic dependence on $k$.
Prior works, such as those of \cite{karbasi2023replicability, hopkins2024replicability}, explicitly mention this limitation and overcame this using problem-specific arguments.
Specifically, \cite{karbasi2023replicability} study multiple statistical queries and use ideas inspired from the Gaussian mechanism in DP to obtain a TV-indistinguishable algorithm. They then transform this to a replicable algorithm using correlated sampling.
\cite{hopkins2024replicability} further study the $k$-coin problem, and more generally high dimensional mean estimation,  and establish equivalence results between replicability and isoperimetric tilings in this context. They then use the existence of such a tiling to obtain a mean estimation algorithm with $\ell_{\infty}$ error.

In contrast to these works, our composition theorem provides a black-box reduction.
Specifically, we set $\cT_i$ to be the single-query estimation problem, which has previously been studied by \cite{impagliazzo2022reproducibility}. Invoking \Cref{thm:main} directly yields an algorithm with (nearly) linear dependence on $k$. We refer to \Cref{thm:multi_dim_mean} and its proof for more details.

Our second application is the \emph{parallel best-arm problem}, introduced in the context of replicable reinforcement learning by \cite{hopkins2025generative}.
In this setting, we are given $k$ independent instances of the best-arm identification problem, each with $|A|$ reward distributions (arms), and the goal is to find—replicably—an approximately optimal arm in each instance.
In the single-instance case, \cite{hopkins2025generative} solve the best-arm problem by combining correlated sampling with an approach inspired by the exponential mechanism from differential privacy. 
However, extending this to the parallel case poses a key challenge: \cite{hopkins2025generative} explicitly note that a straightforward composition of $k$ replicable algorithms would fail, leading to quadratic scaling in $k$, and thus design a more elaborate joint exponential mechanism that samples directly over vectors of actions $(a_1, \dots, a_k)$.
Using our composition theorem, this difficulty can be bypassed entirely in a black-box manner.

By setting each $\cT_i$ to be a single-instance best-arm problem, \Cref{thm:main} immediately provides a $\rho$-replicable algorithm for all $k$ instances with nearly linear dependence on $k$.
We refer to \Cref{sec:app_best_arm} and its proof for more details.

\subsection{Lower Bound for Adaptive Composition}
\label{sec:intro_lb}

While \Cref{thm:main} shows that non-adaptive composition can be achieved with nearly linear sample complexity,
a natural follow-up question is whether this guarantee extends to the \emph{adaptive} setting,
where the problems posed in later rounds may depend on the outputs of earlier ones and may be based on new distributions.
We show that the answer is no: adaptive composition necessarily incurs a quadratic cost in $k$.
Formally, we consider an adaptive game (\Cref{def:adaptive_game})
similar to that of \cite{dwork2010boosting} (see Section III.A in their paper)
in which a
random adversary poses $k$ statistical problems sequentially, each potentially
depending on previous outputs and potentially on a new distribution.
The algorithm must solve each problem replicably using fresh samples at every round.

\begin{theorem}[Informal; see \Cref{thm:adaptive_lb}]
    Any algorithm that is $\rho$-replicable in the adaptive $k$-round game and answers all rounds correctly with constant probability requires sample complexity $m = \Omega(nk^2)$, where $n$ is the single-round sample complexity.
\end{theorem}

This establishes a quadratic separation between adaptive and non-adaptive composition:
the upper bound of \Cref{thm:main} achieves $\wtilde{O}(nk)$, while the lower bound shows $\Omega(nk^2)$ is necessary for adaptive adversaries.
Note that the lower bound is tight: boosting the replicability parameter of each round from $\rho$ to $\rho/k$ and applying naive composition yields an $O(nk^2)$ algorithm.

\paragraph{Proof sketch.}
Our starting point is the hard instance of replicability posed by \cite{impagliazzo2022reproducibility}
(see Lemma 7.2 in their paper): 
any algorithm that takes $m$ samples from $\operatorname{Bernoulli}(\theta)$ and must decide whether $\theta$ is above or below $1/2$ (outputting $+$ when $\theta = 1/2 + \tau$, outputting $-$ when $\theta = 1/2 - \tau$, and either when $|\theta - 1/2| < \tau$) has replicability parameter $\ge \Omega(1/(\tau\sqrt{m}))$.
On a high level, the goal is to repeatedly present instances of this problem to the algorithm and
show that the replicability errors are \emph{independent} and therefore compound:
two different runs with different sample sets should diverge at \emph{some point} with probability at least
$\Omega(k/\tau\sqrt{m})$, implying the desired bound.
It is not a priori clear how to show this though given that the algorithm's randomness can cause the errors to correlate;
indeed, correlating errors is precisely how our non-adaptive algorithm is able to obtain an improved sample complexity.

To overcome this issue, we modify the hard instance of \cite{impagliazzo2022reproducibility} slightly
and show the following. There exists a single hard \emph{distribution} over threshold problem instances
such that for any algorithm and \emph{any fixed random bits $r$}, one of the following should hold:
either the probability (over the randomness of the problem instance and the samples) of the algorithm outputting a wrong solution
is large or the probability (again over the randomness of the problem instance and the samples)
of two runs with independent samples disagreeing is high (see \Cref{lem:threshold_estimation} and its preceding discussion for more details).
Our hard instance for the $k$-round lower bound simply draws $k$ independent problem instances from this distribution, one per round.
In expectation over $k$ rounds, this means that at least $k/2$ rounds have either high error or high replicability parameter.
Since the random bits of the algorithm $r$ is fixed, 
the algorithm's output is a deterministic function of the current samples and history. 
This ensures that as long as the adversary chooses independent problem instances,
the errors (either correctness or replicability) are independent and therefore compound, implying the desired $\Omega(nk^2)$ sample complexity (see \Cref{clm:fixed_r_lb}).

We note a subtlety arises when making the compounding argument rigorous;
when we compare two independent runs with different samples on round $i$ they do not
share the same \emph{history}, even if they have the same output up to round $i-1$ and (by assumption) use the same random bits.
Concretely, the information from the samples of the previous rounds means that in principle, the algorithms
used in round $i$ may be different in the two runs and therefore we cannot apply the hardness of \Cref{lem:threshold_estimation} directly.
We resolve this by introducing a third \emph{phantom} run that shares run $2$'s history but uses fresh independent samples.
\Cref{lem:threshold_estimation} ensures run $2$ and $3$ must disagree with lower bounded probability
which implies at least one of them must disagree with run $1$. Since the two runs have the same distribution,
this implies a lower bound on the probability of run $1$ and $2$ disagreeing (see the proof of \Cref{clm:fixed_r_lb} for more details).

We refer to this technique---introducing the third independent run---as the \emph{phantom run}.
As we discuss next, this technique has applications beyond the lower bound itself.

\subsubsection{Structural Properties of Replicable Algorithms}
The phantom run used in the lower bound above has a broader conceptual message.
In that argument, a fresh resample from the conditional distribution given the sufficient statistic
(here, the empirical distribution) was used to decouple the history from the current round.
The same resampling idea yields a general structural principle: a replicable algorithm never
needs to look at information in the data that is statistically irrelevant to its task.

Precisely, recall that a \emph{sufficient statistic} for a parametric family $\cF = \{D_\theta : \theta \in \Theta\}$
is a function $f(S)$ of the sample that captures all information about $\theta$.
We show that any replicable algorithm can be converted to one that operates
only on $f(S)$ rather than the full sample, with replicability preserved for distributions in $\cF$:

\begin{theorem}\label{thm:intro_stuff_stat}[Informal; see \Cref{thm:suff_stat}]
    Let $\cA$ be a $\rho$-replicable algorithm for a statistical problem over a family $\cF$,
    and let $f$ be a sufficient statistic for $\cF$.
    Then there exists an algorithm $\cB$ with the same accuracy and sample complexity
    such that $\cB(S)$ depends on $S$ only through $f(S)$,
    and $\cB$ is $2\rho$-replicable for any distribution $D \in \cF$.
\end{theorem}

In general, the factor-of-$2$ loss in the replicability parameter is a mild drawback: one can start with a $(\rho/2)$-replicable algorithm and apply \Cref{thm:intro_stuff_stat} to obtain a $\rho$-replicable one with no asymptotic cost.
Concretely, given the replicability parameter boosting result of
\cite{impagliazzo2022reproducibility}, the dependence of sample complexity on
$\rho$ is always $1/\rho^2$ in the worst case, so the loss is fairly mild.

A natural special case of the above result is \emph{order invariance}, obtained by taking $f(S)$ to
be the unordered multiset $\{S_1,\ldots,S_n\}$: the algorithm sees only the set
of samples, not their order.
\Cref{thm:intro_stuff_stat} then implies that any $\rho$-replicable algorithm can be converted into a $2\rho$-replicable, order-invariant one with no change in accuracy or sample complexity (see \Cref{cor:order_inv}).

\paragraph{Label invariance for symmetric problems.}
The same resampling idea also yields \emph{label invariance} when the problem is \emph{symmetric}---that is,
its correctness criterion is invariant under relabeling of the domain elements
(uniformity testing, closeness testing, and entropy estimation are canonical examples).
Here, instead of resampling the datapoints, one resamples the domain labels:
apply a uniformly random permutation $\pi$ of the domain $\cX$
(drawn from $\mathfrak{S}_{|\cX|}$, the symmetric group on $|\cX|$ elements, using shared randomness $r'$)
to the input before running $\cA$.

\begin{corollary}[Informal; see \Cref{cor:label_inv}]
    Let $\cA$ be a $\rho$-replicable algorithm for a symmetric statistical problem over a finite domain $\cX$.
    Then there exists a $\rho$-replicable, label-invariant algorithm $\cB$ with the same accuracy and sample complexity.
\end{corollary}

The construction is $\cB(S; r, r') = \cA(\pi(S); r)$, where $\pi \sim \mathrm{Uniform}(\mathfrak{S}_{|\cX|})$ is drawn from $r'$.
Since $\pi$ is shared, the relabeled samples $\pi(S_1)$ and $\pi(S_2)$ are i.i.d.\ from $\pi_*(D)$,
and replicability applies directly.
The label invariance guaranteed here is \emph{distributional}: for any fixed sample $S$,
the output distribution of $\cB(S)$ is the same as that of $\cB(\sigma(S))$ for every domain permutation $\sigma$.
A stronger \emph{pointwise} guarantee---where $\cB(S; r) = \cB(\sigma(S); r)$ holds for every fixed seed $r$---follows
for any finite output space by a further step: compute the output distribution
of $\cB(S)$ and resample from it using correlated sampling, at the cost of a
factor of $2$ in $\rho$ (see \Cref{cor:pointwise_label_inv}).

\paragraph{Relation to prior work.}
Whether replicable algorithms can always be assumed to have canonical structural properties
has been an active question in the literature.
\cite{liu2024replicable} studied replicable uniformity testing and proved lower bounds only for the
restricted class of \emph{label-invariant} algorithms (their Definition~1.5), explicitly leaving open
whether this restriction is without loss of generality.
Their notion of label invariance is \emph{pointwise}: for every fixed seed $r$, the output is unchanged
under domain relabeling.

\cite{aamand2025structure} resolved this open question for binary hypothesis testing of symmetric properties
(their Theorem~1.2): any $\rho$-replicable binary tester can be converted, with no loss in $\rho$,
to one that is simultaneously order-invariant and label-invariant.
Interestingly, their \emph{definition} of label invariance (Definition~4.2) is distributional,
but their construction---which averages the test statistic $f$ over all domain permutations to form a
deterministic symmetric function $h$---achieves the stronger pointwise notion as a byproduct.
Their proof represents $\cA$ via a scalar acceptance probability $f(S) \in [0,1]$ and
constructs the symmetrized algorithm by considering all possible permutations of $f$.
This is fundamentally specific to binary output: only for binary algorithms does a single scalar
completely characterize the output distribution, so the approach has no counterpart for general output spaces.
Our \Cref{cor:order_inv} and \Cref{cor:label_inv} together give the same two structural properties
for \emph{any} output space and \emph{any} symmetric statistical problem, via simple black-box resampling.

Regarding the sufficient statistic result specifically: \cite{aamand2025structure} prove a related reduction in their lower bound proof for Gaussian mean testing, using the fact that the empirical mean is a sufficient statistic for the Gaussian family.
Their reduction, like \Cref{thm:intro_stuff_stat}, guarantees replicability when both datasets are drawn from the same member of the parametric family.
\Cref{thm:intro_stuff_stat} is stated as a general standalone transformation that applies to any statistical problem and any sufficient statistic, rather than being specific to hypothesis testing.
We note that their approach preserves the replicability parameter exactly, whereas \Cref{thm:intro_stuff_stat} incurs a factor-of-$2$ loss; as discussed above, this loss is negligible.

\subsection{Boosting Success Probability}
\label{sec:intro_boost}
Following \Cref{thm:main}, a natural question is whether the dependence on the failure probability~$\beta$ can be improved.
Our next result shows that, for a broad class of problems, this is indeed possible, providing a black-box boosting of the success probability.
The question of when boosting is possible is a central theme studied in learning theory, with the celebrated AdaBoost result of \cite{freund1997decision} being perhaps the most famous example.\footnote{Indeed, the primary motivation of the work of \cite{dwork2010boosting} who introduce advanced composition is boosting as indicated by the title.} Previous works in replicability have studied boosting the replicability parameter and accuracy (see \Cref{sec:related} for more details). Our work provides the first boosting results for success probability, directly improving several existing bounds from the literature.

To illustrate our approach we begin with a very simple idea. Suppose that we are given a $\rho$-replicable algorithm with failure probability $\rho$, and we want to reduce the failure probability to some value $\beta \ll \rho$.
Assume that we can efficiently test whether or not a given solution is valid for a statistical problem.
In this case, to boost the success probability, we can first run a replicable algorithm to find a candidate solution and then, if the solution is invalid, fall back on a non-replicable algorithm and use that instead.
It is easy to see that this approach (approximately) preserves replicability: if we take two separate runs of the algorithm, with probability $1-\rho$ they agree and with probability $1-\rho$, the output is correct. Assuming the tester performs correctly in both cases, with probability at least $1-2\rho$ the algorithms will both have the same output; indeed, they will both have the same output as the original algorithm. It is also clear that, since the original replicable algorithm's output is always checked, the failure probability is determined by the non-replicable algorithm which, in most cases, is much more efficient. 

The problem with the above idea is that sometimes it can be very difficult to test whether or not a proposed value is a valid solution. As an example, take the problem of estimating the mean of a Bernoulli random variable with additive error $\epsilon$. 
Let $\mu$ and $x$ denote the true mean and proposed value respectively.
A natural approach is to estimate the mean empirically and compare its distance from $x$ with $\epsilon$. If $\abs{\mu - x}$ is exactly $\epsilon$ however, this distance can be either higher or lower than $\epsilon$ with roughly equal probability.

To address the issue, we slightly adjust our approach by employing a stronger algorithm to obtain the initial solution, which in turn allows us to use an approximate tester later on.
To formalize this,
we consider problems parameterized by an \emph{accuracy parameter}~$\alpha$, where smaller values of $\alpha$ correspond to more accurate solutions,
and require the following two properties for the problem:
\begin{enumerate}
    \item The replicable sample complexity depends polynomially on $(\rho^{-1},\alpha^{-1}, \beta^{-1})$.
    More precisely, changing the parameters by constant factors should not dramatically change the sample complexity.
    \item There is a small-sample \emph{tester} that, given a candidate solution, accepts whenever it is valid for~$\cT_\alpha$ and rejects whenever it is invalid for~$\cT_{2\alpha}$.
\end{enumerate}
Below, we briefly mention how a tester can be implemented for two example problems.
\begin{itemize}
    \item \textbf{Mean estimation.}
    Consider the problem of estimating the mean with additive error $\alpha$. To implement a tester, estimate the mean non-replicably with additive error $\alpha/4$;
    a standard analysis shows that this can be achieved using at most $O(\log(1/\beta)/\alpha^2)$ samples.
    We then test whether a proposed value $x$ is within error $3\alpha/2$ of the estimated mean. By the triangle inequality, if the proposed value is at most $\alpha$ away from the actual mean, then it is at most $3\alpha/2$ away from the estimated mean and, conversely, if it is at least $2\alpha$ away from the actual mean then it is at least $3\alpha/2$ away from the estimated mean. Therefore, comparing the distance from the estimated mean with the threshold $3\alpha/2$ gives us a tester.
    \item \textbf{PAC learning.}
    Consider the problem of finding some hypothesis $h$ from a class of $\cH$ with error at most $\alpha$.
    To test a hypothesis, non-replicably estimate its error using fresh samples to test whether it is below $\alpha$ or above $2\alpha$. Using standard multiplicative Chernoff bounds, it can be shown that this can be achieved with sample complexity $O(\log(1/\beta)/\alpha)$.
\end{itemize}

We briefly state our result here and refer to \Cref{sec:boost} for more details (see \Cref{thm:boost_beta} and \Cref{rem:boost}).
\begin{theorem}[Informal]\label{thm:boost_intro}
    Let $\cT_{\alpha}$ be a statistical problem parameterized by an accuracy parameter $\alpha$ satisfying the two properties mentioned above.
    Let $n_{\textnormal{replic}}(\rho, \alpha, \beta)$ denote the sample complexity of a $\rho$-replicable algorithm solving the problem with failure probability $\beta$.
    Let $n_{\textnormal{tester}}(\alpha,\beta)$ denote the sample complexity of an approximate tester and 
    $n_{\textnormal{non-replic}}(\alpha,\beta)$ denote the sample complexity of a non-replicable algorithm for the problem.

    For sufficiently small $\alpha, \rho, \beta > 0$, there exists a replicable algorithm for the problem with sample complexity
    \begin{align*}
        n_{\textnormal{new}}(\rho,\alpha,\beta)
        =
        O\rbr{
            n_{\textnormal{replic}}(\rho,\alpha,\rho)
            +
            n_{\textnormal{tester}}(\alpha,\beta)
            +
            n_{\textnormal{non-replic}}(\alpha,\beta)
        }
    \end{align*}
\end{theorem}
The above result shows that the cost of driving the failure probability of a replicable algorithm to a very small value $\beta$ appears as a separate  \emph{additive} term independent of $\rho$.
One can view the input to the algorithm as consisting of two parts: a \enquote{replicable part}, which produces a stable candidate, and a small \enquote{non-replicable part}, which certifies or corrects its validity.

\paragraph{Applications.}
Our boosting theorem has immediate consequences for several core problems studied in prior work on replicability. 
In particular, many existing results fall into one of two categories:
(1) algorithms whose guarantees only hold for failure probability $\beta = \rho$, and 
(2) algorithms that extend to arbitrary $\beta$ but incur a \emph{multiplicative} $\polylog(1/\beta)$ penalty in the sample complexity. 
Applying our framework transforms both types into algorithms with strictly improved upper bounds where $\beta$ only appears in a separate additive term independent of $\rho$. 
We highlight a few representative examples in \Cref{table:only}.
Each new bound follows directly by plugging the corresponding replicable algorithm into our black-box boosting theorem. 
(See \Cref{sec:boost} for details and proofs.) We emphasize that the list is by no means exhaustive.

\begin{table}[h]
    \label{table:only}
    \centering
    \caption{
    Examples of improved upper bounds obtained via \Cref{thm:boost_intro}.
    The list includes Statistical Query estimation, PAC learning with bounded Littlestone dimension $d$, PAC learning for a finite class $\cH$, PAC learning of thresholds over $\{0, 1, \dots, d\}$, the $(\nu, \epsilon)$-approximate heavy-hitters problem, and the best arm problem with additive error $\alpha$ over a finite set $A$.
    We emphasize that the list is by no means exhaustive. For most existing work the dependence on $\beta$ appears as a multiplicative $\polylog(1/\beta)$ factor. \Cref{thm:boost_intro} (see \Cref{thm:boost_beta} for a formal statement) shows that in many cases $\beta$ appears in a separate additive term independent of $\rho$. 
    }
    \renewcommand{\arraystretch}{1.25}
    \begin{tabular}{lll}
        \toprule
        \textbf{Problem} & \textbf{Previous Bound} & \textbf{New Bound (This work)} \\
        \midrule
        Statistical Query &
        $O\left(\frac{\log(1/\beta)}{\alpha^2 \rho^2}\right)$ for $\beta \ge \rho$~\cite{impagliazzo2022reproducibility} &
        $O\left(\frac{\log(1/\rho)}{\alpha^2 \rho^2} + \frac{\log(1/\beta)}{\alpha^2}\right)$ \\

        PAC, Littlestone &
        $\wtilde{O}\left(\frac{d^{12}\log^3(1/\beta)}{\alpha^2 \rho^2}\right)$~\cite{bun2023stability} &
        $\wtilde{O}\left(\frac{d^{12}}{\alpha^2 \rho^2} + \frac{\log(1/\beta)}{\alpha}\right)$ \\

        PAC, finite classes &
        $O\left(\frac{(\log^2|\cH|+\log(1/(\rho\beta)))\log^3(1/\rho)}{\alpha\rho^2}\right)$~\cite{bun2023stability} &
        $O\left(\frac{(\log^2|\cH|+\log(1/\rho))\log^3(1/\rho)}{\alpha\rho^2} + \frac{\log(1/\beta)}{\alpha}\right)$ \\

        PAC, thresholds  &
        $\wtilde{O}\left(\frac{(\log^*d)^3\log^2(1/\beta)}{\alpha^2\rho^2}\right)$~\cite{bun2023stability} &
        $\wtilde{O}\left(\frac{(\log^*d)^3}{\alpha^2\rho^2} + \frac{\log(1/\beta)}{\alpha}\right)$ \\
        Heavy-hitters &
        $\wtilde{O}\left(\frac{\log(1/\beta)}{\nu\epsilon^2\rho^2}\right)$
        ~\cite{esfandiari2024replicable}
        &
        $\wtilde{O}\left(\frac{1}{\nu\epsilon^2\rho^2} + \frac{\log(1/\beta)}{\epsilon^2}\right)$ \\
        Best arm &
        $O(\frac{\log^3(|A|/\beta)}{\rho^2 \alpha^2})$
        ~\cite{hopkins2025generative}
        &
        $O\rbr{
         \frac{\log^3(|A|/\rho)}{\rho^2 \alpha^2} 
         + 
         \frac{\log(|A|/\beta)}{\alpha^2}
         }$ \\
        \bottomrule

    \end{tabular}
\end{table}

It's worth noting that the approximate heavy-hitter problem which appears in the table is parameterized by two parameters $\nu, \epsilon$ (where $\epsilon < \nu$) instead of a single parameter $\alpha$. Given a distribution $D$, the goal is to output a list $L$ that contains all elements satisfying $\Pru{D}{x} \ge \nu$ but does not contain any elements such that $\Pru{D}{x} < \nu - \epsilon$. While \Cref{thm:boost_intro} considers only a single parameter, the same idea applies with minor modification. Specifically, we need to first (replicably) solve a \enquote{stronger} version of the problem and then implement a tester that accepts all solutions to the stronger version and rejects all solutions to the original version. In this context, we can achieve this by setting $\nu' = \nu - \epsilon/4$ and $\epsilon' = \epsilon/2$. This ensures that
$\nu' < \nu$ and $\nu' - \epsilon' > \nu - \epsilon$; as such, any solution to $\cT_{\nu', \epsilon'}$ is also a solution for $\cT_{\nu ,\epsilon}$. Implementing a tester and non-replicable algorithm is slightly more involved than statistical query estimation and PAC learning; we refer to \Cref{sec:heavy} for more details.

\paragraph{Comparison with DP.} 
For the special case of PAC learning, similar results our known in the context of differential privacy. Specifically, \cite{sivakumar2021multiclass} show that failure probability can be boosted at the cost of a \emph{multiplicative} $\log(1/\beta)$ term in the sample complexity. 
This is done by taking $O(\log(1/\beta))$ copies of the algorithm and choosing a \enquote{good} output using the exponential mechanism where the scores of the outputs are their error with respect to a fresh set of samples. 
In contrast, our approach for replicability causes only an additive increase of $\log(1/\beta)/\alpha$.

\subsection*{Organization}
\Cref{sec:related} discusses the related work.
\Cref{sec:prelim} introduces notation and preliminaries.
\Cref{sec:main_thing} presents our replicable composition theorem;
\Cref{sec:app_sq,sec:app_best_arm}
present applications to multiple statistical queries and parallel best-arm
identification
and
\Cref{sec:compose} develops the composition theory for perfectly generalizing
algorithms that underlies the proof.
\Cref{sec:adaptive_lb} proves the $\Omega(nk^2)$ lower bound for adaptive composition; the structural consequences via the phantom run technique---the sufficient statistic theorem, order invariance, and label invariance---are developed in \Cref{sec:suff_stat,sec:suff_order_inv,sec:label_inv,sec:pointwise_label_inv}.
\Cref{sec:boost} presents our boosting framework; \Cref{sec:app_sq_boost,sec:app_boot_pac,sec:heavy,sec:app_best_arm_boost} present applications to statistical queries, PAC learning, heavy-hitters, and best-arm identification.
Omitted proofs appear in \Cref{sec:omitted}.
\Cref{sec:pg_error_explain} contains a detailed discussion of the challenges that arise in the adaptive composition of perfectly generalizing algorithms.

\section{Related Work}
\label{sec:related}
\paragraph{Replication.}
The formal study of replicability was initiated by Impagliazzo, Lei, Pitassi, and Sorrell~\cite{impagliazzo2022reproducibility}, who examined algorithmic guarantees ensuring that outputs remain consistent under independent resampling. The work introduced replicability as a rigorous framework for reasoning about stability in statistical learning and provided initial examples of replicable algorithms for basic estimation and learning tasks. Subsequent research explored its conceptual links to other stability notions. Bun, Gaboardi, Hopkins, Impagliazzo, Lei, Pitassi, Sivakumar, and Sorrell~\cite{bun2023stability} investigated the relationships among replicability, approximate differential privacy, and generalization, providing black-box reductions with polynomial overhead in the sample complexity. Kalavasis, Karbasi, Moran, and Velegkas~\cite{kalavasis2023statistical} studied similar connections in the PAC learning setting, introducing the related notion of TV-indistinguishability and studying its relation with privacy and replicability. Together, these works clarified how replicability fits within the broader landscape of stability concepts in learning theory.

Replicability has been examined in a broad range of algorithmic settings. This includes 
statistical query estimation~\cite{ghazi2021userlevel,impagliazzo2022reproducibility},
clustering~\cite{esfandiari2024replicable}, bandits and online learning~\cite{esfandiarireplicable, ahmadi2025replicable}, reinforcement learning~\cite{karbasi2023replicability,eaton2024replicable,hopkins2025generative,eaton2025replicable}, convex optimization~\cite{zhang2024optimal}, uniformity testing~\cite{liu2024replicable},
prophet inequalities~\cite{banihashem2025replicable},
median estimation~\cite{banihashem2026replicablemedian},
high-dimensional estimation and hypothesis testing~\cite{hopkins2024replicability, aamand2025structure,diakonikolas2025replicable}, and learning half-spaces~\cite{kalavasis2024replicable}. In addition, several papers have used topological and geometric techniques to identify fundamental limitations of replicability, establishing lower bounds and separation results for a variety of statistical learning problems~\cite{ReplicFocs0001MY23,DBLP:conf/nips/0002PWV23,LocalCMY24,DBLP:conf/nips/Woude0PRV24}. These developments have collectively expanded the theoretical understanding of replicable algorithms and their role in statistical and learning theory.

\paragraph{Composition.}
Given $k$ differentially private algorithms, each with parameter $(\eps, \delta)$, it is easy to see that their composition is $(\eps^*, \delta^*)$-DP where $(\eps^*, \delta^*) = (k\eps, k\delta)$. As shown by the advanced composition theorem however, this bound can be improved significantly; in particular, one can bound $\eps^*$ with approximately $\eps\sqrt{k}$~\cite{dwork2010boosting}. The theorem has numerous applications in privacy, and the underlying ideas of its proof are frequently used in different problems~\cite{bassily2016typical,bun2019heavy,dong2020optimal,dong2022gaussian,bun2023stability}. In particular, the transformation from differential privacy to replicability obtained by~\cite{bun2023stability} relies on the fact that differentially private algorithms are also perfectly generalizing; the proof of this fact draws heavily on techniques from the proof of advanced composition.

In the context of replicability, a similar naive analysis shows that the composition of $k$ $\rho$-replicable algorithms is $(k\rho)$-replicable. In this case however, the analysis is essentially tight, as such one cannot hope for an improved analysis (see \Cref{app:naive_is_tight}).
As we show in this paper however, if we change our perspective and look at the sample complexity of statistical problems, one can obtain a nearly linear scaling, as opposed to the nearly quadratic scaling implied by the naive composition, and this is essentially tight (see \Cref{sec:tight}).
Furthermore, we show in \Cref{thm:adaptive_lb} that once the adversary is allowed to choose problems adaptively based on previous outputs, no algorithm can achieve better than $\Omega(nk^2)$ sample complexity, establishing a sharp separation between the adaptive and non-adaptive settings.

Beyond naive composition, prior work has achieved better sample-complexity bounds under the relaxed paradigm of \emph{coordinate sampling}.
In this model, we are given $k$ statistical problems $\cT_1, \dots, \cT_k$, possibly over different domains.
As a simple example, consider estimating the means of $k$ distinct coins, which we also study in our work.
In the coordinate-sampling model, the algorithm may allocate samples adaptively across coordinates—deciding, based on previous observations, how many samples to take from each coin—so as to minimize the \emph{total number} of samples.
In particular, for the $k$-coin problem, the number of samples may vary across different coins.
\cite{hopkins2024replicability} analyze this model and show that it yields non-trivial improvements, effectively shaving a factor of $k$ from the sample complexity implied by naive composition (see Section 1.1.3 of their paper where they introduce this relaxation).
The key insight is that while a $\rho$-replicable algorithm requires $O(1/\rho^2)$ samples in the worst case, it needs only $O(1/\rho)$ samples in expectation.
Hence, one can boost the replicability parameter of each subroutine to $\rho/k$ without incurring a quadratic blowup in total sample complexity.
Notably, this improvement does not extend to the standard \emph{vector-sampling} model, where each coin must be sampled the same number of times.

\paragraph{Parameter boosting.}
As part of our results, we present a general reduction that amplifies the success probability of any replicable algorithm.
Parameter boosting is a well-studied theme in learning theory—appearing both in the contexts of replicability and differential privacy, and more broadly in PAC learning (see, e.g., the AdaBoost algorithm~\cite{freund1997decision}).

In the context of replicability, prior works have studied boosting the replicability parameter of the algorithm and, for replicable PAC learning, the accuracy of the underlying learner.
For the replicability parameter, this was studied by the original work of \cite{impagliazzo2022reproducibility}, who showed that one can boost replicability from a constant value $0.0001$ to a much smaller value $\rho$ at (nearly) a $\rho^{-2}$ cost in sample complexity, which is essentially tight (see the lower bound for the coin problem in the same paper).
For boosting accuracy, the problem was also studied originally by \cite{impagliazzo2022reproducibility}, and the recent work of \cite{larsen2025improved} improves on their bounds, showing that one can obtain an $\alpha$-accurate learner at the cost of $\alpha^{-1}$ in sample complexity (see Theorem 1.2 in their paper for full results).

In contrast to these works, our focus is on boosting the correctness probability.
For differential privacy, 
prior work has proposed approaches for this using multiple runs of the base algorithm~\cite{DBLP:conf/soda/GuptaLMRT10,DBLP:conf/stoc/0001T19,sivakumar2021multiclass}.
Most closely to our work, \cite{sivakumar2021multiclass} use multiple parallel
runs to show that one can boost the correctness of a PAC learning algorithm at
the cost of a $\log(1/\beta)$ multiplicative factor in the sample complexity.
We study boosting the correctness parameter for replicability and obtain a general result that applies to a broad class of problems (see \Cref{thm:boost_beta} and the surrounding discussion for details).
To our knowledge, no such result was previously known in the context of replicability, and as we discuss in \Cref{sec:intro_boost}, our boosting result directly yields improved bounds for a large class of problems studied in the literature. Interestingly, for the special case of PAC learning, our framework shows that failure probability can be boosted to $\beta$ at the cost of an \emph{additive} $\alpha^{-1}\log(\beta^{-1})$ term.

\section{Preliminaries}
\label{sec:prelim}

\paragraph{Statistical problem.}
We study randomized algorithms that solve statistical problems under replicability constraints and analyze how such guarantees compose across multiple problems. Formally, we define the concept of a statistical problem as follows.

\begin{definition}[Statistical Problem]
  A \emph{statistical problem} over a domain $\cX$ and output space $\cY$ is a set of pairs $\cT = \{(D, G_D)\}$, where $D$ is a distribution over $\cX$ and $G_D \subseteq \cY$ specifies the set of acceptable outputs for $D$.
  A randomized algorithm $\cA$ solves the problem $\cT$ with failure probability $\beta$ and sample complexity $n$ if, for all $(D, G_D) \in \cT$, the algorithm $\cA$, given $n$ i.i.d. samples from $D$, outputs a value in $G_D$ with probability at least $1 - \beta$.
\end{definition}

For example, in the problem of mean estimation up to additive error $\eps$, the output set is $G_D = [\mu_D - \eps, \mu_D + \eps]$ where $\mu_D$ is the mean of $D$. In a binary hypothesis testing problem distinguishing between distributions with mean above $1$ versus below $0$, the output space is $\{0,1\}$ and $G_D$ depends on which case $D$ falls into. Note that $\cY$ may differ from $\cX$.

We assume throughout the paper that $\cT$ contains a pair $(D, G_D)$ for every distribution $D$ over $X$. This assumption is without loss of generality, as any missing pair can be added by defining $G_D = \cY$. From an algorithmic standpoint, this can be achieved by terminating $\cA_i$ when a predefined time limit is exceeded and returning an arbitrary output. Under this assumption, we will often write $\cT(D)$ to refer to the set $G_D$.
We further note that this does not affect replicability; the standard definition of replicability considered in the literature requires \Cref{eq:def_replic} to hold for \emph{all} distributions over $\cX$.

In this work, we focus on algorithms $\cA$ for which the output space $\cY$ is finite. This assumption is needed because the existing transformation between replicability and perfect generalization, which we use in our paper, requires finite $\cY$ (see \Cref{lm:pg_to_replic}). 
In most cases, this assumption is not restrictive; when $\cY$ is continuous one can generally discretize the space and round the output to the nearest point in the discretization scheme. 
For instance, for mean estimation up to additive error $\epsilon$, we can first run a replicable algorithm with additive error $\epsilon/2$, then round the output to the nearest integer multiple of $(\eps/2)$. Note that the rounding procedure is deterministic and as such does not affect the replicability guarantee.
An alternative approach is to assume that the input domain $\cX$ is countable and, in order to transform a perfectly generalizing algorithm to a replicable one, apply Corollary 1.4 from \cite{bun2023stability} instead of Theorem 3.17.

\paragraph{Replicability and Perfect Generalization.}
Our main focus in this paper is designing \emph{replicable} algorithms for statistical problems, which we formally define below.

\begin{definition}[Replicability \cite{impagliazzo2022reproducibility}]
    \label{def:replic}
    A randomized algorithm $\cA : \cX^n \to \cY$ is said to be $\rho$-replicable if for every distribution $D$ over $\cX$,
    \begin{align}
        \Pru{S_1, S_2, r}{\cA(S_1; r) = \cA(S_2; r)}\ge 1 - \rho,
        \label{eq:def_replic}
    \end{align}
    where $S_1, S_2$ are independent samples of size $n$ drawn from $D$ and $r$ is the internal randomness of $\cA$.
\end{definition}

As part of our analysis, we will use the concept of perfect generalization as introduced by \cite{CummingsLNRW16}. 
\begin{definition}[Perfect Generalization (PG) {\cite{CummingsLNRW16}}]
    \label{def:pg}
    An algorithm $\cA : \cX^n \to \cY$ is said to be $(\gamma, \eps, \delta)$-\emph{perfectly generalizing} (abbreviated as \emph{PG}) if, for every distribution $D$ over $\cX$, there exists a distribution $\Sim_D$ over $\cY$ such that, with probability at least $1 - \gamma$ over the sample $S \sim D^n$, the output distribution of $\cA(S)$ satisfies, for every set of outcomes $O \subseteq \cY$,
    \begin{align*}
        e^{-\eps} \rbr{\Pr{\Sim_D \in O} - \delta}
        \le \Pr{\cA(S) \in O} \le e^\eps \Pr{\Sim_D \in O} + \delta.
    \end{align*}
\end{definition}

Intuitively, perfect generalization requires that if we independently sample two datasets, the output distribution of the algorithm is similar on these two datasets. In contrast, replicability requires that their output is almost always equal when using the same random bits. 
The two concepts are closely related and, as formally shown by \cite{bun2023stability}, replicable algorithms can be transformed into perfectly generalizable algorithms and vice versa.

\paragraph{Composition of statistical problems.}
Let $\cT_1, \dots, \cT_k$ be $k$ statistical problems defined over the same input space $\cX$, with possibly different output spaces $\cY_1, \dots, \cY_k$. For each $i \in [k]$, let $\cA_i$ be a replicable algorithm for $\cT_i$. Given access to these algorithms, we are interested in designing a new algorithm that solves all of these problems simultaneously and remains replicable.
Formally, we define the \emph{composed problem}  $\cTcomp$ over the input space $\cX$ and output space $\cYcomp = \cY_1 \times \dots \times \cY_k$ as
\[
\cTcomp(D) = \cT_1(D) \times \dots \times \cT_k(D) \subseteq \cYcomp.
\]
Our goal is to construct a replicable algorithm $\cAcomp$ that solves $\cTcomp$ while minimizing sample complexity.

\paragraph{Comparison with composition in privacy.}
In differential privacy, the composition is generally defined over algorithms $\cA_1, \dots, \cA_k$
directly, rather than the statistical problems. Specifically, the main question is what the privacy parameters
of the algorithm $\cA$ defined as $\cA(S) = (\cA_1(S), \dots, \cA_k(S))$ are.
For replicability, we focus on the composition of statistical problems because, perhaps surprisingly, it allows
us to obtain a better sample complexity.
Indeed, if we have $k$ algorithms
$\cA_1, \dots, \cA_k$ where $\cA_i$ is $\rho$-replicable,
their composed algorithm $\cA$ is $(k\rho)$-replicable and this is essentially tight  (see \Cref{app:naive_is_tight}).
It follows that boosting the replicability parameters of each $\cA_i$ to $\rho/k$
would incur a quadratic dependence on $k$.
In contrast, \Cref{thm:main} shows that 
one can obtain a nearly linear dependence which is essentially tight.

\section{Replicable Composition}
\label{sec:main_thing}

In this section, we prove \Cref{thm:main}.
Our proof leverages the connection between replicability and perfect generalization (PG) as formalized by \cite{bun2023stability}. We begin by summarizing existing tools from the literature that we will use in our analysis.
We start with a transformation from replicability to PG.
\begin{lemma}\label{lm:replic_to_pg}(Theorem 3.19 in \cite{bun2023stability})
  Let $\delta, \beta > 0$ be sufficiently small, and let $\eps \in (0, 1]$.
  Any $(10^{-4})$-replicable algorithm $\cA$ with sample complexity $n$ and failure probability $\beta$
  can be converted into a $(\delta, \eps, \delta)$-perfectly generalizing algorithm $\cB$ with failure probability
  \begin{align*}
    O(\delta^2) + \sqrt{\beta} \log(8/\delta^2)
  \end{align*}
  and sample complexity
  \begin{align*}
    \frac{n \log(1/\eps)}{\eps^2} \cdot \polylog(1/\delta).
  \end{align*}
\end{lemma}

This lemma follows by combining two results from \cite{bun2023stability}; see \Cref{app:replic_to_pg} for details.
Note that the output set of the corresponding perfectly generalizing algorithm $\cB$ is the same as $\cA$. This holds because their approach (Algorithm 3 in their paper) designs $\cB$ by running the exponential mechanism on the outputs of multiple runs of $\cA$ with different random bits and sample sets. Since all such values are in the output space of $\cA$, the output of $\cB$ lies in the same space as well.

We next state a lemma that bounds the generalization parameters of
the composition of $k$ PG algorithms.
The result actually holds in the stronger \emph{adaptive} model;
while we do not need this, we state it here as it may be of independent interest.
We refer to \Cref{sec:compose} for the definition of the adaptive composition model.
\begin{theorem}\label{lm:pg_compose}
  There exist absolute constants $c_1, c_2, c_3, c_4$ such that the following holds.
  Let $\cA_1, \dots, \cA_k$ be algorithms that are $(\delta_i, \eps_i, \delta_i)$-PG.
    Let $\delta' \in (0, 1/2)$.
  Suppose $\eps_i, \frac{\delta_i}{\eps_i^2} \le c_1$, and define
  \begin{align}
      \eps^* = c_2 \rbr{ \sqrt{\log(1/\delta')\rbr{\sum_{i} \eps_i^2}} + \sum_i \eps_i^2 }.
      \label{eq:def_eps_star}
  \end{align}
  If $\eps^* \le c_3$, then the adaptive composition of $\cA_1, \dots, \cA_k$ is $(\delta^*, \eps^*, \delta^*)$-PG, where
  \begin{align}
      \delta^* = c_4\sqrt{\frac{1}{\eps^{*}}\rbr{ k\delta' + \sum_{i}\frac{\delta_i}{\eps_i} } }
      .
      \label{eq:def_delta_star}
  \end{align}
\end{theorem}

We refer to \Cref{sec:compose} for a proof sketch and full proof.

Finally, we use a transformation from perfect generalization to replicability. This transformation applies to all algorithms that are one-way perfectly generalizing, a weaker notion than standard perfect generalization.

\begin{theorem}[Theorem 3.17 in \cite{bun2023stability}]
\label{lm:pg_to_replic}
    Fix $n \in \mathbb{N}$ and $\gamma, \eps, \delta \in (0,1)$. 
    Let $\cA : \cX^n \to \cY$ be a $\gamma, \eps, \delta$ one-way perfectly generalizing algorithm with finite output space.
    Then there exists a $\rho$-replicable algorithm with $\rho = 4(\gamma + 2\eps + \delta)$ such that, for any dataset $S$, its output distribution when run on $S$ with oracle access to $\cA$ matches that of $\cA(S)$.
\end{theorem}

We next describe the approach used in the proof.
Let $(\eps_i, \delta_i)$ be parameters to be specified later.
We first use \Cref{lm:replic_to_pg} to transform each $\cA_i$
to a PG algorithm $\cB_i$ with parameters $(\delta_i, \eps_i, \delta_i)$.
Let $\cBcomp = (\cB_1, \dots, \cB_k)$ denote the composition of these algorithms.
\Cref{lm:pg_compose} implies that $\cB$ is PG with the parameters specified in the lemma.
We then transform this into a replicable algorithm again using \Cref{lm:pg_to_replic}.
As mentioned earlier, the output set of each $\cB_i$ is the same as $\cA_i$ and as such it is finite. It follows that the output set of $\cBcomp$ is also finite and as such we can invoke \Cref{lm:pg_to_replic}.
A formal pseudocode is provided in \Cref{alg:main}.

The value of $\eps_i, \delta_i$ is provided in \Cref{eq:def_eps_i_delta_i}. Intuitively, the values are set such that the sample complexity of solving $\cA_i$ with the desired parameters, as implied by \Cref{lm:replic_to_pg}, is the same for all $i$. 
In particular, this means that for smaller $n_i$ we choose a smaller $\eps_i$. 

\begin{algorithm}[H]
    \caption{Construction of $\cAcomp$}
    \label{alg:main}
    \KwIn{$\{\cA_i\}_{i=1}^k$, where $\cA_i : \cX^{n_i} \to \cY_i$ is $10^{-4}$-replicable with failure probability $\beta_0$; parameters $\rho, \beta$}
    \KwOut{$\cAcomp : \cX^n \to \cY_1 \times \dots \times \cY_k$, where $n$ is defined as in \Cref{eq:bound_n_main_thm}}
    
    Set 
    $\eps_i, \delta_i$ as in \Cref{eq:def_eps_i_delta_i}
    \tcp*{see proof of \Cref{thm:main}}
    
    \For{$i \gets 1$ \KwTo $k$}{
      $\cB_i \leftarrow \textsc{ReplicabilityToPG}(\cA_i, \eps_i, \delta_i)$ \tcp*{$\cB_i : \cX^n \to \cY_i$ (Lemma~\ref{lm:replic_to_pg})}
    }
    
    $\cBcomp \leftarrow (\cB_1, \dots, \cB_k)$ \tcp*{$\cBcomp : \cX^n \to \cY_1 \times \dots \times \cY_k$}
    
    $\cAcomp \leftarrow \textsc{PGToReplic}(\cBcomp)$ \tcp*{Lemma~\ref{lm:pg_to_replic}}
\end{algorithm}

\begin{proof}[Proof of \Cref{thm:main}]
    Throughout the proof, we assume $\beta_0$ is sufficiently small to satisfy the preconditions of \Cref{lm:replic_to_pg}; otherwise
    the lemma holds trivially by picking large enough constants under $O(.)$ since one can always
    obtain a $1$-replicable algorithm with failure probability $1$ by outputting a fixed value such as $0$. We similarly assume that $\beta_0 \le 1/4$.
    
    For $i \in [k]$, set $\eps_i, \delta_i$ as
    \begin{align}
        \eps_i = 
        c\rho\sqrt{\frac{1}{\log(\frac{k}{\rho \beta_0})}
        \rbr{\frac{n_i}{(\sum_{j=1}^k n_j)} + \frac{1}{k}}
        }
        \quad\text{ and }\quad
        \delta_i = \beta_0\eps_i^{10}
        ,
        \label{eq:def_eps_i_delta_i}
    \end{align}
    where $c < 1$ is a parameter to be chosen later;
    we will set $c$ to be a small absolute constant but for now we explicitly track the dependence on $c$ in $O(.)$ notation.
    Set $\delta' = \frac{1}{k}\rbr{\beta_0\sum_{i=1}^k \eps_i^{10}}$.
    Set $\eps^*, \delta^*$ as in \Cref{lm:pg_compose} and
    let $c_1, c_2, c_3, c_4$ be the constants from the same lemma.

    We start by bounding $\log(1/\delta')$.
    Observe that
    for any $i \in [k]$,
    \begin{align}
        \eps_i \ge c\rho\sqrt{\frac{1}{k \log(k/\rho \beta_0)}}
        .
        \label{eq:lower_bound_eps_i_1_over_k}
    \end{align}
    By definition of $\delta'$,
    it follows that
    \begin{align*}
        k \delta' \ge \beta_0\eps_i^{10} 
        \ge 
        \beta_0
        \rbr{\frac{c^2\rho^2}{k \log(k/\rho\beta_0)}}^5,
    \end{align*}
    which implies
    \begin{align}
        \notag
        \log(1/\delta') &\le 
        O\rbr{
            \log(1/\beta_0) + 
            \log(1/c) + \log(1/\rho) + \log(k) + \log \log(k/\rho \beta_0)
        }
        \\&\le
        O\rbr{
            \log\rbr{1/c} + \log\rbr{\frac{k}{\rho\beta_0}}
        }
        \label{eq:bound_delta_prime}
    \end{align}

    We claim that if $c$ is small enough, the preconditions of \Cref{lm:pg_compose} hold.
    \begin{claim}\label{cl:c_small_pg_compose_precond_hold}
        If $c$ is smaller than some absolute constant,
        then $\eps_i, \delta_i/\eps_i^2 \le c_1$, $\eps^* \le \min\cbr{\rho/40, c_3}$
        and
        $\delta^* \le \rho^4/40$, and $\delta' \le \frac{1}{2k}$.
    \end{claim}
    \begin{proof}
        Since $\rho \le 1$, and $\log(k / \rho \beta_0) \ge \log(2) =1$,
        we have
        $\eps_i \le c\sqrt{2}$
        which implies
        $\eps_i \le c_1$ holds for $c\le c_1/\sqrt{2}$.
        The bound $\delta_i/\eps_i^2\le c_1$ holds because
        $\delta_i/\eps_i^2 = \beta_0\eps_i^8$ which is
        at most $\eps_i$ assuming $c \le 1/\sqrt{2}$ since $\eps_i\le c\sqrt{2}$ and $\beta_0 \le 1$.

        We next bound $\eps^*$.    
        We first note that
        \begin{align}
            \notag
            \sum_{i} \eps_i^2
            &= 
            \sum_{i=1}^k \frac{c^2 \rho^2}{\log(\frac{k}{\rho \beta_0})}
            \rbr{ \frac{n_i}{(\sum_{j=1}^{k} n_j)} + 1/k}
            &\EqComment{\Cref{eq:def_eps_i_delta_i}}
            \\&=
            \frac{2c^2 \rho^2}{\log(\frac{k}{\rho \beta_0})}
            .
            \label{eq:sum_eps_i_squared}
        \end{align}
        Plugging this together with \Cref{eq:bound_delta_prime} in the definition of $\eps^*$ we have
        \begin{align*}
            \eps^* 
            &= 
            O\rbr{ \sqrt{\log(1/\delta')\rbr{\sum_{i} \eps_i^2}} + \sum_i \eps_i^2 }.
            &\EqComment{\Cref{eq:def_eps_star}}
            \\&\le 
            O\rbr{\sqrt{c^2 \rho^2\frac{\log(1/c) + \log(k/\rho \beta_0)}{\log(k/\rho \beta_0)}} + \frac{c^2 \rho^2}{\log(k/\rho \beta_0)}}
            &\EqComment{\Cref{eq:bound_delta_prime} and \Cref{eq:sum_eps_i_squared}}
            \\&\le 
            O\rbr{
                c\rho\rbr{1 + \sqrt{\frac{\log(1/c)}{\log(k/\rho \beta_0)} }}
            }
            &\EqComment{Since $c, \rho \le 1$}
            \\&\le
            O\rbr{
                c\rho(1 + \sqrt{\log(1/c)})
            }
            .
        \end{align*}
        For some constant $c' > 0$, this is at most
        \begin{align*}
            \frac{\rho}{40} c'c (1 + \sqrt{\log(1/c)}) 
            .
        \end{align*}
        For small enough (but still constant) $c$, we have
        $c(1 + \sqrt{\log(1/c)}) \le \frac{\min\{1, 40c_3/\rho\}}{c'}$, which means
        $\eps^* \le \min\cbr{\rho/40, c_3}$ as claimed.

        We next bound $\delta^*$.
        First observe that
        \begin{align}
            \eps^* \ge
            \Omega\rbr{
                \sqrt{
                    \sum_{i=1}^k \eps_i^2
                }
            }
            \ge
            \Omega
            \rbr{
                \max_{i} \eps_i
            },
            \label{eq:lower_bound_eps_star}
        \end{align}
        which implies
        \begin{align*}
            \frac{1}{\eps^*}
            \rbr{
                k\delta' + \sum_i \frac{\delta_i}{\eps_i}
            }
            &\le 
            \sum_i 
            \rbr{
                \frac{1}{\eps^*}
                (\beta_0 \eps_i^{10} + \frac{\delta_i}{\eps_i} )
            }
            &\EqComment{Definition of $\delta'$}
            \\&\le
            O\rbr{
                \sum_i \rbr{
                    \frac{1}{\eps_i} (\beta_0 \eps_i^{10} + \frac{\delta_i}{\eps_i} )
                }
            }
            &\EqComment{\Cref{eq:lower_bound_eps_star}}
            \\&\le 
            O\rbr{
                \sum_{i} \beta_0 \eps_i^8
            }
            &\EqComment{Definition of $\delta_i$ and $\eps_i \le c_1$}
            \\&\le
            O\rbr{
                (\sum_{i} \eps_i^2)^4
            }
            &\EqComment{Since $\beta_0 \le 1$}
            \\&\le 
            O\rbr{
                c^8 \rho^8
            }
            &\EqComment{\Cref{eq:sum_eps_i_squared}}
            .
        \end{align*}
        It follows that
        $\delta^* \le O(c^4 \rho^4)$ which is at most $\rho^4/40$ for small enough $c$.
        The same proof implies that $\delta'$ is at most $1/(2k)$.
    \end{proof}
    Fix $c$ to be small enough as required by \Cref{cl:c_small_pg_compose_precond_hold}.
    We proceed to prove the replicability of the algorithm and bound its sample complexity and failure probability.

    \paragraph{Replicability.}
    By the above claim, if $c$ is a small enough constant, the preconditions of \Cref{lm:pg_compose} hold and
    therefore $\cBcomp$ is $(\delta^*, \eps^*, \delta^*)$-PG.
    By \Cref{lm:pg_to_replic}, this implies that
    $\cAcomp$ is replicable with parameter
    \begin{align*}
        4(\delta^* + 2\eps^* + \delta^*) &\le
        4(\rho^4/40 + \rho/20 + \rho^4/40)
        &\EqComment{\Cref{cl:c_small_pg_compose_precond_hold}}
        \\&\le
        \rho
        &\EqComment{Since $\rho \le 1$}
    \end{align*}

    \paragraph{Failure probability.}
    By \Cref{lm:replic_to_pg}, the failure probability of each $\cB_i$ is at most  
    \begin{align*}
        \sqrt{\beta_0} \log(8/\delta_i^2) + O(\delta_i^2)
        .
    \end{align*}
    Since $\delta_i = \beta_0\eps_i^{10}$, we have
    $O(\delta_i^2) \le O(\beta_0^2)$.
    Additionally, using \Cref{eq:lower_bound_eps_i_1_over_k} we obtain
    \begin{align}
        \log(1/\eps_i)
        \le 
        O\rbr{
        \log(1/c) +\log(1/\rho) + \log(k)
        + \log\log(k/\rho \beta_0)
        }
        \le
        O\rbr{\log(k/\rho \beta_0)}
        ,
        \label{eq:bound_log_eps_i}
    \end{align}
    where for the final inequality we have used the fact that $k \ge 2$
    to absorb $\log(1/c)$ under $O(.)$.
    Since $\delta_i = \beta_0 \eps_i^{10}$,
    this implies
    \begin{align}
        \log(8/\delta_i^2) 
        &= O(\log(1/\eps_i) + \log(1/\beta_0) + \log(8))
        \le 
        O\rbr{
        \log(k/\rho\beta_0)
        }
        .
        \label{eq:bound_log_delta_i}
    \end{align}
    It follows that each 
    algorithm $\cB_i$ has failure probability $O(\sqrt{\beta_0}\log(k/\rho \beta_0))$. 
    Taking a union bound over all $i$, we conclude that the failure probability of $\cBcomp$ is at most
    $O(k\sqrt{\beta_0}\log(k/\rho \beta_0))$. Since $\cAcomp$ has the same output distribution as $\cBcomp$ for any input, the failure probability bound holds for $\cAcomp$ as well.

    \paragraph{Sample complexity.}
    Observe that $\cAcomp$ uses the same number of samples as $\cBcomp$, which in turn uses the same number of samples as the largest $\cB_i$.
    By \Cref{lm:replic_to_pg}, we can bound this sample complexity as
    \begin{math}
        \frac{n_i \log(1/\eps_i)}{\eps_i^2}\polylog(1/\delta_i).
    \end{math}
    By \Cref{eq:bound_log_eps_i} and \Cref{eq:bound_log_delta_i},
    both $\log(1/\eps_i)$ and $\log(1/\delta_i)$ are at most
    $\polylog(k/\rho\beta_0)$.
    Additionally, 
    \begin{align*}
        \eps_i \ge \Omega\rbr{
        \rho 
        \sqrt{\frac{n_i}{\log(k/\rho \beta_0)(\sum_{j} n_j)}}
        }
        ,
    \end{align*}
    which implies
    \begin{align*}
        \frac{n_i}{\eps_i^2}
        \le O
        \rbr{
        \frac{n_i}{\rho^2\frac{n_i}{\log(k/\rho \beta_0)(\sum_{j} n_j)}}
        }
        &=
        O
        \rbr{
        \rbr{\frac{\sum_{j} n_j}{\rho^2}} \log(k/\rho \beta_0)
        }
        .
    \end{align*}
    It follows that
    the sample complexity of each $\cB_i$ is
    at most
    \begin{math}
        O\rbr{
        \frac{\sum_{j} n_j}{\rho^2} \polylog(k/\rho \beta_0)
        }
    \end{math}
    as claimed.
\end{proof}

\subsection{Application: Multiple Statistical Queries}
\label{sec:app_sq}
For our first application, we study the sample complexity of estimating multiple statistical queries on the same distribution. 
Formally, we are given $k$ functions $\phi_1, \dots, \phi_k$ where
each $\phi_i$ maps an input $x \in \cX$ from the ground set to some value in the range $[0, 1]$. Defining $\mu_i = \Exu{X \sim D}{\phi_i(X)}$, the goal is to estimate all values $\mu_i$ at the same time with additive error $\alpha$. 
The problem has numerous applications in statistics; in its simplest form, it captures the problem of estimating the mean of $k$ Bernoulli random variables with $\ell_{\infty}$ error $\alpha$, also known as the \emph{$k$-coin} problem.

When $k=1$, the problem can be solved with sample complexity $\wtilde{O}(\frac{\log(1/\beta)}{\rho^2\alpha^2})$ as shown by \cite{impagliazzo2022reproducibility}, where $\beta$ denotes the failure probability.
A naive generalization to $k$ queries is to estimate each query with $\mu_i$ with replicability parameter $\rho/k$, but this leads to sample complexity quadratic in $k$ which is sub-optimal.
Previously, this problem was studied by \cite{karbasi2023replicability} who also noted the sub-optimality of the naive approach and proposed a new algorithm with (nearly) linear dependence on $k$ through connections to TV-indistinguishability.
Concretely, they first (non-replicably) estimate all values $\mu_i$ with error $\eps \rho / \sqrt{k}$. Next, inspired by the Gaussian mechanism
from DP, they add Gaussian noise to the estimate in order to ensure that the output distribution has the TV-indistinguishability property.
Finally, they use correlated sampling to transform this back to a replicable algorithm.

Subsequent work of \cite{hopkins2024replicability} also studied the $k$-coin problem. They showed that this problem is part of a broader class for which replicable estimation is equivalent to isoperimetric tiling. They then used the existence of such a tiling to obtain an algorithm with nearly linear scaling in $k$ (see Theorem 1.9 in their paper).

\Cref{thm:main} provides an alternative, and in our view much simpler, approach for the problem. 
Specifically, we set $\cT_i$ to be the problem of estimating $\phi_i$ with additive error $\alpha$. The sample complexity of this problem is bounded by the aforementioned work of \cite{impagliazzo2022reproducibility}. 
Applying \Cref{thm:main}, we immediately obtain an algorithm where the dependence on $k$ is (nearly) linear as opposed to quadratic.
This is formalized in the following theorem.

\begin{theorem}\label{thm:multi_dim_mean}
    Let $\phi_1, \dots, \phi_k$ be $k$ statistical queries where $\phi_i: \cX \to [0, 1]$.
    For an input distribution $D$, define
    $\mu_{D} \in [0, 1]^k$ as
    $\mu_{D, i} = \Exu{X \sim D}{\phi_i(X)}$.
    For any $\rho, \alpha <1/2$ and $\beta < \rho$, there exists a $\rho$-replicable algorithm for finding a vector $v$ satisfying $\norm{v - \mu_{D}}_{\infty} \le \alpha$ with failure probability $\beta$ and sample complexity 
    $\wtilde{O}(\frac{k\polylog(1/\beta)}{\rho^2\alpha^2})$.
\end{theorem}
\begin{proof}
    We first observe that for the special case of $k=1$,
    the problem can be solved with a sample complexity of
    $\wtilde{O}(\frac{\log(1/\beta)}{\rho^2\alpha^2})$ as
    shown by \cite{impagliazzo2022reproducibility} (see Theorem 2.3 in their paper).
    Note that the output space can be made finite; one can first estimate the value $\Ex{\phi_i}$ with error $\alpha/2$ and then round the answer to the nearest integer multiple of $\alpha / 2$. 
    This does not affect replicability as the rounding procedure is deterministic.
    The sample complexity also remains the same up to constant factors. 

    We next solve the general case by reducing to the $k=1$ case via
    \Cref{thm:main}.
    Set $\beta_0$ as
    \begin{align*}
        \beta_0 = \rbr{\rbr{\frac{c\beta}{k\log(k/\rho)\log(1/\beta)}}^2}
    \end{align*}
    for some value of $c < 1/2$ to be chosen later. 
    We observe that
    \begin{align}
        \log(1/\beta_0) \le 
        \Theta\rbr{\log\rbr{\frac{k\log(k/\rho)\log(1/\beta)}{\beta}} + \log\rbr{\frac{1}{c}}}
        \le
        \Theta\rbr{\log(k/\rho) +\log(1/\beta) + \log(1/c)}
        .
        \label{eq:1ycnw8}
    \end{align}
    For simplicity, we use $\mu_i$ instead of $\mu_{D, i}$ throughout the rest of the proof.
    Let $\cT_i$ denote the problem of estimating $\mu_i$ with error $\alpha$ (i.e., finding some $v_i$ satisfying $\abs{\mu_i - v_i} \le \alpha$).
    Let $n_0$ denote the sample complexity of this problem for a $\rho'$-replicable algorithm with
    failure probability $\beta_0$ where $\rho'=0.0001$.
    Since each $\cT_i$ corresponds to an instance of the problem with $d=1$ and $\rho'$ is constant,
    we have 
    $n_0 \le \wtilde{O}(\frac{\log(1/\beta_0)}{\alpha^2})$.
    It is clear that finding a vector $\nu$ satisfying
    $\norm{\nu - \mu_{D}}_{\infty} \le \alpha$
    corresponds to the composition of all $\cT_i$.
    \Cref{thm:main} gives an algorithm for solving this problem.
    The failure probability of the algorithm is bounded with
    \begin{align*}
        &\Theta\rbr{k\sqrt{\beta_0}\log(k/\rho\beta_0)}
        \\&=
        \Theta\rbr{k\sqrt{\beta_0}\rbr{\log(k/\rho) + \log(1/\beta_0)}}
        \\&\le
        \Theta\rbr{k\sqrt{\beta_0}\rbr{\log(k/\rho) + \log(1/\beta) + \log(1/c)}}
        &\EqComment{\Cref{eq:1ycnw8}}
        \\&\le 
        \Theta\rbr{k\sqrt{\beta_0}\log(k/\rho)\log(1/\beta)\log(1/c)}
        &\EqComment{Since $\log(k/\rho), \log(1/\beta), \log(1/c) \ge \Omega(1)$}
        \\&=
        \Theta\rbr{k\frac{c\beta}{k\log(k/\rho)\log(1/\beta)}\log(k/\rho)\log(1/\beta)\log(1/c)}
        &\EqComment{Definition of $\beta_0$}
        \\&=
        \Theta\rbr{c\beta \log(1/c)}
        &\EqComment{Simplifying}
    \end{align*}
    Letting $c$ be a small enough absolute constant, the above quantity is at most $\beta$ as required.
    It follows that
    \begin{align*}
        n_0 \le 
        \wtilde{O}\rbr{\frac{\log(1/\beta_0)}{\alpha^2}}
        \le 
        \wtilde{O}\rbr{\frac{\log(k/\rho\beta)}{\alpha^2}},
    \end{align*}
    which means
    the sample complexity of the final algorithm is at most
    \begin{align*}
        \frac{n_0k}{\rho^2} \polylog(k\rho/\beta_0)
        = \wtilde{O}\rbr{\frac{k\polylog(1/\beta)}{\rho^2\alpha^2}}
        ,
    \end{align*}
    as claimed.
\end{proof}
\subsection{Application: Parallel Best Arm}
\label{sec:app_best_arm}
The recent work of \cite{hopkins2025generative} studies
replicable reinforcement learning. As part of their analysis, they focus on solving parallel instances of the best arm problem. In the best arm problem, we are given i.i.d samples from $|A|$ reward distributions (i.e., arms), and the goal is to (replicably) output an arm with (approximately) maximum reward. 
Formally, each $a \in A$ has a corresponding distribution $r_{a}$ over $[0, 1]$. Denoting $\mu_{a} = \Ex{r_a}$, the goal is to choose some $a$ satisfying
$\mu_a \ge \max_{a' \in A} \mu_{a'} - \alpha$.
They show that this problem can be solved with $O(\frac{1}{\rho^2\eps^2}\log^3(|A|/\beta))$ samples from each arm where $\beta$ denotes the failure probability.
Their approach relies on using correlated sampling together with the exponential mechanism from differential privacy.
They then discuss the generalization of this to solving $|S|$ parallel instances of the same problem.
They specifically note that the naive composition approach leads to a scaling of $|S|^2$ and to avoid this, they use a more nuanced approach by directly sampling from the joint exponential mechanism on vectors of actions $(a_1, \dots, a_S) \in [A]^S$.

In this section, we show that the generalization can actually be done in a black-box manner using \Cref{thm:main}.
\begin{theorem}
    For any parameters $\beta \le \rho \le 1/2$ and $\alpha > 0$, there exists a $\rho$-replicable algorithm that solves $k$ parallel instances of the best-arm problem over a set $A$ with error at most $\alpha$ and failure probability at most $\beta$, using $\wtilde{O}(k\alpha^{-2}\rho^{-2}\polylog(|A|/\beta))$ samples per instance-arm pair.
\end{theorem}
\begin{proof}
    The same proof as \Cref{thm:multi_dim_mean} holds.
    Formally, set $\beta_0$ as
    \begin{align*}
        \rbr{
            \frac{c\beta}{k \log(k/\rho)\log(1/\beta)}
        }^2
    .
    \end{align*}
    Let $n_0$ be the sample complexity of solving a single instance of the problem with replicability parameter $0.0001$ and failure probability $\beta_0$.
    Define $a = |A|$.
    We have
    $n_0 \le O(\frac{1}{\alpha^2}\log^3(a/\beta_0))$.
    As before we have
    $\log(1/\beta_0) \le \Theta(\log(k/\rho\beta) + \log(1/c))$.
    It follows that the failure probability is bounded with
    $\Theta(c\beta \log(1/c))$ as before.
    Setting $c$ to be small enough, this ensures that
    the failure probability is at most $\beta$.
    Therefore,
    $n_0 \le O(\frac{1}{\alpha^2} \log^3(a/\beta_0))$, which is at most $O(\frac{1}{\alpha^2}\log^3(ka/\rho\beta))$.
    It follows that the total sample complexity is bounded as
    \begin{align*}
        \frac{n_0 k}{\rho^2}
         \polylog(k\rho/\beta_0)
         \le \wtilde{O}\rbr{
            \frac{k}{\alpha^2\rho^2}
            \polylog(a/\beta)
         }
         .
    \end{align*}
\end{proof}

\subsection{Heterogeneous Composition for Perfect Generalization}
\label{sec:compose}

In this section, we analyze the \emph{adaptive} composition of $k$ perfectly generalizing algorithms with heterogeneous parameters $(\gamma_i, \eps_i, \delta_i)$. Our main result is \Cref{lm:pg_compose_het} provided below.
In \Cref{sec:proof_compose_easier}, we use this to prove \Cref{lm:pg_compose} which is easier to work with and is used in \Cref{sec:main_thing}.

Throughout, for any two random variables $U, V$, we use the shorthand
$U \approx_{\eps, \delta} V$ to mean
that for all measurable $O$:
\begin{align*}
    e^{-\eps}\rbr{
        \Pr{V \in O}
        -\delta
    }
    \le
    \Pr{U \in O}
    \le
    e^{\eps} \Pr{V \in O} + \delta
    .
\end{align*}

\paragraph{Adaptive composition model.}
We consider a sequence of $k$ adaptively chosen algorithms
$\cA_i : \cX^n \times \cY_1 \times \dots \times \cY_{i-1} \to \cY_i$, $i \in [k]$.
Given a dataset $X \sim D^n$, the composition proceeds iteratively:
at round $i$, we compute $Z_i = \cA_i(X, Z^{i-1})$ where $Z^{i-1} = (Z_1, \dots, Z_{i-1})$ is the transcript of previous outputs (with $Z^0 = \perp$).
We say that each $\cA_i$ is $(\gamma_i, \eps_i, \delta_i)$-PG in the adaptive sense if,
for every distribution $D$ over $\cX$ and every fixed prefix $z^{i-1} \in \cY_1 \times \dots \times \cY_{i-1}$,
there exists a distribution $\cW_i(z^{i-1})$ over $\cY_i$ (the \emph{oracle} for prefix $z^{i-1}$) such that,
with probability at least $1 - \gamma_i$ over $S \sim D^n$:
\begin{align*}
    \cA_i(S, z^{i-1}) \approx_{\eps_i, \delta_i} \cW_i(z^{i-1}).
\end{align*}
Define the \emph{oracle transcript} $\wtilde{Z}^k = (\wtilde{Z}_1, \dots, \wtilde{Z}_k)$
by sampling $\wtilde{Z}_i \sim \cW_i(\wtilde{Z}^{i-1})$ iteratively,
using the oracle prefix $\wtilde{Z}^{i-1}$ (not the real prefix $Z^{i-1}$).
Since $\cW_i(z^{i-1})$ depends only on $D$ and $z^{i-1}$, the oracle transcript $\wtilde{Z}^k$ is independent of $X$.
We say that the adaptive composition $\cAcomp$ is $(\gamma^*, \eps^*, \delta^*)$-PG if, with probability at least $1 - \gamma^*$ over $X \sim D^n$, the real transcript satisfies $\cAcomp(X) \approx_{\eps^*, \delta^*} \wtilde{Z}^k$, where $\wtilde{Z}^k$ serves as the oracle for the composed algorithm.

\begin{lemma}\label{lm:pg_compose_het}
    Let $\cA_1, \dots, \cA_k$ be $k$ adaptively chosen algorithms such that
    $\cA_i(\cdot, z^{i-1})$ is
    $(\gamma_i, \eps_i, \delta_i)$-PG for every prefix $z^{i-1}$.
    Assume that
    $\eps_i \in (0, 1]$,
    $\delta_i \in (0, \eps_i / 50]$ and $\gamma_i \in (0, 1)$.
    For any $\delta' \in (0, 1/2)$, the adaptive composition
    $\cAcomp = (\cA_1, \dots, \cA_k)$ is $(\gamma^*, \eps^*, \delta^*)$-PG with
    \begin{align*}
        \eps^* = 3\eps^{(k)}
        ,
        \quad
        \delta^* = \gamma^*
        = 5\sqrt{
            \frac{\delta^{(k)}}{\min\cbr{\eps^{(k)}, 1}}
        }
    \end{align*}
    where
    \begin{gather*}
        \eps^{(j)} =
        2\sqrt{2\ln(1/\delta') \rbr{\sum_{\ell \le j} \eps_\ell^2}}
        +
        \sum_{\ell \le j}
        \rbr{
        \psi_\ell(\delta_\ell)
        +
        2\eps_\ell\rbr{\frac{e^{2\eps_\ell}}{1-\hat{\delta}_\ell} - 1}
        }
        ,
        \\
        \delta^{(j)} = j\delta' + \sum_{\ell \le j} \widehat{\delta}_\ell + \sum_{\ell \le j} e^{\eps^{(\ell-1)}} \gamma_{\ell}
        ,
        \\
        \hat{\delta}_\ell =
        \frac{2\delta_\ell}{1 - e^{-\eps_\ell}},
        \qquad
        \psi_\ell(\delta_\ell) =
        \delta_\ell\rbr{2e^{\eps_\ell} + 1}
        +
        2\delta_\ell^2\rbr{\frac{2e^{2\eps_\ell}}{e^{\eps_\ell} - 1} + 1}^2
        ,
    \end{gather*}
    with $\eps^{(0)} = 0$.
\end{lemma}

We define the key quantities used in the proof.
For any $i$ and
$z^i \in \cY_1 \times \dots \times \cY_i$,
define
\begin{align*}
    f_i(x, z^i) = \ln\rbr{\frac{\Pr{\cA_i(x, z^{i-1}) = z_i}}{\Pr{\cW_i(z^{i-1}) = z_i}}}
    .
\end{align*}
Define
the sets
$\cC_i$ as
\begin{align*}
    \cC_i =
    \cbr{
        (x, z^{i-1}):
        \cA_i(x, z^{i-1}) \approx_{\eps_i, \delta_i}
        \cW_i(z^{i-1})
    }.
\end{align*}
Next, we define
$\cE_i$ as
\begin{align*}
    \cE_i =
    \cbr{
        (x, z^i):
        \abs{f_i(x, z^i)} \le 2\eps_i
    }
    .
\end{align*}
Define the
sets $\cH_i, \cHhat_i$ as
\begin{gather*}
    \cH_i =
    \cbr{
        (x, z^i):
            \forall
            \ell \le i:
            (x, z^{\ell - 1}) \in \cC_\ell
            \text{ and }
            (x, z^{\ell})
            \in \cE_\ell
    }
    ,\\
    \cHhat_i =
    \cbr{
        (x, z^{i-1}):
            \forall
            \ell \le i:
            (x, z^{\ell - 1}) \in \cC_\ell
            \text{ and }
            (x, z^{\ell -1})
            \in \cE_{\ell-1}
    }
    ,
\end{gather*}
where $\cE_0 = \cX^n \times \{\perp\}$.
Note that $(x, z^i) \in \cH_i$ if and only if
$(x, z^{i-1}) \in \cHhat_i$ and $(x, z^i) \in \cE_i$.
Define
$F_j(x, z^j) =
\sum_{\ell \le j}
f_\ell(x, z^\ell)
$
and
\begin{gather*}
    \cG_j =
    \cbr{
    (x, z^j):
    F_j(x, z^j)
    \le \eps^{(j)}
    },
    \qquad
    \cG'_j =
    \cbr{
    (x, z^j):
    F_j(x, z^j)
    \ge -\eps^{(j)}
    }
    .
\end{gather*}

\paragraph{High level sketch.}
We next give the high level sketch of the proof.
The main goal will be to show that
\begin{align}
    (X, Z^{k}) \approx_{\eps^{(k)}, \delta^{(k)}} (X, \widetilde{Z}^{k})
    \label{eq:vec_sim_z_and_z_tilde}
\end{align}
This would imply the lemma using a standard conditioning argument (see \Cref{lm:joint_implies_cond}).
To prove the above, we need to show that
$(X, Z^{k}) \in \cG_k$ and
$(X, \widetilde{Z}^{k}) \in \cG'_k$ each hold with probability at least $1 - \delta^{(k)}$;
this implies \Cref{eq:vec_sim_z_and_z_tilde} by standard techniques (\Cref{lm:G_G_prime_implies_joint}).

The main idea is to first show that, with high probability\footnote{For the purposes of the sketch, we use \enquote{with high probability} informally to denote
probabilities that are close to $1$ for small enough values of $\eps_i, \delta_i$; the exact bounds are provided in the relevant lemma statements.}, we have
$(X, Z^{i-1}), (X, \widetilde{Z}^{i-1}) \in \cHhat_i$ and then use
Azuma's inequality (similar to the proof of advanced composition in DP) to show that
with high probability $(X, Z^{i}) \in \cG_i$ and $(X, \widetilde{Z}^{i}) \in \cG'_i$ (\Cref{lm:h_hat_to_h} and \Cref{lm:h_to_g}).
Showing $(X, \widetilde{Z}^{i-1}) \in \cHhat_i$ is not particularly difficult:
since $X$ and $\widetilde{Z}^{i-1}$ are independent, one can show that
for each $i$ we have $(X, \widetilde{Z}^{i-1}) \in \cC_i$ with probability $1-\gamma_i$ over the randomness of $X$ (\Cref{lm:tilde_z_C_i}).
Additionally, if $(X, \widetilde{Z}^{i-1}) \in \cC_i$ and $(X, \widetilde{Z}^{i-1}) \in \cHhat_i$
then with high probability we have $(X, \widetilde{Z}^i) \in \cE_i$ (\Cref{lm:h_hat_to_h}).
Taking a union bound, allows us to obtain that
$(X, \widetilde{Z}^{i-1}) \in \cHhat_i$ for any $i$ with high probability (\Cref{lm:tilde_z_not_in_chhat}).
This in turn implies that
$(X, \widetilde{Z}^{i}) \in \cG'_i$ with high probability as well (\Cref{lm:tilde_z_in_G_i}).

To show that $(X, Z^{i-1}) \in \cHhat_i$ however, we cannot apply the same reasoning.
The key issue is that since $Z^{i-1}$ itself depends on $X$, we cannot directly use
the Perfect Generalization of $\cA_i$ to argue that $(X, Z^{i-1}) \in \cC_i$.
Instead, we argue this inductively by connecting $(X, Z^{i-1})$ to $(X, \widetilde{Z}^{i-1})$:
we use the fact that (with high probability) $(X, Z^{i-1}) \in \cG_{i-1}$
together with the fact that (with high probability) $(X, \widetilde{Z}^{i-1}) \in \cC_i$ to infer that
$(X, Z^{i-1}) \in \cC_i$ as well.
The former condition ensures that probabilities under $(X, Z^{i-1})$ are upper bounded by probabilities under
$(X, \widetilde{Z}^{i-1})$. The condition itself holds because of induction:
since 
$(X, Z^{i-2}) \in \cHhat_{i-1}$ with high probability we can conclude that $(X, Z^{i-1}) \in \cG_{i-1}$ with high probability (\Cref{lm:h_hat_to_h} and \Cref{lm:h_to_g}).

Formally, the informal induction steps are as follows.
\begin{enumerate}
    \item If $(X, Z^{i-1}) \in \cHhat_i$ w.h.p then 
        $(X, Z^{i}) \in \cH_i$ w.h.p (\Cref{lm:h_hat_to_h}).
    \item If $(X, Z^{i}) \in \cH_i$ w.h.p then $(X, Z^{i}) \in \cG_i$ w.h.p (\Cref{lm:h_to_g}).
    \item If $(X, Z^{i}) \in \cG_i$ then $(X, Z^{i}) \in \cC_{i+1}$ w.h.p (\Cref{lm:z_C_i}).
        Here, we are using the fact that $(X, \widetilde{Z}^{i}) \in \cC_{i+1}$ w.h.p (\Cref{lm:tilde_z_C_i}).
    \item If $(X, Z^{i}) \in \cH_i$ w.h.p and $(X, Z^{i}) \in \cC_{i+1}$ w.h.p then
        $(X, Z^{i}) \in \cHhat_{i+1}$ w.h.p. This follows from the definition of $\cHhat_{i+1}$ and
        $\cH_i$.
\end{enumerate}
The above induction allows us to show that
$(X, Z^{i-1}) \in \cHhat_i$ w.h.p (\Cref{lm:z_not_in_chhat}) and, by extension,
$(X, Z^{i}) \in \cG_i$ (\Cref{lm:z_in_G_i}). Since we already have
$(X, \widetilde{Z}^i) \in \cG'_i$, this implies the claim as desired (\Cref{lm:joint_implies_cond}).

We proceed with a formal proof.
We first prove \Cref{lm:G_G_prime_implies_joint} and \Cref{lm:joint_implies_cond}
which show
that \Cref{lm:pg_compose_het} reduces to proving $(X, Z^k) \in \cG_k$ and $(X, \widetilde{Z}^k) \in \cG'_k$ hold with high probability.

\begin{lemma}\label{lm:G_G_prime_implies_joint}
    Assume that
    \begin{align}
        \Pr{(X, Z^k) \notin \cG_k},\;
        \Pr{(X, \widetilde{Z}^k) \notin \cG'_k}
        \le
        \delta^{(k)}
        .
        \label{eq:to_show_lem_gen_pg}
    \end{align}
    Then
    \Cref{eq:vec_sim_z_and_z_tilde} holds.
\end{lemma}
\begin{proof}
    Observe that
    for any $(x, z^k) \in \cG_k$ we have
    \begin{align*}
        \ln
        \rbr{\frac{\Pr{(X, Z^k) = (x, z^k)}}{\Pr{(X, \wtilde{Z}^k) = (x, z^k)}}}
        &=
        \ln
        \rbr{\frac{\Pr{X\!=\!x} \prod_{\ell=1}^k \Pr{\cA_\ell(x, z^{\ell-1}) = z_\ell}}
             {\Pr{X\!=\!x} \prod_{\ell=1}^k \Pr{\cW_\ell(z^{\ell-1}) = z_\ell}}}
        \\&=
        \sum_{\ell \le k}
        f_\ell(x, z^\ell)
        = F_k(x, z^k)
        \le
        \eps^{(k)}
        ,
    \end{align*}
    where the first equality uses the chain rule for both processes
    (for $Z^k$: $\Pr{Z_\ell = z_\ell \mid X\!=\!x, Z^{\ell-1}\!=\!z^{\ell-1}} = \Pr{\cA_\ell(x, z^{\ell-1}) = z_\ell}$;
    for $\wtilde{Z}^k$: $\Pr{\wtilde{Z}_\ell = z_\ell \mid \wtilde{Z}^{\ell-1}\!=\!z^{\ell-1}} = \Pr{\cW_\ell(z^{\ell-1}) = z_\ell}$, using independence of $X$ and $\wtilde{Z}^k$).
    It follows that for any $(x, z^k) \in \cG_k$:
    \begin{align}
        \Pr{(X, Z^k) = (x, z^k)}
        \le e^{\eps^{(k)}}
        \Pr{(X, \wtilde{Z}^k) = (x, z^k)}
        .
        \label{eq:bound_pr_g_k}
    \end{align}
    Therefore, for any event $E$,
    \begin{align*}
        \Pr{(X, Z^k) \in E}
        &\le
        \Pr{(X, Z^k) \notin \cG_k}
        +
        \sum_{(x, z^k) \in E \cap \cG_k}
        \Pr{(X, Z^k) = (x, z^k)}
        \\&\le
        \delta^{(k)}
        +
        e^{\eps^{(k)}}
        \Pr{(X, \wtilde{Z}^k) \in E}
        .
    \end{align*}
    Similarly, for any $(x, z^k) \in \cG'_k$ we have $F_k(x, z^k) \ge -\eps^{(k)}$, giving
    $\Pr{(X, \wtilde{Z}^k) = (x, z^k)} \le e^{\eps^{(k)}} \Pr{(X, Z^k) = (x, z^k)}$,
    and the same argument yields
    $\Pr{(X, \wtilde{Z}^k) \in E}
    \le \delta^{(k)} + e^{\eps^{(k)}} \Pr{(X, Z^k) \in E}$.
    Therefore,
    $(X, Z^k) \approx_{\eps^{(k)}, \delta^{(k)}} (X, \wtilde{Z}^k)$.
\end{proof}
\begin{lemma}\label{lm:joint_implies_cond}
    Assume that \Cref{eq:vec_sim_z_and_z_tilde} holds.
    Then the adaptive composition $\cAcomp$ is $(\gamma^*, \eps^*, \delta^*)$-PG with
    $\eps^* = 3\eps^{(k)}$ and $\delta^* = \gamma^* = 5\sqrt{\delta^{(k)}/\min\cbr{\eps^{(k)}, 1}}$.
\end{lemma}
To prove this lemma, we use the following result by \cite{kasiviswanathan2014semantics}.
\begin{lemma}[\cite{kasiviswanathan2014semantics}]
    \label{lm:kasivi}
    Assume that
    $(U, V) \approx_{\eps, \delta} (U', V')$.
    Let $U \mid_{V=t}$ denote
    the random variable $U$ conditioned on the event $V=t$.
    For any
    $\hat{\delta} > 0$ we have
    \begin{align*}
        \Pru{t \sim V}{U \mid_{V=t} \approx_{3\eps, \hat{\delta}}
        U'\mid_{V'=t}
        }
        \ge 1- 2\delta/\hat{\delta} - 2\delta/(1-e^{-\eps})
        .
    \end{align*}
\end{lemma}
\begin{proof}[Proof of \Cref{lm:joint_implies_cond}]
    We apply \Cref{lm:kasivi}
    with $(U, V) = (Z^k, X)$ and $(U', V') = (\wtilde{Z}^k, X)$
    and $\eps = \eps^{(k)}$, $\delta = \delta^{(k)}$.
    Since \Cref{eq:vec_sim_z_and_z_tilde} gives
    $(X, Z^k) \approx_{\eps^{(k)}, \delta^{(k)}} (X, \wtilde{Z}^k)$,
    we conclude that for any $\hat{\delta} > 0$,
    \begin{align*}
        \Pru{X}{Z^k \mid_{X} \approx_{3\eps^{(k)}, \hat{\delta}}
        \wtilde{Z}^k}
        \ge
        1 -
        \frac{2\delta^{(k)}}{\hat{\delta}}
         - \frac{2\delta^{(k)}}{1-e^{-\eps^{(k)}}}
         .
    \end{align*}
    We use the bound
    $1-e^{-x} \ge \frac{2}{5}\min(x, 1)$ for all $x > 0$.
    Setting
    $\hat{\delta} = 2\sqrt{\delta^{(k)} \cdot \min\cbr{\eps^{(k)}, 1}/5}$,
    we obtain
    \begin{align*}
        \frac{2\delta^{(k)}}{\hat{\delta}}
        = \sqrt{\frac{5\delta^{(k)}}{\min\cbr{\eps^{(k)}, 1}}}
        ,\qquad
        \frac{2\delta^{(k)}}{1-e^{-\eps^{(k)}}}
        \le
        \frac{5\delta^{(k)}}{\min\cbr{\eps^{(k)}, 1}}
        .
    \end{align*}
    We may assume $5\delta^{(k)} \le \min\cbr{\eps^{(k)}, 1}$,
    since otherwise $\delta^* > 1$ and the conclusion holds trivially.
    Under this assumption,
    the second term is at most $\sqrt{5\delta^{(k)}/\min\cbr{\eps^{(k)},1}}$,
    so the total failure probability is at most
    $2\sqrt{5\delta^{(k)}/\min\cbr{\eps^{(k)},1}} \le 5\sqrt{\delta^{(k)}/\min\cbr{\eps^{(k)},1}}$.
    Since $\hat{\delta} \le 5\sqrt{\delta^{(k)}/\min\cbr{\eps^{(k)},1}}$ as well,
    we set
    $\delta^* = \gamma^* = 5\sqrt{\delta^{(k)}/\min\cbr{\eps^{(k)},1}}$ and $\eps^* = 3\eps^{(k)}$.
\end{proof}

It remains to show 
\Cref{eq:to_show_lem_gen_pg}.
We start with the following lemma from 
\cite{kasiviswanathan2014semantics}.
\begin{lemma}\label{lm:help_kasi}
    If $X \approx_{\eps, \delta} Y$ then,
    defining
    $\cE = 
    \cbr{x: \abs{\ln\rbr{\frac{\Pr{X=x}}{\Pr{Y=x}}}} \le 2\eps},
    $
    we have
    $\Pr{X \in \cE} \ge 1-\frac{2\delta}{1-e^{-\eps}}$.
\end{lemma}
The following two lemmas are stated for a fixed $(x, z^{i-1}) \in \cHhat_i$
and bound the behavior of both the real process ($Z_i \sim \cA_i(x, z^{i-1})$)
and the oracle process ($Z_i \sim \cW_i(z^{i-1})$).
\begin{lemma}\label{lm:h_hat_to_h}
    For any $(x, z^{i-1}) \in \cHhat_i$:
    \begin{enumerate}
        \item (Real process.) If $Z_i \sim \cA_i(x, z^{i-1})$, then
        $\Pr{(x, (z^{i-1}, Z_i)) \notin \cE_i} \le \hat{\delta}_i$.
        \item (Oracle process.) If $Z_i \sim \cW_i(z^{i-1})$, then
        $\Pr{(x, (z^{i-1}, Z_i)) \notin \cE_i} \le \hat{\delta}_i$.
    \end{enumerate}
    In particular, conditioning on the prefix being $(x, z^{i-1}) \in \cHhat_i$,
    the probability that $(x, z^{i-1}, Z_i) \notin \cH_i$ is at most $\hat{\delta}_i$
    for both the real and oracle processes.
\end{lemma}
\begin{proof}
    By definition of $\cHhat_i$,
    $(x, z^{i-1}) \in \cC_i$, so $\cA_i(x, z^{i-1}) \approx_{\eps_i, \delta_i} \cW_i(z^{i-1})$.
    For part (1), applying \Cref{lm:help_kasi} with $\cA_i(x, z^{i-1})$ and $\cW_i(z^{i-1})$
    gives
    $\Pr{(x, (z^{i-1}, Z_i)) \in \cE_i} \ge 1 - \frac{2\delta_i}{1 - e^{-\eps_i}} = 1 - \hat{\delta}_i$.
    For part (2), note that $\approx_{\eps, \delta}$ is symmetric (the upper bound of one direction gives the lower bound of the other), so $\cW_i(z^{i-1}) \approx_{\eps_i, \delta_i} \cA_i(x, z^{i-1})$.
    Moreover, $\cE_i$ is unchanged under swapping since $\abs{\ln(a/b)} = \abs{\ln(b/a)}$.
    Applying \Cref{lm:help_kasi} with $X = \cW_i(z^{i-1})$ and $Y = \cA_i(x, z^{i-1})$ gives $\Pr{(x, (z^{i-1}, Z_i)) \in \cE_i} \ge 1 - \hat{\delta}_i$ for $Z_i \sim \cW_i(z^{i-1})$.
    The ``in particular'' statement follows since $(x, z^{i-1}, Z_i) \notin \cH_i$
    implies $(x, (z^{i-1}, Z_i)) \notin \cE_i$.
\end{proof}

\begin{lemma}\label{lm:bound_expect}
    Assume $\eps_i \in (0, 1]$ and $\delta_i \in (0, \eps_i/50]$.
    For any $(x, z^{i-1}) \in \cHhat_i$, define
    $\rho_i = 2\eps_i\rbr{e^{2\eps_i}/(1-\hat{\delta}_i) - 1} + \psi_i(\delta_i)$.
    \begin{enumerate}
        \item (Real process.) If $Z_i \sim \cA_i(x, z^{i-1})$, then
        $\Ex{f_i(x, (z^{i-1}, Z_i)) \mid (x, (z^{i-1}, Z_i)) \in \cE_i}
        \le \rho_i$.
        \item (Oracle process.) If $Z_i \sim \cW_i(z^{i-1})$, then
        $\Ex{-f_i(x, (z^{i-1}, Z_i)) \mid (x, (z^{i-1}, Z_i)) \in \cE_i}
        \le \rho_i$.
    \end{enumerate}
\end{lemma}
\begin{proof}
    We only prove Part (1).
    Part (2) follows from the same proof as part (1) with the roles of $\cA_i(x, z^{i-1})$ and $\cW_i(z^{i-1})$ swapped.
    Fix $(x, z^{i-1}) \in \cHhat_i$.
    Write $P = \cA_i(x, z^{i-1})$ and $Q = \cW_i(z^{i-1})$ for brevity, so that $P \approx_{\eps_i, \delta_i} Q$.
    We abuse notation and write $\cE_i$ for the
    set $\cbr{z_i: (x, (z^{i-1}, z_i)) \in \cE_i}$.

    \begin{claim}\label{cl:Q_weighted_bound}
    $\sum_{z_i \in \cE_i}
        \Pr{Q = z_i}
        \ln\rbr{
        \frac{\Pr{P = z_i}}{\Pr{Q=z_i}}
        }
        \le
        \psi_i(\delta_i)$.
    \end{claim}
    \begin{proof}[Proof of \Cref{cl:Q_weighted_bound}]
    Observe that
    \begin{align*}
        &\sum_{z_i \in \cE_i}
        \Pr{Q = z_i}
        \ln\rbr{
        \frac{\Pr{P = z_i}}{\Pr{Q=z_i}}
        }
        \\&=
        \Pr{Q \in \cE_i}
        \sum_{z_i \in \cE_i}
        \frac{\Pr{Q = z_i}}{\Pr{Q \in \cE_i}}
        \ln\rbr{
        \frac{\Pr{P = z_i}}{\Pr{Q=z_i}}
        }  
        &\EqComment{Cancelling product}
        \\&\le
        \Pr{Q \in \cE_i}
        \ln(
        \sum_{z_i \in \cE_i}
         \frac{\Pr{Q = z_i}}{\Pr{Q \in \cE_i}}
         \frac{\Pr{P = z_i}}{\Pr{Q=z_i}}
        )
        &\EqComment{Jensen's inequality}
        \\&\le
        \ln\rbr{
        \sum_{z_i \in \cE_i}
         \frac{\Pr{P = z_i}}{\Pr{Q \in \cE_i}}
        }
        &\EqComment{Since $\Pr{Q \in \cE_i} \le 1$; see below}
        \\&=
        \ln\rbr{
         \frac{\Pr{P \in \cE_i}}{\Pr{Q \in \cE_i}}
        }
        &\EqComment{Total probability}
        \\&=
        \ln\rbr{
         \frac{1 - \Pr{P \notin \cE_i}}{1 - \Pr{Q \notin \cE_i}}
        }
        .
    \end{align*}
    (The ``see below'' step uses $\Pr{Q \in \cE_i} \le 1$: if $\Pr{P \in \cE_i} \ge \Pr{Q \in \cE_i}$ then the log is non-negative and the inequality holds; otherwise the post-Jensen expression is already negative, so the claim $\le \psi_i(\delta_i)$ holds trivially.)
    Observe however that,
    since $P \approx_{\eps_i, \delta_i} Q$, we have
    $\Pr{Q \notin \cE_i} \le e^{\eps_i}\Pr{P \notin \cE_i} + \delta_i$.
    Define
    $q := \Pr{P \notin \cE_i}$. 
    We claim that $e^{\eps_i}q + \delta_i \le 1/2$.
    This is because
    \begin{align*}
        q 
         &\le
         \hat{\delta}_i
         &\EqComment{By \Cref{lm:h_hat_to_h}}
         \\&= \frac{2\delta_i}{1 - e^{-\eps_i}}
         &\EqComment{Definition of $\hat{\delta}_i$}
         \\&\le \frac{2\eps_i}{50\rbr{1 - e^{-\eps_i}}}
         &\EqComment{Since $\delta_i \le \eps_i/50$}
         \\&\le 
         \frac{4}{50}
         &\EqComment{Since $1-e^{-x} \ge x/2$ for $x \le 1$ and $\eps_i \le 1$}
         .
    \end{align*}
    Which further implies
    \begin{align*}
        e^{\eps_i} q+ \delta_i \le \frac{4e}{50} + \frac{1}{50} <\frac{1}{4}
        .
    \end{align*}
    It follows that
    \begin{align*}
        \ln\rbr{
         \frac{1 - \Pr{P \notin \cE_i}}{1 - \Pr{Q \notin \cE_i}}
        }
        &\le \ln(\frac{1-q}{1 - e^{\eps_i} q - \delta_i})
        \\&=
        \ln(1-q) - \ln(1 - (e^{\eps_i}q + \delta_i))
        \\&\le 
        -q + e^{\eps_i}q + \delta_i + 2(e^{\eps_i} q+ \delta_i)^2
        \\&=
        \delta_i + 
        (e^{\eps_i} - 1) q
         + 
         2(e^{\eps_i} q+ \delta_i)^2,
    \end{align*}
    where for the second inequality we have used
    $\ln(1-x) \le -x$ which holds for all $x$ and
    $\ln(1-x) \ge -x - 2x^2$ which holds for all $x \le 1/2$.
    Recall however that
    $q \le \hat{\delta_i}= \frac{2\delta_i}{1-e^{-\eps_i}}$.
    Since the above expression is increasing in $q$, it is at most
    \begin{align*}
        \delta_i +
        (e^{\eps_i} - 1) 
        \frac{2\delta_i}{1-e^{-\eps_i}}
        + 
        2(\frac{2\delta_i}{1-e^{-\eps_i}} e^{\eps_i} + \delta_i)^2
        &=
        \delta_i\rbr{1 + \frac{2(e^{\eps_i} - 1)}{1 - e^{-\eps_i}}}
        + 
        2\delta_i^2\rbr{\frac{2e^{2\eps_i}}{e^{\eps_i} - 1} + 1}^2
        \\&=
        \delta_i\rbr{2e^{\eps_i} + 1}
        +
        2\delta_i^2\rbr{\frac{2e^{2\eps_i}}{e^{\eps_i} - 1} + 1}^2
        = \psi_i(\delta_i)
        ,
    \end{align*}
    where for the final equality we have substituted the definition of $\psi_i$
    from the lemma statement.
    Combining the above derivations we obtain
    \begin{align}
        \sum_{z_i \in \cE_i}
        \Pr{Q = z_i}
        \ln\rbr{
        \frac{\Pr{P = z_i}}{\Pr{Q=z_i}}
        }
        \le
        \psi_i(\delta_i)
        \label{eq:bound_with_psi_in_rhs}
    \end{align}
    \end{proof}

    \begin{claim}\label{cl:P_weighted_bound}
    $\sum_{z_i \in \cE_i}
        \Pr{P = z_i \mid P \in \cE_i}
        \ln\rbr{
        \frac{\Pr{P = z_i}}{\Pr{Q = z_i}}
        }
        \le
        \sum_{z_i \in \cE_i}
        \Pr{Q = z_i}
        \ln\rbr{
        \frac{\Pr{P = z_i}}{\Pr{Q = z_i}}
        }+2\eps_i
        \rbr{
        \frac{e^{2\eps_i}}{1 - \hat{\delta_i}} - 1
        }$.
    \end{claim}
    \begin{proof}[Proof of \Cref{cl:P_weighted_bound}]
    We note that,
    \begin{align*}
        &\sum_{z_i \in \cE_i}
        \Pr{P = z_i \mid P \in \cE_i}
        \ln\rbr{
        \frac{\Pr{P = z_i}}{\Pr{Q = z_i}}
        }
        - 
        \sum_{z_i \in \cE_i}
        \Pr{Q = z_i}
        \ln\rbr{
        \frac{\Pr{P = z_i}}{\Pr{Q = z_i}}
        }
        \\&\le 
        \sum_{z_i \in \cE_i}
        \abs{
        \Pr{P = z_i \mid P \in \cE_i}
        - 
        \Pr{Q = z_i}
        }
        \cdot
        \abs{
            \ln\rbr{
            \frac{\Pr{P = z_i}}{\Pr{Q = z_i}}
            }
        }
        \\&\le 
        2\eps_i 
        \sum_{z_i \in \cE_i}
        \abs{
        \Pr{P = z_i \mid P \in \cE_i}
        - 
        \Pr{Q = z_i}
        }
        \\&=
        2\eps_i 
        \sum_{z_i \in \cE_i}
        \Pr{Q = z_i}
        \abs{
        \frac{\Pr{P = z_i \mid P \in \cE_i}}{\Pr{Q = z_i}}
        - 
        1
        }.
    \end{align*}
    where for the second inequality we have used definition of $\cE_i$.
    Substituting
    $\Pr{P = z_i \mid P \in \cE_i} =
    \Pr{P = z_i}/\Pr{P \in \cE_i}
    $
    and using the fact that
    $\abs{\ln\rbr{\frac{\Pr{P = z_i}}{\Pr{Q = z_i}}}} \le 2\eps_i$ for $z_i \in \cE_i$
    we can write this as
    \begin{align*}
        &2\eps_i 
        \sum_{z_i \in \cE_i}
        \Pr{Q = z_i}
        \max \cbr{
        \frac{e^{2\eps_i}}{\Pr{P \in \cE_i}} - 1,
        1 - 
        \frac{e^{-2\eps_i}}{\Pr{P \in \cE_i}}
        }
        \\&\le 
        2\eps_i 
        \rbr{
        \frac{e^{2\eps_i}}{1 - \hat{\delta_i}} - 1
        },
    \end{align*}
    where for the final inequality we have used the facts that
    $\Pr{Q \in \cE_i} \le 1$,
    $\Pr{P \in \cE_i} \ge 1-\hat{\delta}_i$ and
    $\frac{e^{2\eps_i}}{\Pr{P \in \cE_i}} - 1 \ge 
     1 - 
        \frac{e^{-2\eps_i}}{\Pr{P \in \cE_i}}
    $.
    The final inequality is equivalent to
    \begin{align*}
        \frac{e^{-2\eps_i}}{\Pr{P \in \cE_i}}
        + 
        \frac{e^{2\eps_i}}{\Pr{P \in \cE_i}}
        \ge 2,
    \end{align*}
    which holds because of the inequality $x + 1/x \ge 2$ and
    the fact that
    $\Pr{P \in \cE_i} \le 1$.
    It follows that
    \begin{align}
        \notag
        &\sum_{z_i \in \cE_i}
        \Pr{P = z_i \mid P \in \cE_i}
        \ln\rbr{
        \frac{\Pr{P = z_i}}{\Pr{Q = z_i}}
        }
        \\&\le 
        \sum_{z_i \in \cE_i}
        \Pr{Q = z_i}
        \ln\rbr{
        \frac{\Pr{P = z_i}}{\Pr{Q = z_i}}
        }+2\eps_i 
        \rbr{
        \frac{e^{2\eps_i}}{1 - \hat{\delta_i}} - 1
        }
        \label{eq:midway_1021}
    \end{align}
    \end{proof}

    Combining the two claims: since $Z_i \sim P$ conditioned on $(X, Z^{i-1}) = (x, z^{i-1})$, we have
    \begin{align*}
        &\Ex{f_i(X, Z^i) \mid (X, Z^{i-1}) = (x, z^{i-1}), (X, Z^i) \in \cH_i}
        \\&=
        \Ex{f_i(x, (z^{i-1}, Z_i)) \mid (x, (z^{i-1}, Z_i)) \in \cE_i}
        &\EqComment{$Z_i \sim P$; $(x,z^{i-1}) \in \cHhat_i$}
        \\&= 
        \sum_{z_i \in \cE_i}
        \Pr{P = z_i \mid P \in \cE_i}
        \ln\rbr{
        \frac{\Pr{P = z_i}}{\Pr{Q = z_i}}
        }
        &\EqComment{Expanding $\Ex{.}$}
        \\&\le 
        2\eps_i 
        \rbr{
        \frac{e^{2\eps_i}}{1 - \hat{\delta_i}} - 1
        }
        + 
        \sum_{z_i \in \cE_i}
        \Pr{Q = z_i}
        \ln\rbr{
        \frac{\Pr{P = z_i}}{\Pr{Q=z_i}}
        }
        &\EqComment{\Cref{cl:P_weighted_bound}}
        \\&\le
        2\eps_i
        \rbr{
        \frac{e^{2\eps_i}}{1 - \hat{\delta_i}} - 1
        }
        +
        \psi_i(\delta_i)
        &\EqComment{\Cref{cl:Q_weighted_bound}}
        .
    \end{align*}
\end{proof}

We next state the Azuma's inequality which we will use in our proof.
\begin{lemma}\label{lm:azuma}
    Let $X_1, \dots, X_n$ be random variables such 
    that for any $i$ we have
    $\Pr{|X_i| \le c_i} = 1$ and
    for any sequence of values $x^{i-1}$ we have
    $\Ex{X_i \mid X^{i-1} = x^{i-1}} \le \mu_i$. Then
    for any $t > 0$,
    \begin{align*}
        \Pr{\sum_{i=1}^n X_i \ge \sum_{i=1}^n \mu_i + t}
        \le 
        e^{-\frac{t^2}{2\sum_{i=1}^n c_i^2}}
        .
    \end{align*}
\end{lemma}
Note that the standard statement for Azuma's inequality considers a supermartingale sequence $(Y_i)$ such that $\Pr{\abs{Y_i - Y_{i-1}} \le c_i} = 1$ and $\Ex{Y_i \mid Y^{i-1}} \le Y_{i-1}$.
Applying the change of variable $Y_i = \sum_{j=1}^i (X_j - \mu_j)$, this directly yields the above lemma.

\begin{lemma}\label{lm:h_to_g}
    For any $i \in [k]$,
    \begin{align*}
        \Pr{(X, Z^i) \notin \cG_i
        \text{ and }
        (X, Z^{i}) \in \cH_i
        }
        \le 
        \delta'
    \end{align*}
\end{lemma}
\begin{proof}
    Define
    $g_i(x, z^i)$ as
    \begin{align*}
        g_i(x, z^i)
        := 
        \begin{cases}
            f_i(x, z^i) &\quad \text{if}\quad (x, z^i) \in \cH_i
            \\
            0 &\quad \text{otherwise}
        \end{cases}
        .
    \end{align*}
    We claim that
    \begin{align}
        &
        \Ex{g_i(X, Z^i) \mid X=x, Z^{i-1}=z^{i-1}}
        \le 
        2\eps_i 
        \rbr{
        \frac{e^{2\eps_i}}{1 - \hat{\delta_i}} - 1
        }
        + 
        \psi_i(\delta_i)
    \end{align}
    To see why, observe first that by total expectation,
    \begin{equation*}
        \Ex{g_i(X, Z^i) \mid X=x, Z^{i-1}=z^{i-1}}
    \end{equation*}
    is a weighted average of
    \begin{equation*}
        \Ex{g_i(X, Z^i) \mid X=x, Z^{i-1}=z^{i-1}, (X, Z^i) \in \cH_i}
    \end{equation*}
    and
    \begin{equation*}
        \Ex{g_i(X, Z^i) \mid X=x, Z^{i-1}=z^{i-1}, (X, Z^i) \notin \cH_i}
        .
    \end{equation*}
    The latter term equals $0$ by definition of $g_i$.
    It therefore suffices to show that
    \begin{align*}
        \Ex{g_i(X, Z^i) \mid X=x, Z^{i-1}=z^{i-1}, (X, Z^i) \in \cH_i}
        \le 
        2\eps_i 
        \rbr{
        \frac{e^{2\eps_i}}{1 - \hat{\delta_i}} - 1
        }
        + 
        \psi_i(\delta_i)
    \end{align*}
    for all $(x, z^{i-1})$ such that
    $\Pr{(X, Z^i) \in \cH_i \mid X=x, Z^{i-1}=z^{i-1}} > 0$.
    To prove this, observe that
    we can assume 
    $(x, z^{i-1}) \in \cHhat_i$ since otherwise,
    the mentioned probability is $0$.
    Since $(x, z^{i-1}) \in \cHhat_i$, invoking \Cref{lm:bound_expect} finishes the proof.

    Define
    \begin{align*}
        \rho_i =
        2\eps_i 
        \rbr{
        \frac{e^{2\eps_i}}{1 - \hat{\delta_i}} - 1
        }
        + 
        \psi_i(\delta_i)
        .
    \end{align*}
    By definition of $\cH_i$, we have
    $\abs{g_i(x, z^i)} \le 2\eps_i$.
    Invoking Azuma's inequality and substituting the definition of $\eps^{(i)}$,
    \begin{align}
        \notag
        \Pr{\sum_{j \le i} g_j(X, Z^j) > \eps^{(i)}}
        &\le
        \Pr{\sum_{j \le i} g_j(X, Z^j) > \sum_{j \le i} \rho_j + 2\sqrt{2\sum_{j} \eps_j^2 \ln(1/\delta')}}
        \notag
        \\&\le
        e^{-\frac{8\sum_{j \le i} \eps_j^2 \ln(1/\delta')}{8\sum_{j \le i} \eps_j^2}}
        \notag
        \\&= 
        \delta'
        \label{eq:rhs_delta_prime}
    \end{align}
    Observe however that,
    \begin{align*}
        \Pr{(X, Z^i) \notin \cG_i \text{ and } (X, Z^i) \in \cH_i}
        &=
        \Pr{\sum_{j\le i} f_j(X, Z^j) > \eps^{(i)} \text{ and } (X, Z^i) \in \cH_i}
        \\&=
        \Pr{\sum_{j\le i} g_j(X, Z^j) > \eps^{(i)} \text{ and } (X, Z^i) \in \cH_i}
        \\&\le
        \Pr{\sum_{j\le i} g_j(X, Z^j) > \eps^{(i)}}
    \end{align*}
    where the first equality follows from the definition of $\cG_i$ and the second equality follows from
    the definition of $g_i$.
    Plugging \Cref{eq:rhs_delta_prime} finishes the proof.
\end{proof}
The same result holds for the oracle transcript $\wtilde{Z}$ with $\cG'_i$ in place of $\cG_i$.
\begin{lemma}\label{lm:h_to_g_tilde}
    For any $i \in [k]$,
    \begin{align*}
        \Pr{(X, \wtilde{Z}^i) \notin \cG'_i
        \text{ and }
        (X, \wtilde{Z}^{i}) \in \cH_i
        }
        \le
        \delta'
    \end{align*}
\end{lemma}
\begin{proof}
    The proof is identical to that of \Cref{lm:h_to_g},
    replacing $Z$ with $\wtilde{Z}$ and $f_i$ with $-f_i$ throughout.
    In particular, define $\tilde{g}_i(x, z^i) = -f_i(x, z^i)$ if $(x, z^i) \in \cH_i$
    and $0$ otherwise. By \Cref{lm:bound_expect} (Part (2)),
    $\Ex{\tilde{g}_i(X, \wtilde{Z}^i) \mid X=x, \wtilde{Z}^{i-1}=z^{i-1}} \le \rho_i$
    for any $(x, z^{i-1}) \in \cHhat_i$.
    Since $\abs{\tilde{g}_i} \le 2\eps_i$, Azuma's inequality gives
    $\Pr{\sum_{j \le i} \tilde{g}_j(X, \wtilde{Z}^j) > \eps^{(i)}} \le \delta'$.
    The event $\{(X, \wtilde{Z}^i) \notin \cG'_i \text{ and } (X, \wtilde{Z}^i) \in \cH_i\}$
    implies $\sum_{j \le i} (-f_j(X, \wtilde{Z}^j)) > \eps^{(i)}$
    (by definition of $\cG'_i$), which equals $\sum_{j \le i} \tilde{g}_j(X, \wtilde{Z}^j) > \eps^{(i)}$
    on $\{(X, \wtilde{Z}^i) \in \cH_i\}$.
\end{proof}

We next claim the following.
\begin{lemma}\label{lm:tilde_z_C_i}
        For any $i \in [k]$,
    $\Pr{(X, \wtilde{Z}^{i-1}) \notin \cC_i} \le \gamma_i$.
\end{lemma}
\begin{proof}
        By the PG assumption on $\cA_i$, for any fixed prefix $z^{i-1}$,
        $\Pru{X}{(X, z^{i-1}) \notin \cC_i} \le \gamma_i$.
        Since $X \perp \wtilde{Z}^{i-1}$ (the oracle transcript is independent of the dataset),
        marginalizing over $\wtilde{Z}^{i-1}$ preserves this bound.
\end{proof}

\begin{lemma}\label{lm:tilde_z_not_in_chhat}
    For any $i \in [k]$,
    $\Pr{(X, \wtilde{Z}^{i-1}) \notin \cHhat_i}
    \le \sum_{\ell \le i} \gamma_\ell + \sum_{\ell < i} \hat{\delta}_\ell$.
\end{lemma}
\begin{proof}
    The event $(X, \wtilde{Z}^{i-1}) \notin \cHhat_i$ implies either
    $(X, \wtilde{Z}^{\ell-1}) \notin \cC_\ell$ for some $\ell \le i$,
    or $(X, \wtilde{Z}^\ell) \notin \cE_\ell$ with $(X, \wtilde{Z}^{\ell-1}) \in \cHhat_\ell$ for some $\ell < i$.
    The former has probability $\le \gamma_\ell$ by \Cref{lm:tilde_z_C_i};
    the latter has probability $\le \hat{\delta}_\ell$ by \Cref{lm:h_hat_to_h}.
    A union bound gives the result.
\end{proof}

As outlined in the sketch, using induction, this allows us to obtain the following.
\begin{lemma}\label{lm:z_C_i}
    For any $i \in [k]$,
    \begin{align*}
        \Pr{(X, Z^{i-1}) \notin \cC_i,\; (X, Z^{i-1}) \in \cG_{i-1}}
        \le e^{\eps^{(i-1)}} \gamma_i,
    \end{align*}
    where $\cG_0$ denotes the entire space (so the condition is vacuous for $i=1$).
\end{lemma}
\begin{proof}
    For fixed $x$, on the event $(x, Z^{i-1}) \in \cG_{i-1}$ we have
    $F_{i-1}(x, Z^{i-1}) \le \eps^{(i-1)}$, i.e., the log-likelihood ratio of
    $Z^{i-1} \mid X\!=\!x$ versus $\wtilde{Z}^{i-1}$ is bounded pointwise. Therefore,
    for any event $E$ over prefixes:
    \begin{align*}
        \Pr{Z^{i-1} \in E,\; (x, Z^{i-1}) \in \cG_{i-1} \mid X=x}
        \le e^{\eps^{(i-1)}} \Pr{\wtilde{Z}^{i-1} \in E}.
    \end{align*}
    Applying this with $E = \cbr{z^{i-1}: (x, z^{i-1}) \notin \cC_i}$:
    \begin{align*}
        \Pr{(x, Z^{i-1}) \notin \cC_i,\; (x, Z^{i-1}) \in \cG_{i-1} \mid X=x}
        \le e^{\eps^{(i-1)}} \Pr{(x, \wtilde{Z}^{i-1}) \notin \cC_i}.
    \end{align*}
    Taking expectation over $X$ and using $X \perp \wtilde{Z}^{i-1}$:
    \begin{align*}
        \Pr{(X, Z^{i-1}) \notin \cC_i,\; (X, Z^{i-1}) \in \cG_{i-1}}
        \le e^{\eps^{(i-1)}} \Pr{(X, \wtilde{Z}^{i-1}) \notin \cC_i}
        \le e^{\eps^{(i-1)}} \gamma_i,
    \end{align*}
    where the last step uses \Cref{lm:tilde_z_C_i}.
\end{proof}

\begin{lemma}\label{lm:z_not_in_chhat}
    For any $j \in [k]$,
    \begin{align*}
        \Pr{(X, Z^j) \notin \cH_j}
        \le
        \sum_{\ell \le j} \Pr{(X, Z^{\ell-1}) \notin \cC_\ell}
        + \sum_{\ell \le j} \hat{\delta}_\ell.
    \end{align*}
\end{lemma}
\begin{proof}
    Consider the sequential process generating $Z_1, \ldots, Z_j$.
    At step $\ell$: first check $(X, Z^{\ell-1}) \in \cC_\ell$; if not, we call this
    a \emph{type-1 failure}.
    Otherwise, if no prior failure occurred, $(X, Z^{\ell-1}) \in \cHhat_\ell$,
    so by \Cref{lm:h_hat_to_h},
    $\Pr{(X, Z^\ell) \notin \cE_\ell,\; (X, Z^{\ell-1}) \in \cHhat_\ell} \le \hat{\delta}_\ell$;
    we call this a \emph{type-2 failure}.
    The event $(X, Z^j) \notin \cH_j$ implies a type-1 or type-2 failure at some step
    $\ell \le j$; a union bound gives the result.
\end{proof}

Combined with \Cref{lm:h_to_g,lm:h_to_g_tilde,lm:z_C_i,lm:tilde_z_C_i,lm:h_hat_to_h}, this gives us the following.
\begin{lemma}\label{lm:z_in_G_i}
    $\Pr{(X, Z^k) \notin \cG_k}
    \le \delta^{(k)}
    = k\delta' + \sum_{\ell \le k} \hat{\delta}_\ell
    + \sum_{\ell \le k} e^{\eps^{(\ell-1)}} \gamma_\ell$.
\end{lemma}
\begin{proof}
    Define $\tau$ as the first step $j \in [k]$ at which one of the following
    three events occurs (checked in order):
    \begin{enumerate}
        \item \emph{Type-1 (PG failure):} $(X, Z^{j-1}) \notin \cC_j$.
        \item \emph{Type-2 (log-ratio failure):} $(X, Z^j) \notin \cE_j$.
        \item \emph{Type-3 (Azuma failure):} $(X, Z^j) \notin \cG_j$ but $(X, Z^j) \in \cH_j$.
    \end{enumerate}
    If $\tau > k$, then $(X, Z^k) \in \cG_k$.
    Therefore $\Pr{(X, Z^k) \notin \cG_k} \le \sum_{j=1}^k \Pr{\tau = j}$.

    When $\tau = j$, steps $1, \ldots, j-1$ all passed, so
    $(X, Z^{j-1}) \in \cG_{j-1} \cap \cH_{j-1}$.
    We bound each failure type:
    \begin{itemize}
        \item Type-1 at $j$:
            $\Pr{(X, Z^{j-1}) \notin \cC_j,\; (X, Z^{j-1}) \in \cG_{j-1}}
            \le e^{\eps^{(j-1)}} \gamma_j$
            by \Cref{lm:z_C_i}.
        \item Type-2 at $j$:
            Since no type-1 failure occurred, $(X, Z^{j-1}) \in \cHhat_j$, so
            $\Pr{(X, Z^j) \notin \cE_j,\; (X, Z^{j-1}) \in \cHhat_j}
            \le \hat{\delta}_j$
            by \Cref{lm:h_hat_to_h}.
        \item Type-3 at $j$:
            Since neither type-1 nor type-2 occurred at step $j$,
            we have $(X, Z^{j-1}) \in \cC_j$ and $(X, Z^j) \in \cE_j$;
            combined with $(X, Z^{j-1}) \in \cH_{j-1}$, this gives $(X, Z^j) \in \cH_j$.
            Therefore
            $\Pr{(X, Z^j) \notin \cG_j,\; (X, Z^j) \in \cH_j}
            \le \delta'$
            by \Cref{lm:h_to_g}.
    \end{itemize}
    Summing over $j$:
    $\Pr{(X, Z^k) \notin \cG_k}
    \le \sum_{j=1}^k \rbr{e^{\eps^{(j-1)}} \gamma_j + \hat{\delta}_j + \delta'}
    = \delta^{(k)}$.
\end{proof}

Similarly, the following holds for the oracle transcript.
\begin{lemma}\label{lm:tilde_z_in_G_i}
    $\Pr{(X, \wtilde{Z}^k) \notin \cG'_k}
    \le k\delta' + \sum_{\ell \le k} \hat{\delta}_\ell
    + \sum_{\ell \le k} \gamma_\ell
    \le \delta^{(k)}$.
\end{lemma}
\begin{proof}
    The same stopping-time argument applies, with $Z$ replaced by $\wtilde{Z}$,
    $\cG_j$ replaced by $\cG'_j$, and \Cref{lm:h_to_g} replaced by \Cref{lm:h_to_g_tilde}.
    The only difference is in the type-1 bound:
    since $X \perp \wtilde{Z}^{j-1}$,
    $\Pr{(X, \wtilde{Z}^{j-1}) \notin \cC_j} \le \gamma_j$ directly by \Cref{lm:tilde_z_C_i}
    (no need to go through $\cG_{j-1}$).
    Summing gives
    $\sum_{j=1}^k (\gamma_j + \hat{\delta}_j + \delta')
    \le \delta^{(k)}$,
    where the inequality uses $\gamma_j \le e^{\eps^{(j-1)}} \gamma_j$.
\end{proof}

We are now able to prove \Cref{lm:pg_compose_het}.
\begin{proof}[Proof of \Cref{lm:pg_compose_het}]
    By \Cref{lm:z_in_G_i} and \Cref{lm:tilde_z_in_G_i},
    \begin{align*}
        \Pr{(X, Z^k) \notin \cG_k},\;
        \Pr{(X, \wtilde{Z}^k) \notin \cG'_k}
        \le \delta^{(k)}.
    \end{align*}
    This is exactly \Cref{eq:to_show_lem_gen_pg}.
    By \Cref{lm:G_G_prime_implies_joint},
    $(X, Z^k) \approx_{\eps^{(k)}, \delta^{(k)}} (X, \wtilde{Z}^k)$.
    By \Cref{lm:joint_implies_cond},
    the adaptive composition $\cAcomp$ is $(\gamma^*, \eps^*, \delta^*)$-PG with
    $\eps^* = 3\eps^{(k)}$ and $\delta^* = \gamma^* = 5\sqrt{\delta^{(k)}/\min\cbr{\eps^{(k)}, 1}}$.
\end{proof}

\subsubsection{Proof of \Cref{lm:pg_compose}}
\label{sec:proof_compose_easier}

\begin{proof}
    We invoke \Cref{lm:pg_compose_het} and bound the values there.
    We first note that the assumption $\frac{\delta_i}{\eps_i^2} \le c_1$ also implies the following bounds:
    \begin{align}
        \label{fi}
        \frac{\delta_i}{\eps_i} &\le O(\eps_i), \\
        \label{seceq}
        \frac{\delta_i^2}{\eps_i^2} &\le O(\delta_i), \\
        \label{th}
        \delta_i &\le O(\eps_i^2).
    \end{align}
    Observe that, since $\eps_i \le c_1$ for small enough $c_1$,
    we have
    $e^{-\eps_i} \le 1 -\frac{\eps_i}{2}$ which implies
    $1-e^{-\eps_i} \ge \frac{\eps_i}{2}$ and therefore
    \begin{align*}
        \hat{\delta}_i = 
        \frac{2\delta_i}{1-e^{-\eps_i}}
        \le O(\frac{\delta_i}{\eps_i}).
    \end{align*}
    Plugging into 
    \eqref{fi}, this further implies
    \begin{align}
        \label{hatdelta}
        \hat{\delta_i} \le O(\eps_i).
    \end{align}
    Next, we observe that
    since $\eps_i \le c_1$, we have
    $\delta_i(2e^{\eps_i} + 1) \le O(\delta_i)$.
    Additionally, since $e^x - 1 \ge x$ for all $x$,
    we have
    $\frac{2e^{2\eps_i}}{e^{\eps_i} - 1} + 1
    \le
    O(\frac{1}{\eps_i})
    .
    $
    It follows that
    \begin{align*}
        \psi_i(\delta_i) \le O(\delta_i + \frac{\delta_i^2}{\eps_i^2})
        .
    \end{align*}
    Plugging in \eqref{seceq} and \eqref{th}, we obtain:
    \begin{align} \label{psi}
    \psi_i(\delta_i) \le O(\delta_i) \le O(\eps_i^2).
    \end{align}
    
    We proceed to bound $\eps^*$. 
    We first note that for small values of $x$, we can approximate $1 - x \approx e^{-x}$.
    More specifically, for $x > 0$ smaller than some constant we have
    $1-x \ge e^{-2x}$ and $e^{x} -1 \le 2x$.
    Since $\hat{\delta}_i = O(\epsilon_i)$ and $\epsilon_i \le c_1$ where $c_1$ is an arbitrary small constant, this implies
    \begin{align*}
        \frac{e^{2 \eps_i}}{1 - \hat{\delta}_i} - 1 \le e^{2 \eps_i + 2\hat{\delta}_i} - 1 \le 4\rbr{\eps_i + \hat{\delta_i}} = O(\eps_i).
    \end{align*}

    Combined with \Cref{psi}, this yields:
    \begin{align*}
        \Big(2 \eps_i \big( e^{2 \eps_i} / (1-\hat{\delta}_i) - 1 \big) + \psi_i(\delta_i) \Big) = O(\eps_i^2).
    \end{align*}
    Plugging in the definition of $\eps^{*}$ from \Cref{lm:pg_compose_het} we obtain
    \begin{align*}
        \eps^{*} 
        = 3\eps^{(k)}
        &= 
        6\sqrt{2\ln(1/\delta') (\sum_i \eps_i^2)}
        + 
        3
        \sum_{i}
        \rbr{
        \psi_i(\delta_i)
        +
        2\eps_i\rbr{\frac{e^{2\eps_i}}{1-\hat{\delta}_i} - 1}
        }
        \\&\le 
        O\rbr{\sqrt{\ln(1/\delta') \sum_{i}\eps_i^2} + \sum_{i} \eps_i^2}
    \end{align*}
    as claimed.
    
    We proceed to bound $\delta^*$ and $\gamma^*$.
    Since $\gamma_i = \delta_i$ in \Cref{lm:pg_compose}, the formula for $\delta^{(k)}$
    from \Cref{lm:pg_compose_het} gives
    \begin{align*}
        \delta^{(k)}
        = k\delta' + \sum_{i} \hat{\delta}_i + \sum_{i} e^{\eps^{(i-1)}} \delta_i.
    \end{align*}
    Since $\eps^* \le c_3$ and $\eps^{(i-1)} \le \eps^{(k)} = \eps^*/3$,
    we have $e^{\eps^{(i-1)}} = O(1)$.
    Together with $\hat{\delta}_i \le O(\delta_i/\eps_i)$, this gives
    \begin{align*}
        \delta^{(k)}
        \le O\rbr{
            k\delta' + \sum_{i} \delta_i/\eps_i
        }.
    \end{align*}
    Since $\eps^{(k)} \le \eps^{*} \le c_3$,
    we have
    $\min\cbr{\eps^{(k)}, 1} \in \Theta(\eps^*)$.
    Plugging the definition of $\gamma^*$ we obtain
    \begin{align*}
        \gamma^* = O\rbr{ 
            \sqrt{\frac{\delta^{(k)}}{\eps^{(k)}}}
        }
        \le
        O\rbr{
            \sqrt{\frac{k\delta' + \sum_{i} \delta_i/\eps_i}{\eps^*}}
        }
        .
    \end{align*}
    Similarly,
    \begin{align*}
        \delta^* = \gamma^*
        \le
        O\rbr{
            \sqrt{\frac{k\delta' + \sum_{i} \delta_i/\eps_i}{\eps^*}}
        }
        .
    \end{align*}
    Therefore, the lemma follows.
    
\end{proof}

\section{Lower Bound for Adaptive Composition}
\label{sec:adaptive_lb}

We demonstrate that while nonadaptive composition of replicable algorithms can be achieved with nearly linear sample complexity, extending this guarantee to an adaptive adversary necessarily incurs a quadratic cost in the number of rounds.

\begin{definition}[Adaptive Composition against a Random Adversary]\label{def:adaptive_game}
Let $\cT$ be a domain of statistical problems.
An adaptive game of $k$ rounds between a random adversary and an algorithm $\cA$ proceeds as follows.
The adversary possesses internal random bits $\wtilde{r}$.
For each round $i \in [k]$:
\begin{enumerate}
    \item The adversary announces a statistical problem $\cT_i = \operatorname{Adv}(y_1, \dots, y_{i-1}; \wtilde{r})$,
        potentially depending on previous outputs.
        The adversary also selects a distribution $\cD_i$ with $(\cD_i, G_{\cD_i}) \in \cT_i$;
        this choice is not revealed to $\cA$.
    \item The algorithm $\cA$, using fixed internal coins $r$, observes $\cT_i$,
        draws $m$ independent samples $S_i \sim \cD_i^m$,
        and produces an output $y_i = \cA(\cT_i, S_1, \dots, S_i, y_1, \dots, y_{i-1}; r)$.
        A valid output satisfies $y_i \in G_{\cD_i}$.
\end{enumerate}
The algorithm $\cA$ is $\rho$-replicable in the adaptive setting if, upon resampling all datasets $S = (S_1, \dots, S_k)$ and $S' = (S'_1, \dots, S'_k)$ while keeping the coins $r$ and $\wtilde{r}$ fixed, the sequence of outputs matches with high probability:
\begin{align*}
    \Pru{r, \wtilde{r}, S, S'}{(y_1, \dots, y_k) = (y'_1, \dots, y'_k)} \ge 1 - \rho.
\end{align*}
\end{definition}

To establish the lower bound, we rely on the following lemma concerning threshold estimation, demonstrating that forcing high accuracy over a sharp decision boundary induces an unavoidable replicability error dictated by the sample noise floor.

The lemma is analogous to Lemma~7.2 in \cite{impagliazzo2022reproducibility}, but differs in two important respects tailored to our adaptive setting.
First, the lemma is stated for a \emph{fixed} choice of internal coins $r$: for any fixed $r$, either the algorithm's correctness or its replicability (with that same $r$) must take a hit.
Second, the hard instance $\theta$ is drawn from a \emph{fixed known distribution} $\cP$ rather than being an arbitrary worst-case parameter.

These two features work together in the proof of \Cref{thm:adaptive_lb}.
Because $r$ is fixed, the algorithm's behavior on each round is a deterministic function of the samples alone, which ensures that the errors across rounds are not correlated with each other---without fixing $r$, the algorithm could in principle correlate its randomness across rounds to avoid simultaneous failures.
Because the hard instance is drawn fresh from the known distribution $\cP$ at each round (independently of history), the algorithm cannot ``learn'' anything useful from earlier rounds: the second problem instance is statistically no easier than the first, even after observing the samples and outcome of round one.
Together, these properties allow us to re-use the single-round hardness argument across all $k$ rounds of the adaptive game.

\begin{lemma}[Replicability Lower Bound for Threshold Estimation] \label{lem:threshold_estimation}
    Let $m \in \mathbb{N}$ and $\tau \in (0, 1/4)$.
    Let $\Theta = [1/2 - \tau, 1/2 + \tau]$ and define an adversary distribution $\cP$ over $\Theta$ such that $\theta = 1/2 - \tau$ with probability $1/3$, $\theta = 1/2 + \tau$ with probability $1/3$, and $\theta$ is drawn uniformly from $(1/2 - \tau, 1/2 + \tau)$ with probability $1/3$.

    Define the set of valid outputs for each $\theta \in \Theta$ as
    \begin{align*}
        G_\theta = \begin{cases} \{-\} & \theta = 1/2 - \tau, \\ \{+\} & \theta = 1/2 + \tau, \\ \{+, -\} & \theta \in (1/2-\tau, 1/2+\tau). \end{cases}
    \end{align*}
    Let $\cA$ be any randomized algorithm mapping $m$ samples $S \sim \operatorname{Bernoulli}(\theta)^m$ and internal coins $r$ to a sign in $\{+, -\}$.
    For fixed internal coins $r$, if $\cA$ satisfies
    \begin{align*}
        \Exu{\theta \sim \cP}{\Pru{S \sim \theta^m}{\cA(S; r) \in G_\theta}} \ge 1 - \frac{1}{16},
    \end{align*}
    then, with $r$ still fixed,
    \begin{align}
        \Exu{\theta \sim \cP}{\Pru{S, S' \sim \theta^m}{\cA(S; r) \neq \cA(S'; r)}} \ge \Omega\rbr{\frac{1}{\tau\sqrt{m}}}.
        \label{eq:lower_bound_lme_threshold}
    \end{align}
\end{lemma}

\begin{proof}
    Fix the internal coins $r$.
    Let $p_+(\theta) = \Pru{S \sim \theta^m}{\cA(S; r) = +}$ denote the probability that $\cA$ outputs $+$, with $r$ held fixed.

    Since $G_\theta = \{+,-\}$ for interior $\theta$, the hypothesis gives expected failure probability at most $1/16$:
    \begin{align*}
        \frac{1}{3}p_+(1/2-\tau) + \frac{1}{3}\rbr{1 - p_+(1/2+\tau)} \le \frac{1}{16}.
    \end{align*}
    Since both terms are non-negative, each is individually at most $1/16$, yielding
    $p_+(1/2 - \tau) \le 3/16$ and $p_+(1/2 + \tau) \ge 13/16$.

    For a fixed target $\theta$, given that exactly $j$ out of $m$ samples are ones, the sequence is uniformly distributed among all Boolean vectors of Hamming weight $j$.
    Let $a_j \in [0, 1]$ denote the fraction of sample sequences of weight $j$ for which $\cA(\cdot; r)$ outputs $+$.
    We express the probability of outputting $+$ as a polynomial:
    \begin{align*}
        p_+(\theta) = \sum_{j=0}^m a_j \binom{m}{j} \theta^j (1 - \theta)^{m - j}.
    \end{align*}
    By the Intermediate Value Theorem, since $p_+(1/2 - \tau) < 1/2$ and $p_+(1/2 + \tau) > 1/2$, there is a point $q \in (1/2 - \tau, 1/2 + \tau)$ where $p_+(q) = 1/2$.
    Taking the derivative yields:
    \begin{align*}
        p_+'(\theta) = \sum_{j=0}^m a_j \binom{m}{j} \theta^j (1 - \theta)^{m-j} \frac{j - m\theta}{\theta(1 - \theta)}.
    \end{align*}
    For $\tau < 1/4$, we have $\theta \in (1/4, 3/4)$, meaning the denominator $\theta(1 - \theta) > 3/16$.
    Since $0 \le a_j \le 1$, the absolute derivative is bounded by $6\, \Ex{\abs{X - \Ex{X}}}$ where $X \sim \operatorname{Bin}(m, \theta)$.
    As the expected absolute deviation is bounded by the standard deviation $\sqrt{m\theta(1 - \theta)}$, we have $\abs{p_+'(\theta)} = O(\sqrt{m})$.

    Because $\abs{p_+'(\theta)} = O(\sqrt{m})$ and $p_+(q) = 1/2$, there exists an interval $I$ around $q$ of length $\Omega(1/\sqrt{m})$ wherein $1/3 < p_+(\theta) < 2/3$.
    The adversary distribution $\cP$ places $1/3$ uniform probability mass on $(1/2 - \tau, 1/2 + \tau)$.
    The probability that a randomly drawn $\theta$ lands in $I$ is at least $\frac{1}{3} \cdot \frac{\abs{I}}{2\tau} = \Omega\rbr{\frac{1}{\tau\sqrt{m}}}$.
    For any $\theta \in I$, the disagreement probability $\Pru{S, S' \sim \theta^m}{\cA(S; r) \neq \cA(S'; r)} = 2p_+(\theta)(1-p_+(\theta)) > 4/9$ (with $r$ fixed), so
    \begin{align*}
        \Exu{\theta \sim \cP}{\Pru{S, S' \sim \theta^m}{\cA(S; r) \neq \cA(S'; r)}}
        \ge \frac{4}{9} \cdot \Pru{\theta \sim \cP}{\theta \in I}
        = \Omega\rbr{\frac{1}{\tau\sqrt{m}}},
    \end{align*}
    as required.
\end{proof}

Assuming \Cref{lem:threshold_estimation}, we establish the main lower bound for adaptive composition.

\begin{theorem}[Adaptive Composition Lower Bound] \label{thm:adaptive_lb}
    Consider the adaptive game where at each round $i \in [k]$, the adversary
    presents the threshold estimation problem $\cT$ (with parameter $\tau$ as in \Cref{lem:threshold_estimation})
    and independently draws $\theta_i \sim \cP$ without revealing it.
    Any algorithm $\cA$ that is $\rho$-replicable in this game and answers all $k$ rounds
    correctly with probability $\ge 1 - \beta$, for a sufficiently small constant $\beta$,
    requires sample complexity $m = \Omega(nk^2)$, where $n = \Theta(1/\tau^2)$ is the
    single-round sample complexity of threshold estimation.
\end{theorem}

\begin{proof}
    The proof proceeds in two steps: we first establish a lower bound for any fixed $r$,
    then derive the theorem by averaging over $r$.

    \begin{claim}[Fixed-$r$ lower bound] \label{clm:fixed_r_lb}
        The adversary uses $\wtilde{r}$ to draw $\theta_1, \dots, \theta_k \sim \cP$ independently,
        presenting $\cT$ each round but withholding $\theta_i$.
        Let $H_{i-1} = (S_1, y_1, \dots, S_{i-1}, y_{i-1})$ denote the full history before round $i$
        (past samples and outputs).
        Note that on the event $\{M_{i-1} = 1\}$ (runs 1 and 2 agreed on all previous outputs),
        the two histories $H_{i-1}$ and $H'_{i-1}$ share the same output sequence
        but differ in the samples, so in general $H_{i-1} \ne H'_{i-1}$.
        For any fixed $r$, define round $i$ to be \emph{accurate} if
        \begin{align*}
            X_i = \ind{
                \Exu{\theta \sim \cP}{\Pru{S \sim \theta^m}{\cA(\cT, S, H_{i-1}; r) \in G_\theta}}
                \ge 1 - \frac{1}{16}
            } = 1,
        \end{align*}
        and let $Y = \sum_{i=1}^k X_i$, $Z = k - Y$.
        Define the mismatch and failure probabilities for fixed $r$ as
        \begin{align*}
            \rho(r) &= \Pru{\wtilde{r}, S, S'}{\exists\, i \in [k] \colon y_i \ne y'_i \;\Big|\; r}, \\
            \beta(r) &= \Pru{\wtilde{r}, S}{\exists\, i \in [k] \colon y_i \notin G_{\theta_i} \;\Big|\; r}.
        \end{align*}
        Let $p = \Omega\rbr{1/(\tau\sqrt{m})}$ be the value from \Cref{eq:lower_bound_lme_threshold}.
        Then for any fixed $r$, letting $\Ex{\cdot}$ denote expectation over $\wtilde{r}$ and samples with $r$ fixed:
        \begin{align*}
            \Ex{Y} \ge k/2 &\implies \rho(r) \ge \frac{pk/4}{1+pk/2}, \\
            \Ex{Y} < k/2  &\implies \beta(r) \ge \frac{k/32}{1+k/16}.
        \end{align*}
    \end{claim}

    \begin{proof}[Proof of \Cref{clm:fixed_r_lb}]
        The key observation is that $\theta_i$ is drawn independently of $H_{i-1}$
        (the adversary ignores history when sampling $\theta_i$).
        Hence, for fixed $r$ and fixed $H_{i-1}$, the function $S \mapsto \cA(\cT, S, H_{i-1}; r)$
        is a fixed-$r$ single-round algorithm to which \Cref{lem:threshold_estimation} applies,
        with a fresh $\theta_i \sim \cP$ independent of $H_{i-1}$.
        This independence is what decouples the disagreement events across rounds, allowing us to show
        that the errors (either correctness or replicability) compound.

        We consider two cases.

        \medskip\noindent\textbf{Case 1: Many accurate rounds} ($\Ex{Y} \ge k/2$).
        Let $E$ be the event that both runs agree on all outputs:
        $(y_1, \dots, y_k) = (y'_1, \dots, y'_k)$.
        We bound $\Pru{\wtilde{r}, S, S'}{E \mid r}$ from above.

        A first attempt is to apply \Cref{lem:threshold_estimation} directly at each round:
        since $\theta_i$ is independent of $H_{i-1}$, the lemma applied to $\cA(\cT, \cdot, H_{i-1}; r)$
        gives expected disagreement $\ge p$ whenever round $i$ is accurate,
        which would yield agreement probability at most $1 - pX_i$.
        The difficulty is that \Cref{lem:threshold_estimation} bounds the disagreement of two draws
        using the \emph{same} history, whereas on $\cbr{M_{i-1} = 1}$ the two histories
        $H_{i-1}$ and $H'_{i-1}$ share the same outputs but differ in samples,
        so $H_{i-1} \ne H'_{i-1}$ in general.
        It is therefore unclear how to bound $\Pr{y_i \ne y'_i}$ directly from the lemma.

        To recover a bound, we introduce a third draw.
        For each round $i$ (on $\cbr{M_{i-1}=1}$), let $y''_i = \cA(\cT, S''_i, H'_{i-1}; r)$
        where $S''_i \sim \theta_i^m$ is an independent fresh sample.
        Since $y'_i$ and $y''_i$ both use \emph{run 2's history} $H'_{i-1}$ with fresh independent samples,
        \Cref{lem:threshold_estimation} applied to $\cA(\cT, \cdot, H'_{i-1}; r)$ gives
        $\Exu{\theta_i \sim \cP}{\Pru{S'_i, S''_i}{y'_i \ne y''_i}} \ge p X'_i$
        where $X'_i = X_i(r, H'_{i-1})$.
        Since $y'_i$ and $y''_i$ are identically distributed (same history, fresh samples),
        $\Pr{y_i \ne y'_i} = \Pr{y_i \ne y''_i}$, and the triangle inequality gives
        $\Pr{y_i \ne y'_i} \ge \frac{1}{2}\Pr{y'_i \ne y''_i} \ge \frac{p}{2} X'_i$.
        These per-round bounds, formalized below via conditional expectations,
        yield $\Ex{M_i \mid \cF_{i-1}, r} \le M_{i-1}(1 - \tfrac{p}{2}X'_i)$ for each $i$,
        from which additive telescoping gives $\rho(r) \ge \Omega(pk)$ when $\Ex{Y'} \ge k/2$,
        where $Y' = \sum_i X'_i$.
        Since runs 1 and 2 are symmetric, $Y'$ is identically distributed to $Y$, so $\Ex{Y'} = \Ex{Y} \ge k/2$.

        More formally, let $\cF_{i-1}$ be the $\sigma$-algebra generated by
        $\theta_1, \dots, \theta_{i-1}$ and the samples $S_1, S'_1, \dots, S_{i-1}, S'_{i-1}$.
        Then $H_{i-1}$, $H'_{i-1}$, $M_{i-1}$, $X_i = X_i(r, H_{i-1})$, and $X'_i = X_i(r, H'_{i-1})$
        are all $\cF_{i-1}$-measurable.
        Introduce an auxiliary sample $S''_i \sim \theta_i^m$ independent of $\cF_{i-1}$ given $\theta_i$,
        and let $y''_i = \cA(\cT, S''_i, H'_{i-1}; r)$.
        On $\cbr{M_{i-1} = 1}$, conditioning on $\cF_{i-1}$ and $\theta_i$:
        \begin{itemize}
            \item $y'_i$ and $y''_i$ have the same conditional distribution (both use $H'_{i-1}$ with fresh samples), so $\Pr{y_i = y'_i \mid \cF_{i-1}, r, \theta_i} = \Pr{y_i = y''_i \mid \cF_{i-1}, r, \theta_i}$.
            \item By the triangle inequality: $\Pr{y'_i \ne y''_i \mid \cF_{i-1}, r, \theta_i} \le 2\Pr{y_i \ne y'_i \mid \cF_{i-1}, r, \theta_i}$.
            \item \Cref{lem:threshold_estimation} applied to $\cA(\cT, \cdot, H'_{i-1}; r)$ gives $\Exu{\theta_i \sim \cP}{\Pru{S'_i, S''_i \sim \theta_i^m}{y'_i \ne y''_i \mid \cF_{i-1}, r, \theta_i}} \ge p X'_i$.
        \end{itemize}
        Hence $\Exu{\theta_i \sim \cP}{\Pr{y_i \ne y'_i \mid \cF_{i-1}, r, \theta_i}} \ge \frac{p}{2} X'_i$,
        so $\Ex{\ind{y_i = y'_i} \mid \cF_{i-1}, r} \le 1 - \frac{p}{2} X'_i$.
        Therefore $\Ex{M_i \mid \cF_{i-1}, r} \le M_{i-1}(1 - \frac{p}{2}X'_i)$ always
        (trivially on $\cbr{M_{i-1}=0}$).
        Since $M$ is non-increasing ($M_k \le M_{i-1}$ pointwise for all $i \le k$), we proceed via additive telescoping.
        Taking expectations of $\Ex{M_i \mid \cF_{i-1}, r} \le M_{i-1}(1-\frac{p}{2}X'_i)$ and using $M_{i-1} \ge M_k$:
        \begin{align*}
            \Ex{M_{i-1} - M_i \mid r} \ge \tfrac{p}{2}\Ex{M_{i-1}X'_i \mid r} \ge \tfrac{p}{2}\Ex{M_k X'_i \mid r}.
        \end{align*}
        Summing from $i=1$ to $k$ and using $M_0 = 1$:
        \begin{align*}
            1 - \Exu{\wtilde{r},S,S'}{M_k \mid r} \ge \tfrac{p}{2}\,\Exu{\wtilde{r},S,S'}{M_k Y' \mid r}.
        \end{align*}
        Since $Y'$ is identically distributed to $Y$, we have $\Ex{Y' \mid r} \ge k/2$.
        Using $Y' \le k$ and $1-M_k \ge 0$:
        \begin{align*}
            \Ex{M_k Y' \mid r} = \Ex{Y' \mid r} - \Ex{(1-M_k)Y' \mid r} \ge \tfrac{k}{2} - k\rbr{1 - \Ex{M_k \mid r}}.
        \end{align*}
        Substituting the bound on $\Ex{M_k Y' \mid r}$ into $1 - \Ex{M_k \mid r} \ge \frac{p}{2}\Ex{M_k Y' \mid r}$:
        \begin{align*}
            1 - \Ex{M_k \mid r}
            \ge \tfrac{p}{2}\Bigl(\tfrac{k}{2} - k\rbr{1 - \Ex{M_k \mid r}}\Bigr)
            = \tfrac{pk}{4} - \tfrac{pk}{2}\rbr{1 - \Ex{M_k \mid r}}.
        \end{align*}
        Collecting the $\rbr{1 - \Ex{M_k \mid r}}$ terms:
        \begin{align*}
            \rbr{1 - \Ex{M_k \mid r}}\rbr{1 + \tfrac{pk}{2}} \ge \tfrac{pk}{4},
            \quad\text{so}\quad
            \rho(r) = 1 - \Exu{\wtilde{r},S,S'}{M_k \mid r} \ge \tfrac{pk/4}{1+pk/2}.
        \end{align*}

        \medskip\noindent\textbf{Case 2: Many inaccurate rounds} ($\Ex{Y} < k/2$).
        Here $\Ex{Z} \ge k/2$, where $Z = \sum_{i=1}^k (1-X_i)$ counts the inaccurate rounds.
        Let $N_i = \prod_{j=1}^i \ind{y_j \in G_{\theta_j}}$ be the indicator that all rounds $1, \dots, i$ are answered correctly, with $N_0 = 1$.
        Since $X_i$ and $N_{i-1}$ are both $\cF_{i-1}$-measurable, and an inaccurate round satisfies
        $\Exu{\theta_i \sim \cP}{\Pru{S_i \sim \theta_i^m}{\cA(\cT, S_i, H_{i-1}; r) \in G_{\theta_i}}} < 15/16$,
        we have $\Ex{N_i \mid \cF_{i-1}, r} \le N_{i-1}\rbr{1 - \tfrac{1}{16}(1-X_i)}$, and hence:
        \begin{align*}
            \Ex{N_{i-1} - N_i \mid \cF_{i-1}, r} \ge \tfrac{1}{16} N_{i-1}(1-X_i) \ge \tfrac{1}{16} N_k(1-X_i).
        \end{align*}
        Taking expectations and summing from $i = 1$ to $k$:
        \begin{align*}
            1 - \Ex{N_k \mid r} \ge \tfrac{1}{16}\,\Ex{N_k Z \mid r}.
        \end{align*}
        Using $Z \le k$ and $1 - N_k \ge 0$:
        \begin{align*}
            \Ex{N_k Z \mid r} = \Ex{Z \mid r} - \Ex{(1-N_k)Z \mid r} \ge \tfrac{k}{2} - k\rbr{1 - \Ex{N_k \mid r}}.
        \end{align*}
        Substituting and rearranging:
        \begin{align*}
            \rbr{1 - \Ex{N_k \mid r}}\rbr{1 + \tfrac{k}{16}} \ge \tfrac{k}{32},
            \quad\text{so}\quad
            \beta(r) = 1 - \Pru{\wtilde{r},S}{\text{all correct} \mid r} \ge \tfrac{k/32}{1+k/16}.
        \end{align*}
    \end{proof}

    \noindent We now derive the theorem from \Cref{clm:fixed_r_lb}.
    Let $R_1 = \cbr{r : \Ex{Y} \ge k/2}$ and $R_2 = \cbr{r : \Ex{Y} < k/2}$ be a partition of the support of $r$.
    Since $\Pru{r}{R_1} + \Pru{r}{R_2} = 1$ (where $\Pr{\cdot}$ is over the randomness of $r$),
    at least one of $\Pru{r}{R_1} \ge 1/2$ or $\Pru{r}{R_2} \ge 1/2$ must hold.

    If $\Pr{R_2} \ge 1/2$: averaging $\beta(r)$ over $r \in R_2$ gives
    $\beta \ge \frac{1}{2} \cdot \frac{k/32}{1+k/16}$,
    which exceeds $\beta$ for large enough $k$---a contradiction.
    Hence $\Pr{R_1} \ge 1/2$, and averaging $\rho(r)$ over $r \in R_1$:
    \begin{align*}
        \rho \ge \frac{1}{2} \cdot \frac{pk/4}{1+pk/2} = \frac{pk/8}{1+pk/2}.
    \end{align*}
    Since $x \mapsto x/(1+x/2)$ is increasing and positive, keeping $\rho$ small (a sufficiently small constant) requires $pk = O(1)$, i.e., $p = O(1/k)$.
    Combined with $p = \Omega\rbr{1/(\tau\sqrt{m})}$, this gives $m = \Omega(k^2/\tau^2) = \Omega(nk^2)$.
\end{proof}

\subsection{Replicable Algorithms and Sufficient Statistics}
\label{sec:suff_stat}

A classical principle in statistics is that a sufficient statistic captures all the information a sample carries about the underlying distribution.
In this section, we show that this principle extends to replicability:
any replicable algorithm for a problem over $\cF$ can be converted into an equally accurate algorithm that operates only on a sufficient statistic of its input, with replicability preserved for distributions in $\cF$.

We first recall the definition of a sufficient statistic.

\begin{definition}[Sufficient Statistic]
    Let $\cF = \{D_\theta : \theta \in \Theta\}$ be a parametric family of distributions over $\cX$.
    A function $f : \cX^n \to \cZ$ is a \emph{sufficient statistic} for $\cF$ if the conditional distribution of $S \sim D_\theta^n$ given $f(S) = t$ does not depend on $\theta$, for every $t$ in the range of $f$.
\end{definition}

The key algorithmic consequence of sufficiency is that one can resample a fresh dataset $S'$ with the same marginal distribution as $S$, using only $f(S)$ and fresh randomness, without knowing the underlying parameter $\theta$.
Concretely, for any $t$ in the range of $f$, let $\mu_t$ denote the conditional distribution of $S$ given $f(S) = t$;
this is well-defined and $\theta$-free by sufficiency.
Given $f(S) = t$ and fresh randomness $r'$, one can draw $S' \sim \mu_t$ using $r'$; we write $S' = \mathrm{Resample}(f(S), r')$.

\begin{lemma}\label{lm:suff_marginal}
    Let $f$ be a sufficient statistic for $\cF$, and let $S \sim D_\theta^n$ for some $\theta \in \Theta$.
    Let $r'$ be independent of $S$ and let $S' = \mathrm{Resample}(f(S), r')$.
    Then $(f(S), S') \overset{d}{=} (f(S), S)$,
    and in particular $S' \sim D_\theta^n$.
\end{lemma}

\begin{proof}
    We give the argument for the case when $\cX$ is discrete; the continuous case is analogous
    with probability mass functions replaced by densities.
    For any $t$ and $s$,
    \begin{align*}
        \Pr{f(S) = t,\; S' = s}
        &= \Pr{f(S) = t} \cdot \Pr{S' = s \mid f(S) = t} \\
        &= \Pr{f(S) = t} \cdot \mu_t(s) \\
        &= \Pr{f(S) = t} \cdot \Pr{S = s \mid f(S) = t} \\
        &= \Pr{f(S) = t,\; S = s},
    \end{align*}
    where the second equality uses that $S' \sim \mu_t$ by definition of $\mathrm{Resample}$,
    and the third uses the definition of $\mu_t$.
    Hence $(f(S), S') \overset{d}{=} (f(S), S)$.
    Marginalizing over $t$ gives $S' \overset{d}{=} S \sim D_\theta^n$.
\end{proof}

We can now state and prove the main result of this section.

\begin{theorem}\label{thm:suff_stat}
    Let $\cA : \cX^n \to \cY$ be a $\rho$-replicable algorithm for a statistical problem $\cT$ over the family $\cF = \{D_\theta : \theta \in \Theta\}$, with failure probability $\beta$.
    Let $f : \cX^n \to \cZ$ be a sufficient statistic for $\cF$.
    Then there exists an algorithm $\cB : \cX^n \to \cY$ for $\cT$ with the same sample complexity and failure probability $\beta$, such that $\cB(S; r, r') = \cA(\mathrm{Resample}(f(S), r'); r)$ for all $S$, and for every $\theta \in \Theta$ and independent $S_1, S_2 \sim D_\theta^n$:
    \begin{align*}
        \Pru{r, r'}{\cB(S_1; r, r') = \cB(S_2; r, r')} \ge 1 - 2\rho.
    \end{align*}
    In particular, the output of $\cB$ depends on $S$ only through $f(S)$.
\end{theorem}

\begin{proof}
    Define $\cB(S; r, r') = \cA(S'; r)$, where $S' = \mathrm{Resample}(f(S), r')$ and $(r, r')$ are the shared random bits.

    \paragraph{Correctness.}
    By \Cref{lm:suff_marginal}, $S' \overset{d}{=} S \sim D_\theta^n$.
    Therefore the output distribution of $\cB$ is identical to that of $\cA$ run on a fresh sample from $D_\theta^n$, and $\cB$ has failure probability $\beta$.

    \paragraph{Replicability.}
    Let $S_1, S_2 \overset{\text{iid}}{\sim} D_\theta^n$ be two independent samples and let $(r, r')$ be shared random bits drawn independently of $S_1, S_2$, with $r$ and $r'$ independent of each other.
    Set $S_i' = \mathrm{Resample}(f(S_i), r')$ for $i = 1, 2$.
    We want to show $\Pr{\cB(S_1; r, r') = \cB(S_2; r, r')} \ge 1 - 2\rho$, i.e., $\Pr{\cA(S_1'; r) = \cA(S_2'; r)} \ge 1 - 2\rho$.

    Introduce an auxiliary sample $S^* \sim D_\theta^n$ drawn independently of $(S_1, S_2, r, r')$.
    By a union bound,
    \begin{align*}
        \Pr{\cA(S_1'; r) \neq \cA(S_2'; r)}
        \le \Pr{\cA(S_1'; r) \neq \cA(S^*; r)} + \Pr{\cA(S^*; r) \neq \cA(S_2'; r)}.
    \end{align*}
    We bound each term using the $\rho$-replicability of $\cA$.
    \begin{itemize}
        \item $S_1' = \mathrm{Resample}(f(S_1), r')$ depends only on $(S_1, r')$, both of which are independent of $S^*$.
            By \Cref{lm:suff_marginal}, $S_1' \sim D_\theta^n$.
            Hence $(S_1', S^*)$ are i.i.d.\ draws from $D_\theta^n$, independent of each other and of $r$,
            so $\Pr{\cA(S_1'; r) \neq \cA(S^*; r)} \le \rho$.
        \item Symmetrically, $S_2'$ depends only on $(S_2, r')$, independent of $S^*$, and $S_2' \sim D_\theta^n$,
            so $\Pr{\cA(S^*; r) \neq \cA(S_2'; r)} \le \rho$.
    \end{itemize}
    Therefore $\Pr{\cA(S_1'; r) = \cA(S_2'; r)} \ge 1 - 2\rho$, establishing the replicability guarantee for $D_\theta \in \cF$.

    \paragraph{Dependence on $f(S)$ only.}
    By construction, $\cB(S; r, r') = \cA(\mathrm{Resample}(f(S), r'); r)$, which depends on $S$ only through $f(S)$.
\end{proof}

\begin{remark}\label{remark:rho_2}
    The factor-of-$2$ loss in the replicability parameter is unavoidable in the above argument
    because it relies on an intermediate auxiliary sample $S^*$ to decouple $S_1'$ and $S_2'$,
    which share the random bits $r'$.
    In many applications this loss is negligible, since one can simply start
    with a $(\rho/2)$-replicable algorithm and apply \Cref{thm:suff_stat} to
    obtain a $\rho$-replicable algorithm at no asymptotic cost.
    Concretely, given the replicability parameter boosting result of \cite{impagliazzo2022reproducibility}, the dependence of sample complexity on $\rho$ is always $1/\rho^2$ in the worst case and as such the drawback is fairly mild.
\end{remark}

\begin{remark}
    \Cref{thm:suff_stat} gives a useful design principle: when designing a replicable algorithm for a problem over a parametric family $\cF$, it is enough to design one that operates on $f(S)$ rather than the full sample $S$, since replicability (for distributions in $\cF$) is preserved under the transformation.
    For example, for problems over Bernoulli distributions, the sufficient statistic is $\sum_{i} S_i$, and it suffices to design replicable algorithms based on the empirical sum rather than the full sample.
\end{remark}

\paragraph{Comparison with prior work.}
A closely related argument result is proved by
\cite{aamand2025structure} in their 
lower bound proof for replicable Gaussian mean testing.
Using the fact that the empirical mean $\bar{X}$ is a sufficient statistic for $\mu$ in
$\mathcal{N}(\mu, I)$, they reduce, without loss of generality, to algorithms that depend only on
$\bar{X}$.
Their reduction guarantees stability only when both datasets are drawn from the same member of the
parametric family $\cF$; similarly, \Cref{thm:suff_stat} guarantees replicability for distributions
within $\cF$.
\Cref{thm:suff_stat} is stated as a general standalone transformation that applies to any problem and any sufficient statistic, rather than being specific to hypothesis testing.
It should be noted however that the approach of
\cite{aamand2025structure} preserves the replicability parameter exactly (no factor-of-$2$ loss).
\Cref{thm:suff_stat} incurs a factor-of-$2$ loss due to the union bound over the phantom sample
but this is a negligible loss given that one can generally replace $\rho$ with $\rho/2$ with no asymptotic cost (see \Cref{remark:rho_2})

\subsection{Special Case: Order-Invariant Algorithms}
\label{sec:suff_order_inv}

We instantiate \Cref{thm:suff_stat} explicitly for the case of order invariance,
where the sufficient statistic is the unordered multiset of samples.
This serves both as a sanity check for the general proof and as a self-contained result
that is easy to state and verify directly.

\begin{definition}[Order-Invariant Algorithm]
    An algorithm $\cA$ is \emph{(pointwise) order-invariant} if
    for every sample $S = (S_1, \ldots, S_n)$, every permutation $\sigma \in \mathfrak{S}_n$,
    and every fixed randomness $r$:
    \begin{math}
        \cA(S; r) = \cA(\sigma(S); r).
    \end{math}
\end{definition}

\begin{corollary}\label{cor:order_inv}
    Let $\cA : \cX^n \to \cY$ be a $\rho$-replicable algorithm for a statistical problem $\cT$
    over an i.i.d.\ family with failure probability $\beta$.
    Assume $\cX$ carries a total order.
    Define $\cB(S;\, r, r') = \cA(S';\, r)$, where $S' = \sigma(\mathrm{sort}(S))$,
    $\mathrm{sort}(S)$ is the sample sorted in non-decreasing order,
    and $\sigma \sim \mathrm{Uniform}(\mathfrak{S}_n)$ is drawn using $r'$.
    Then $\cB$ is $2\rho$-replicable, (pointwise) order-invariant, solves $\cT$ with failure probability $\beta$,
    and uses the same number of samples as $\cA$.
\end{corollary}

\begin{proof}
    This is \Cref{thm:suff_stat} instantiated with $f(S) = \{S_1, \ldots, S_n\}$, the unordered multiset.
    We verify the three conditions.

    \emph{Sufficient statistic.}
    For any i.i.d.\ family, the conditional distribution of $(S_1, \ldots, S_n)$ given the multiset
    $\{S_1, \ldots, S_n\} = M$ is the uniform distribution over all $n!$ orderings of $M$,
    independently of $D$.
    Hence the multiset is a sufficient statistic.

    \emph{Correctness.}
    Since $\sigma$ is a uniformly random permutation independent of $S$,
    $S' = \sigma(\mathrm{sort}(S))$ is a uniformly random ordering of the multiset $\{S_1,\ldots,S_n\}$,
    i.e., $S' = \mathrm{Resample}(f(S), r')$ with $f(S) = \{S_1,\ldots,S_n\}$.
    By \Cref{lm:suff_marginal}, $S' \overset{d}{=} S \sim D^n$,
    so $\cB$ inherits failure probability $\beta$.

    \emph{Replicability.}
    Let $S_1, S_2 \overset{\text{iid}}{\sim} D^n$ and $(r, r')$ shared, with $\sigma$ determined by $r'$.
    Set $S_i' = \sigma(\mathrm{sort}(S_i))$ for $i = 1, 2$.
    Introduce $S^* \sim D^n$ independent of $(S_1, S_2, r, r')$.
    Each $S_i'$ is a function of $(S_i, r')$ and hence independent of $S^*$;
    by \Cref{lm:suff_marginal}, $S_i' \sim D^n$.
    Hence $(S_1', S^*)$ and $(S_2', S^*)$ are each i.i.d.\ pairs from $D^n$, independent of $r$,
    giving $\Pr{\cA(S_1'; r) \neq \cA(S^*; r)} \le \rho$ and $\Pr{\cA(S^*; r) \neq \cA(S_2'; r)} \le \rho$.
    By a union bound, $\Pr{\cB(S_1; r, r') \neq \cB(S_2; r, r')} \le 2\rho$.

    \emph{Order invariance.}
    For any permutation $\tau \in \mathfrak{S}_n$ and any fixed $(r, r')$:
    $\mathrm{sort}(\tau(S)) = \mathrm{sort}(S)$, so $S' = \sigma(\mathrm{sort}(\tau(S))) = \sigma(\mathrm{sort}(S))$.
    Hence $\cB(\tau(S);\, r, r') = \cA(S';\, r) = \cB(S;\, r, r')$ as values, not merely in distribution.
\end{proof}

\paragraph{Comparison with prior work.}
Theorem~1.2 of \cite{aamand2025structure} establishes an analogous order-invariance result:
any $\rho$-replicable algorithm for testing a symmetric property of discrete distributions
can be converted to a $\rho$-replicable, order-invariant algorithm of the same sample complexity.
The two results differ in a number of ways.
\Cref{cor:order_inv} applies to any statistical problem and any (possibly non-binary) output space.
\cite{aamand2025structure} restrict to binary hypothesis testers.
In addition, we do not require the underlying problem to be symmetric
\footnote{If the problem \emph{is} symemtric, \cite{aamand2025structure} further show that their algorithm is label order invariant,
whereas \Cref{cor:order_inv} does not have any such implications.
We discuss this notion and how our techniques yield such a result in the general case in \Cref{sec:label_inv}}
It should be noted however that the approach of
\cite{aamand2025structure} preserves the replicability parameter exactly (no factor-of-$2$ loss).
\Cref{cor:order_inv} incurs a factor-of-$2$ loss due to the union bound over the phantom sample
but this loss is negligible as one can generally replace $\rho$ with $\rho/2$ with no asymptotic cost (see \Cref{remark:rho_2}).

\subsection{Label Invariance for Symmetric Problems}
\label{sec:label_inv}

We show that for symmetric problems, any replicable algorithm can be converted to a label-invariant
one with no loss in the replicability parameter. We first define the notion of a label-invariant algorithm.
\begin{definition}[Label-Invariant Algorithm]
    \label{def:label_inv}
    An algorithm $\cA : \cX^n \to \cY$ is \emph{distributionally label-invariant} if for every sample
    $S = (S_1, \ldots, S_n) \in \cX^n$ and every permutation $\pi : \cX \to \cX$,
    $\cA(S)$ and $\cA(\pi(S))$ have the same output distribution (over all internal randomness of $\cA$),
    where $\pi(S) = (\pi(S_1), \ldots, \pi(S_n))$.
    We say $\cA$ is \emph{pointwise label-invariant} if
    $\cA(S; r) = \cA(\pi(S); r)$ for shared randomness $r$.
\end{definition}

Throughout this section, $\pi$ denotes a permutation of the \emph{domain} $\cX$
(an element of $\mathfrak{S}_{|\cX|}$), in contrast to the permutations of sample indices
in $\mathfrak{S}_n$ used in \Cref{sec:suff_order_inv}.

\begin{definition}[Symmetric Problem]
    \label{def:symmetric_problem}
    A statistical problem $\cT$ over a discrete domain $\cX$ is \emph{symmetric} if for every
    distribution $D$ over $\cX$, every valid output $y \in G_D$, and every permutation $\pi : \cX \to \cX$,
    the output $y$ is also valid for the pushforward distribution $\pi_*(D)$,
    where $\pi_*(D)(x) = D(\pi^{-1}(x))$.
\end{definition}

Informally, a problem is symmetric if the correct answer depends only on the shape of the distribution,
not on which domain element has which probability.
Uniformity testing is the canonical example: whether a distribution is $\varepsilon$-close to uniform
is preserved under any relabeling of $\cX$.

\begin{corollary}\label{cor:label_inv}
    Let $\cA : \cX^n \to \cY$ be a $\rho$-replicable algorithm for a symmetric statistical problem $\cT$
    with failure probability $\beta$ where $\cX$ is finite.
    Define $\cB(S;\, r, r') = \cA(\pi(S);\, r)$, where $\pi \sim \mathrm{Uniform}(\mathfrak{S}_{|\cX|})$
    is drawn using $r'$.
    Then $\cB$ is $\rho$-replicable, distributionally label-invariant, solves $\cT$ with failure probability $\beta$,
    and uses the same number of samples as $\cA$.
\end{corollary}

\begin{proof}
    \emph{Correctness.}
    For any distribution $D$ and fixed $\pi$, the sample $\pi(S)$ is i.i.d.\ from $\pi_*(D)$.
    Since $\cT$ is symmetric, any output valid for $\pi_*(D)$ is also valid for $D$.
    Since $\cA$ has failure probability $\beta$ on every distribution including $\pi_*(D)$,
    so does $\cB$ on $D$.

    \emph{Replicability.}
    Let $S_1, S_2 \overset{\mathrm{iid}}{\sim} D^n$ and let $(r, r')$ be shared random bits,
    with $r$ and $r'$ independent of each other and of $S_1, S_2$.
    Fix $\pi$ determined by $r'$.
    Since $S_1, S_2$ are i.i.d.\ and $\pi$ is shared,
    $\pi(S_1)$ and $\pi(S_2)$ are i.i.d.\ from $\pi_*(D)$.
    By $\rho$-replicability of $\cA$ applied to $\pi_*(D)$:
    \begin{align*}
        \Pr{\cA(\pi(S_1); r) = \cA(\pi(S_2); r)} \ge 1 - \rho.
    \end{align*}
    This holds for every fixed $\pi$, hence after averaging over $r'$:
    $\Pr{\cB(S_1; r, r') = \cB(S_2; r, r')} \ge 1 - \rho$.

    \emph{Distributional label invariance.}
    Fix any sample $S$ and any domain permutation $\sigma$.
    For any fixed $r$, we have $\cB(S; r, r') = \cA(\pi(S); r)$ and
    $\cB(\sigma(S); r, r') = \cA((\pi \circ \sigma)(S); r)$, where $\pi$ is determined by $r'$.
    As $r'$ varies, $\pi$ is uniform over $\mathfrak{S}_{|\cX|}$;
    since left-composition by the fixed $\sigma$ is a bijection on $\mathfrak{S}_{|\cX|}$,
    $\pi \circ \sigma$ has the same distribution over $r'$.
    Hence for every fixed $r$, $\cB(S; r, r')$ and $\cB(\sigma(S); r, r')$ have the same distribution over $r'$.
    Since this holds for every $r$, the two have the same distribution over $(r, r')$.
\end{proof}

\begin{remark}
    Unlike \Cref{thm:suff_stat}, \Cref{cor:label_inv} incurs no loss in the replicability parameter.
    The reason is structural: the permutation $\pi$ is shared across both runs, so $\pi(S_1)$ and $\pi(S_2)$
    are i.i.d.\ from the same distribution $\pi_*(D)$, and replicability of $\cA$ applies in one step
    without a phantom sample argument.
\end{remark}

\subsection{Pointwise Label Invariance via Correlated Sampling}
\label{sec:pointwise_label_inv}

\Cref{cor:label_inv} produces an algorithm that is distributionally label-invariant
(\Cref{def:label_inv}): for any fixed sample $S$, the output distribution is unchanged under domain relabeling.
We now show this can be strengthened to \emph{pointwise} label invariance---the notion used by
\cite{liu2024replicable} (their Definition~1.5)---where for every fixed seed $r$, the output
is identical under any relabeling.

The key tool is \emph{correlated sampling}, formalized below following~\cite{bun2023stability}.

\begin{definition}[Correlated Sampling {\cite{Broder97}}]
\label{def:correlated_sampling}
    A \emph{correlated sampling strategy} for a finite set $\cY$ with error function
    $\varepsilon : [0,1] \to [0,1]$ is a deterministic function $\mathrm{CS}$
    that takes a distribution $P \in \Delta(\cY)$ and a random string $r$ and outputs an element of $\cY$,
    satisfying:
    \begin{itemize}
        \item \emph{(Marginal correctness)} For all $P \in \Delta(\cY)$ and $y \in \cY$:
            $\Pr{\mathrm{CS}(P, r) = y} = P(y)$.
        \item \emph{(Error guarantee)} For all $P, Q \in \Delta(\cY)$:
            $\Pr{\mathrm{CS}(P, r) \neq \mathrm{CS}(Q, r)} \le \varepsilon\!\rbr{\mathrm{TV}(P, Q)}$,
            where $r$ is the same random string in both calls.
    \end{itemize}
\end{definition}

In particular, $\varepsilon(0) = 0$, so $\mathrm{CS}(P, r) = \mathrm{CS}(Q, r)$ with probability $1$ whenever $P = Q$.
\cite{Broder97} construct a correlated sampling strategy achieving
$\varepsilon(\delta) \le \frac{2\delta}{1+\delta}$.

\begin{corollary}\label{cor:pointwise_label_inv}
    Let $\cA : \cX^n \to \cY$ be a $\rho$-replicable algorithm for a symmetric statistical
    problem $\cT$ with failure probability $\beta$, where $\cX$ and $\cY$ are finite.
    Let $\cA'$ be the algorithm from \Cref{cor:label_inv}, and let $P_S \in \Delta(\cY)$ denote
    the output distribution of $\cA'(S)$ over its internal randomness.
    Define $\cB(S; r) = \mathrm{CS}(P_S, r)$ using the correlated sampling strategy of \cite{Broder97},
    which satisfies $\varepsilon(\delta) \le \frac{2\delta}{1+\delta}$.
    Then $\cB$ is $2\rho$-replicable, satisfies $\cB(S; r) = \cB(\sigma(S); r)$ for all domain
    permutations $\sigma$ and all fixed $r$, solves $\cT$ with failure probability $\beta$,
    and uses the same number of samples as $\cA$.
\end{corollary}

\begin{proof}
    \emph{Pointwise label invariance.}
    By \Cref{cor:label_inv}, $\cA'$ is distributionally label-invariant, so $P_S = P_{\sigma(S)}$
    for every $\sigma$ and every $S$.
    Since $\mathrm{CS}$ is a deterministic function of the distribution and the seed,
    $\cB(S; r) = \mathrm{CS}(P_S, r) = \mathrm{CS}(P_{\sigma(S)}, r) = \cB(\sigma(S); r)$
    for every fixed $r$, where the middle equality uses $P_S = P_{\sigma(S)}$ and the fact that
    $\mathrm{CS}$ is a deterministic function (so equal inputs give equal outputs).

    \emph{Replicability.}
    Let $S_1, S_2 \overset{\mathrm{iid}}{\sim} D^n$ and let $r$ be the shared CS seed, independent of $S_1, S_2$.
    Since $P_{S_1}$ and $P_{S_2}$ are fixed distributions once $S_1, S_2$ are fixed,
    the error guarantee of $\mathrm{CS}$ gives, for any fixed $S_1, S_2$:
    \begin{align*}
        \Pru{r}{\cB(S_1; r) \neq \cB(S_2; r)}
        \le \varepsilon\!\rbr{\mathrm{TV}(P_{S_1}, P_{S_2})}
        \le \frac{2\,\mathrm{TV}(P_{S_1}, P_{S_2})}{1 + \mathrm{TV}(P_{S_1}, P_{S_2})}
        \le 2\,\mathrm{TV}(P_{S_1}, P_{S_2}).
    \end{align*}
    It remains to bound $\Exu{S_1, S_2}{\mathrm{TV}(P_{S_1}, P_{S_2})}$.
    Let $\tilde{r}$ be an independent copy of the shared randomness of $\cA'$, independent of $r$ and $S_1, S_2$.
    Running $\cA'(S_1; \tilde{r})$ and $\cA'(S_2; \tilde{r})$ with the same $\tilde{r}$ induces a coupling
    of $P_{S_1}$ and $P_{S_2}$, so by the definition of TV distance:
    \begin{align*}
        \mathrm{TV}(P_{S_1}, P_{S_2})
        \le \Pru{\tilde{r}}{\cA'(S_1; \tilde{r}) \neq \cA'(S_2; \tilde{r})}.
    \end{align*}
    Taking expectation over $S_1, S_2$ and combining:
    \begin{align*}
        \Pru{S_1, S_2, r}{\cB(S_1; r) \neq \cB(S_2; r)}
        &\le 2\,\Exu{S_1, S_2}{\mathrm{TV}(P_{S_1}, P_{S_2})} \\
        &\le 2\,\Pru{S_1, S_2, \tilde{r}}{\cA'(S_1; \tilde{r}) \neq \cA'(S_2; \tilde{r})}
        \le 2\rho,
    \end{align*}
    where the last inequality is $\rho$-replicability of $\cA'$.

    \emph{Correctness.}
    By marginal correctness of $\mathrm{CS}$, $\cB(S; r) \sim P_S$ marginally over $r$.
    Since $P_S$ is the output distribution of $\cA'(S)$, and $\cA'$ inherits failure probability $\beta$
    from $\cA$ by \Cref{cor:label_inv}, so does $\cB$.
\end{proof}

\begin{remark}
    \Cref{cor:pointwise_label_inv} is a structural result: computing $P_S$ exactly
    requires marginalizing over all internal randomness of $\cA'$, which may be computationally
    expensive. The result shows existence of a pointwise label-invariant algorithm, not
    efficient constructibility.
\end{remark}

\paragraph{Comparison with \cite{aamand2025structure}.}
Theorem~1.2 of \cite{aamand2025structure} establishes label invariance for $\rho$-replicable
binary hypothesis testers of symmetric properties, with no loss in $\rho$.
\Cref{cor:label_inv} gives the same guarantee for \emph{any} output space and \emph{any} symmetric
statistical problem.
Their construction represents $\cA$ via a scalar acceptance probability $f(S) \in [0,1]$
and symmetrizes by averaging over all $|\cX|!$ domain permutations; this is specific to binary output,
where a single scalar fully characterizes the output distribution.
Our construction applies a single random domain permutation from the shared randomness,
giving a black-box transformation that requires no structural access to $\cA$.

Regarding the \emph{strength} of the guarantee: their definition of label invariance
(Definition~4.2) is distributional, matching the definition of distributional label invariance in \Cref{def:label_inv},
but their construction---being a deterministic symmetric function of $S$---achieves
the stronger pointwise notion as a byproduct.
\Cref{cor:label_inv} achieves only the distributional guarantee;
the pointwise version requires the additional correlated sampling step of \Cref{cor:pointwise_label_inv},
at the cost of a factor of $2$ in $\rho$.
In general this loss is negligible as one can replace $\rho$ with $\rho/2$ with no asymptotic cost (see \Cref{remark:rho_2}).

\section{Boosting Success Probability}
\label{sec:boost}
In this section, we show that for a broad class of statistical problems, a replicable algorithm's success probability can be boosted at relatively small cost in sample complexity.
Our results hold for problems parameterized by an\enquote{accuracy} parameter, which we will denote with $\alpha$, with smaller values corresponding to more accurate solutions.
The two properties we require from the problem are as follows:
\begin{enumerate}
    \item The dependence of the replicable sample complexity on the accuracy parameter is relatively mild (e.g., polynomial). More concretely, doubling the accuracy should not cause a dramatic increase in sample complexity.
    \item 
    The accuracy should be, with high probability, approximately testable with a relatively small sample complexity. Note the test does not need to be replicable.
\end{enumerate}
For these problems, we show that the dependence of the sample complexity on the failure probability is no worse than the dependence of either the tester (which need not be replicable) or a non-replicable algorithm for the original problem.
For most problems, designing a tester is no more difficult than designing a non-replicable algorithm as such, the dependence on $\beta$ is the same as the non-replicable counter-part. 
Formally, we prove the following theorem.
\begin{theorem}\label{thm:boost_beta}
    Let $\cT_{\alpha}$ be a class of statistical problems
    parameterized by $\alpha$ with sample complexity $n_{\text{replic}}(\rho, \alpha, \beta)$, where $\rho$ and $\beta$ denote, respectively, the replicability parameter and the failure probability.
    Assume that for any $\alpha, \beta > 0$,
    there exists a tester algorithm that has sample complexity $n_{\text{tester}}(\alpha, \beta)$ and
    satisfies the following guarantees:
    \begin{itemize}
        \item For any valid solution for $\cT_{\alpha}$, the tester outputs \emph{accept} with probability at least $1-\beta$.
        \item For any invalid solution for $\cT_{2\alpha}$, the tester outputs \emph{reject} with probability at least $1-\beta$.
    \end{itemize}
    Let $n_{\text{non-replic}}(\alpha, \beta)$ denote the sample complexity of the non-replicable version of the problem.
    Then there exists a $\rho$-replicable algorithm for $\cT_{\alpha}$ with sample complexity
    \begin{align*}
        n_{\textnormal{new}}(\rho, \alpha, \beta) = 
        n_{\textnormal{replic}}\left(\frac{\rho}{4}, \frac{\alpha}{2}, \frac{\rho}{4}\right)
        + n_{\textnormal{tester}}\left(\frac{\alpha}{2}, \frac{\min\{\rho,\beta\}}{4}\right)
        + n_{\textnormal{non-replic}}\left(\alpha, \frac{\beta}{2}\right).
    \end{align*}

\end{theorem}
\begin{proof}
    We first run the original replicable algorithm with parameters $\rho/4, \alpha/2, \rho/4$.
    Let $u$ denote the output.
    We give $u$ to the tester with parameters $\alpha/2$ and $\min\{\rho, \beta\}/4$. If the tester outputs accept, we output $u$ otherwise, we run a non-replicable algorithm with parameters $\alpha, \beta/2$ and give its output as our output.
    A formal pseudocode is provided in \Cref{alg:hybrid}.

    We first analyze the replicability of our algorithm.
    Consider two different runs of the algorithm with shared random bits. With probability $1-\rho/4$, they will produce the same output $u$. 
    With probability $1-\rho/4$, this output is a valid solution to $\cT_{\alpha/2}$. Taking union bound, with probability $1-\rho/2$ both algorithms have the same output and this output is a valid solution to $\cT_{\alpha/2}$. Condition on this event. By definition of the tester, with probability $1-\min\{\beta, \rho\}/2$, the algorithm will output accept for both runs. Therefore, both algorithms will share the same output with probability at least $1-\rho$.

    We next analyze correctness. 
    We will show that, conditioned on the value of $u$, the output is correct with probability at least $1-\beta$. By total expectation, this implies that output is correct with probability at least $1-\beta$ overall.
    Let $u'$ denote the output of the non-replicable algorithm; we assume for the sake of analysis that this output is always computed, even if $u$ is chosen as output.
    We consider two cases, depending on whether or not $u$ is a valid solution to $\cT_{\alpha}$.
    \begin{itemize}
        \item If $u$ is a valid solution then, with probability $1-\beta/2$, the value $u'$ is also a valid solution. Since the output is either $u$ or $u'$, in this case the claim holds.
        \item If $u$ is not a valid solution, then with probability $1- \min\{\rho, \beta\}/4 \ge 1-\beta/4$, the tester outputs reject which means $u'$ is outputted. With probability $1-\beta/2$, the value $u'$ is a valid solution. Taking union bound, it follows that with probability at least $1-3\beta/4$, the output is a valid solution in this case as well.
    \end{itemize}
    Therefore, in both cases the algorithm has a correct output with probability at least $1-\beta$ and we are done.
\end{proof}

\begin{algorithm}[H]
    \caption{Construction of $\cA_{\textnormal{hybrid}}$ for $\cT_\alpha$}
    \label{alg:hybrid}
    \KwIn{
      $\cA_{\textnormal{replic}}$, $\cA_{\textnormal{tester}}$, $\cA_{\textnormal{non\textnormal{-}replic}}$; parameters $\rho,\alpha,\beta$
    }
    \KwOut{Solution $y$ for $\cT_\alpha$}
    
    Set $(\rho_{\textnormal{rep}},\alpha_{\textnormal{rep}},\beta_{\textnormal{rep}})\leftarrow(\rho/4,\alpha/2,\rho/4)$\;
    Set $(\alpha_{\textnormal{test}},\beta_{\textnormal{test}})\leftarrow(\alpha/2,\min\{\rho,\beta\}/4)$\;
    Set $\beta_{\textnormal{nonrep}}\leftarrow \beta/2$\;

    \tcc{Run the replicable solver and test its output}
    $u \leftarrow \cA_{\textnormal{replic}}(\rho_{\textnormal{rep}},\alpha_{\textnormal{rep}},\beta_{\textnormal{rep}})$\;
    $t \leftarrow \cA_{\textnormal{tester}}(u;\alpha_{\textnormal{test}},\beta_{\textnormal{test}})$ \tcp*{$t\in\{\textnormal{accept},\textnormal{reject}\}$}
    
    \If{$t=\textnormal{accept}$}{
        \Return $u$\;
    }
    
    \tcc{Fallback to a non-replicable solver}
    $u' \leftarrow \cA_{\textnormal{non\textnormal{-}replic}}(\alpha,\beta_{\textnormal{nonrep}})$\;
    \Return $u'$\;
\end{algorithm}

\begin{remark}\label{rem:boost}
    In many cases, such as the applications we discuss below, the sample complexity depends polynomially on the $\rho, \alpha, \beta$. 
    Additionally, we may assume $\beta < \rho$ as otherwise one can just fallback on using $\cA_{\textnormal{replic}}$ directly. 
    In these cases, up to constant factors, the sample complexity of the new algorithm can be written as
    \begin{align*}
        n_{\text{replic}}(\rho, \alpha, \rho)
        + n_{\text{tester}}(\alpha, \beta)
        + n_{\text{non-replic}}(\alpha, \beta)
        .
    \end{align*}
    In other words, the dependence on failure probability is essentially separated from replicability. Moreover, in many cases designing the tester is no more difficult than designing a non-replicable algorithm and the sample complexity becomes
    $O(n_{\text{replic}}(\rho, \alpha, \rho) + n_{\text{non-replic}}(\alpha, \beta))$. 
\end{remark}

\subsection{Application: Statistical Queries}
\label{sec:app_sq_boost}
As a concrete example, consider the statistical query problem: given a function $\phi: \cX \to [0, 1]$, we aim to estimate $\Exu{X \sim D}{\phi(X)}$ with additive error $\alpha$.
The problem has both of the aforementioned properties.
The first property holds because the dependence on $\alpha$ is polynomial; specifically the sample complexity scales with $\alpha^{-2}$. Consequently, replacing $\alpha$ with $\alpha/2$ causes the sample complexity to stay the same up to constant factors.
The second property holds because testing whether or not a value $\hat{\mu}$ is at most $\alpha$ away from the mean or it is at least $2\alpha$ away from the mean requires at most $O(\frac{\log(1/\beta)}{\alpha^2})$ samples where $\beta$ denotes the failure probability. Concretely, we can first (non-replicably) estimate the mean with additive error $\alpha/4$. If the value $\hat{\mu}$ is at most $\alpha$ away from the mean, then, by the triangle inequality, it will be at most $5\alpha/4$ away from our estimate.
Conversely, if it is at least $2\alpha$ away from the mean then it will be at least $7\alpha/4$ away from our estimate. 
As such, we can invoke \Cref{thm:boost_beta} to obtain an improved sample complexity bound.
Formally, the following result holds.
\begin{theorem}
    For any $\alpha, \rho, \beta > 0$, there exists a $\rho$-replicable statistical oracle with error $\alpha$ and sample complexity 
    \begin{align*}
        n(\alpha, \beta, \rho) \le 
        O\rbr{\frac{\log(1/\rho)}{\alpha^2 \rho^2}
        + \frac{\log(1/\beta)}{\alpha^2}
        }
    \end{align*}
\end{theorem}
\begin{proof}
    We assume without loss of generality that $\beta < \rho/4$. 
    If $\beta$ is larger, we can simply invoke the theorem with $\beta = \rho/4$ to get the same sample complexity bound.
    
    As shown by \cite{impagliazzo2022reproducibility} (see Theorem 2.3 in their paper), for $\beta < \rho/2$, there exists a replicable algorithm with sample complexity
    \begin{math}
        = O(\frac{\log(1/\beta)}{\alpha^2 (\rho - 2\beta)^2})
        .
    \end{math}
    Since $2\beta < \rho/2$, this can be written as
    $O(\frac{\log(1/\beta)}{\alpha^2 \rho^2})$. 
    While this bound is for $\beta < \rho/4$, it is clear that
    an algorithm with failure probability $\beta$ is also an algorithm with failure probability $\beta'$ for any $\beta' > \beta$. Therefore,
    \begin{align*}
        n_{\textnormal{replic}}(\rho, \alpha, \beta)
        = O(\frac{\log(1/\min\{\rho, \beta\})}{\alpha^2 \rho^2})
        .
    \end{align*}
    
    A tester for the problem can be implemented as mentioned above.
    Formally, to distinguish values that are at most $\alpha$ far from the true mean from those that are at least $2\alpha$ far, we can estimate the mean up to additive error $\alpha/4$ and failure probability $\beta$ 
    with sample complexity
    \begin{align*}
        n_{\textnormal{tester}}(\alpha, \beta) = 
        O(\frac{\log(1/\beta)}{\alpha^2}).
    \end{align*}
    This follows from Hoeffding's inequality which ensures that, using $n$ samples, the probability that the empirical mean is more than $\alpha$ away from the true mean is at most
    $2e^{-2n\alpha^2}$.
    Letting $\hat{\mu}$ denote the estimate,
    we can then accept values that are closer than $3\alpha/2$ to $\hat{\mu}$ and reject those that are not.

    Finally, by Hoeffding's inequality, the sample complexity of the non-replicable version of the problem is bounded by
    \begin{align*}
        n_{\textnormal{non-replic}}(\alpha, \beta) = 
        O(\frac{\log(1/\beta)}{\alpha^2}).
    \end{align*}
    
    Plugging in \Cref{thm:boost_beta}, It follows that the total sample complexity is at most
    \begin{align*}
        n(\alpha, \beta, \rho) \le 
        O\rbr{\frac{\log(1/\rho)}{\alpha^2 \rho^2}
        + \frac{\log(1/\beta)}{\alpha^2}
        }
        .
    \end{align*}
    as claimed.
\end{proof}

\subsection{Application: PAC Learning}
\label{sec:app_boot_pac}

We now discuss the implementation of a tester for a PAC learner.
For simplicity, we focus only on realizable PAC learning as  there are already multiple results from prior work for which our framework applies.

We start by implementing a tester as required by \Cref{thm:boost_beta}.
While the implementation is fairly standard, for completeness we provide a proof.
\begin{lemma}\label{lm:tester}
    Let $X$ be a Bernoulli random variable with a unknown mean $\mu \in [0, 1]$.
    Let $\alpha \in (0, 1)$ be an arbitrary parameter.
    There exists a tester with sample complexity $O(\log(1/\beta)/\alpha)$
    with the following property:
    \begin{itemize}
        \item If $\mu \ge 2\alpha$, the algorithm outputs accept
        \item If $\mu \le \alpha$ the algorithm outputs reject.
    \end{itemize}
\end{lemma}

\begin{proof}
    Draw $n$ i.i.d.\ samples $X_1, \dots, X_n \sim \mathrm{Bern}\rbr{\mu}$ and
    let $\hat{\mu} = \frac{1}{n}\sum_{i=1}^n X_i$.
    Set the threshold $\tau = \frac{3}{2}\alpha$ and define the tester to
    \emph{accept} if $\hat{\mu} \ge \tau$ and \emph{reject} otherwise.
    
    We use the following variant of the Chernoff bound: for any $\delta > 0$,
    \begin{align*}
        \Pr{\abs{\hat{\mu} - \mu} \ge \delta \mu}
        \le 
        \exp\rbr{-\frac{\delta^2 n \mu}{3}} .
    \end{align*}
    
    We first consider the case $\mu \ge 2\alpha$.
    We bound $\Pr{\hat{\mu} < \tau}$, a lower-tail deviation.
    It is clear that the probability is maximized for $\mu = 2\alpha$ so we assume, without loss of generality, that
    $\mu = 2\alpha.$
    It follows that $\tau = (1 - \delta)\mu$, where
    $\delta = 1/4$.
    By Chernoff,
    \begin{align*}
        \Pr{\hat{\mu} < \tau}
        = \Pr{\hat{\mu} \le (1 - \delta)\mu}
        \le \exp\rbr{-\frac{\delta^2 n \mu}{3}}
        \le \exp\rbr{-\frac{n \alpha}{24}} .
    \end{align*}
    Hence if $n \ge 24 \log(1/\beta)/\alpha$, we have $\Pr{\text{reject}} \le \beta$.
    
    We next consider the case $\mu \le \alpha$.
    We now bound $\Pr{\hat{\mu} \ge \tau}$, an upper-tail deviation.
    Similar to before, the probability is maximized for
    $\mu = \alpha$ so we assume without loss of generality that this is the case.
    Write $\tau = (1 + \delta)\mu$, where
    $\delta = 1/2$.
    Then,
    \begin{align*}
        \Pr{\hat{\mu} \ge \tau}
        = \Pr{\hat{\mu} \ge (1 + \delta)\mu}
        \le \exp\rbr{-\frac{\delta^2 n \mu}{3}}
        \le \exp\rbr{-\frac{n \alpha}{12}} .
    \end{align*}
    Thus, if $n \ge 12 \log(1/\beta)/\alpha$, we have $\Pr{\text{accept}} \le \beta$.
    
    Thus, taking $n \ge 24 \log(1/\beta)/\alpha$ ensures both guarantees:
    \begin{itemize}
        \item If $\mu \ge 2\alpha$, the tester accepts with probability at least $1-\beta$.
        \item If $\mu \le \alpha$, the tester rejects with probability at least $1-\beta$.
    \end{itemize}
    Hence the desired tester exists with sample complexity $O\rbr{\log(1/\beta)/\alpha}$.
\end{proof}

The above lemma provides a tester which we can use for PAC learning. Concretely, for any hypothesis $h \in \cH$, we can use $O(\log(1/\beta)/\alpha)$ samples to test whether it is $\alpha$-accurate. 
Since a non-replicable PAC learning algorithm exists with sample complexity nearly linear in the VC dimension, we obtain the following result.

\begin{lemma}\label{lm:boost_pac}
    Let $\cH$ be a binary concept class with VC dimension $d$.
    Assume that for any $\rho, \alpha,  \beta \in (0, 1/2)$, there exists a $\rho$-replicable algorithm with sample complexity $n_{\textnormal{replic}}(\rho, \alpha, \beta)$.
    For $\rho, \alpha, \beta \in (0, 1/2)$, there exists a $\rho$-replicable algorithm for solving the problem with sample complexity
    \begin{align*}
        n_{\textnormal{replic}}(\rho/4, \alpha/2, \rho/4)
        + O\rbr{
        \frac{d\log(1/\alpha) + \log(1/\beta)}{\alpha}
        }
        .
    \end{align*}
    \begin{proof}
        The theorem follows directly from \Cref{thm:boost_beta} by using
        \Cref{lm:tester} for $\cA_{\textnormal{tester}}$ and
        using a standard non-replicable learner for 
        $\cA_{\textnormal{non-replic}}$.
        Specifically, it is well-known that
        we have
        \begin{math}
            n_{\textnormal{non-replic}}(\alpha, \beta) \le 
            O\rbr{
            \frac{d\log(1/\alpha) + \log(1/\beta)}{\alpha}
            }.
        \end{math}
        By \Cref{lm:tester}, we have
        \begin{math}
            n_{\textnormal{tester}}(\alpha, \beta) \le 
            O\rbr{\frac{\log(1/\beta)}{\alpha}}
            .
        \end{math}
        Plugging these bounds in \Cref{thm:boost_beta} we obtain the desired bound.
    \end{proof}
\end{lemma}

As an example application, one can immediately obtain the following corollary for replicable learning of classes with finite little stone dimension.
\begin{corollary}
    Any class $\cH$ with Littlestone dimension $d$
    can be learned with sample complexity
    \begin{align*}
        n(\rho, \alpha, \beta) \le 
        \wtilde{O}\rbr{
            \frac{d^{12}}{\alpha^2 \rho^2} + 
            \frac{\log(1/\beta)}{\alpha}
        }
        .
    \end{align*}
\end{corollary}
\begin{proof}
    As shown by \cite{bun2023stability}, we have
    \begin{align*}
        n_{\textnormal{replic}}(\rho, \alpha, \beta)
        \le 
        \wtilde{O}(\frac{d^{12}\log^{3}(1/\beta)}{\alpha^2 \rho^2}
        ),
    \end{align*}
    which in turn implies
    \begin{math}
        n_{\textnormal{replic}}(\rho/4, \alpha/2, \rho/4)
        \le 
        \wtilde{O}(\frac{d^{12}}{\alpha^2 \rho^2}
        )
        .
    \end{math}
    It is well-known that the Littlestone dimension is larger or equal to the VC dimension.
    Therefore, the VC dimension is at most $d$.
    Applying \Cref{lm:boost_pac} finishes the proof.
\end{proof}
We additionally obtain the following result for learning finite classes.
\begin{corollary}
    For any finite class $\cH$, there exists
    a $\rho$-replicable algorithm for outputting an $\alpha$-accurate learner with failure probability $\beta$ that has
    sample complexity
    \begin{align*}
        O\rbr{
        \frac{\log^2|\cH| + \log(1/\rho)}{\alpha \rho^2}\log^{3}(1/\rho) + \frac{\log(1/\beta)}{\alpha}
        }
        .
    \end{align*}
\end{corollary}
\begin{proof}
    As shown by \cite{bun2023stability} we have
    \begin{align*}
        n_{\textnormal{replic}}(\rho, \alpha, \rho)
        \le 
        O\rbr{
        \frac{\log^2|\cH| + \log(1/\rho)}{\alpha \rho^2}\log^{3}(1/\rho).
        }
    \end{align*}
    Additionally, the VC dimension of a finite class is at most $\log|\cH|$.
    Furthermore, for finite $H$, we can apply the sharper
    bound
    \begin{align*}
        n_{\textnormal{non-replic}}
        (\rho, \alpha, \beta)
        \le 
        \frac{\log |\cH| + \log(1/\beta)}{\alpha}
        ,
    \end{align*}
    removing a $\log(1/\alpha)$ factor.
    Other than the $\log(1/\beta)/\alpha$, the rest of the bound is absorbed by the larger terms already present in $n_{\textnormal{replic}}$ and the result follows.
\end{proof}

\begin{corollary}
    For any $\rho, \alpha, \beta > 0$, there exists a $\rho$-replicable algorithm for PAC learning threshold functions over the finite domain ${0, 1, \dots, d}$ that achieves error at most $\alpha$ and failure probability at most $\beta$, with sample complexity
    \begin{align*}
        n(\rho, \alpha, \beta) \le 
        \wtilde{O}\rbr{
            \frac{(\log^*d)^3}{\alpha^2 \rho^2}
            + \frac{\log(1/\beta)}{\alpha}
        }
        .
    \end{align*}
\end{corollary}
\begin{proof}
    As shown by \cite{bun2023stability},
    \begin{align*}
        n_{\textnormal{replic}}(\rho, \alpha, \rho)
        \le 
        \wtilde{O}\rbr{
            \frac{(\log^*d)^3 \log^{2}(1/\beta)}{\alpha^2 \rho^2}
        }
        .
    \end{align*}
    Since the VC dimension of threshold functions is at most $1$, the bound follows.
\end{proof}
\subsection{Application: Heavy-hitters}
\label{sec:heavy}
We next discuss the approximate heavy-hitters problem. In this problem, we are given a distribution $\cD$ and two parameters $\nu, \epsilon$. The goal is to output a list $L$ such that:
\begin{itemize}
    \item For any $x$ such that $\Pru{X \sim \cD}{X=x} \ge \nu$ we have $x  \in L$, and
    \item For any $x$ such that $\Pru{X \sim \cD}{X=x} \le \nu - \epsilon$ we have $x  \notin L$.
\end{itemize}

For this problem, we prove the following theorem.
\begin{theorem}
    For any $\rho, \nu \in (0, 1)$, $\beta < \rho$ and $\epsilon \in (0, \nu)$, there exists a $\rho$-replicable algorithm for solving the heavy-hitters problem with failure probability $\beta$ and sample complexity
    \begin{align*}
        \wtilde{O}(
        \frac{1}{\nu\epsilon^2\rho^2}
        + \frac{\log(1/\beta)}{\epsilon^2}
        )
        .
    \end{align*}
\end{theorem}

We assume throughout that $\epsilon \le \nu/2$. If not, we can solve the problem with $\epsilon= \nu/2$. While this is clearly a harder problem, and any solution is a valid solution for the original problem, the claimed sample complexity bound is the same.

Given parameters $\nu, \epsilon$, define
    $\nu', \epsilon'$ as 
    $\nu' = \nu - \epsilon/4$ and $\epsilon' = \epsilon/2$ which implies $\nu' - \epsilon' = \nu - \epsilon + \epsilon/4$. 
    
We begin by implementing a tester for this problem. 
\begin{lemma}
    There exists an algorithm with sample complexity
    at most
    $O\rbr{
            \epsilon^{-2}\log(\nu^{-1}\beta^{-1})
        }$
        that accepts any valid solutions to $\cT_{\nu', \epsilon'}$ and rejects any invalid solutions for $\cT_{\nu, \epsilon}$ with failure probability at least $1-\beta$.
\end{lemma}
\begin{proof}
    We will show a tester that accepts all valid solutions to $\cT_{\nu', \epsilon'}$ and rejects all invalid solutions for $\cT_{\nu, \epsilon}$ with high probability.

    To implement this, we start by checking the size of $L$. If the size of $L$
    exceeds $4/\nu$, then we reject. Note that any valid solution to $\cT_{\nu',
    \epsilon'}$ cannot have any elements with probability less than $\nu/4$. As
    such, its size cannot exceed $4/\nu$.

    Next, for all the elements in $L$, we estimate their probability up to error $\epsilon/32$ and failure probability $\frac{\beta}{8|L|}$. 
    This can be done with sample complexity 
    $O(\log(4|L|/\beta)/\epsilon^2)$.
    Note that the same sample set can be used to estimate all probabilities.
    By union bound, 
    it follows that with probability at least $1-\beta/4$, all of the probability estimates are correct up to the desired error.
    Now, if any of the probability estimates are below $\nu' - \epsilon' - \epsilon/16$, we reject. We observe that a valid solution to $\cT_{\nu', \epsilon'}$ cannot be rejected at this stage because all probability estimates are going to be at least $\nu' - \epsilon' - \epsilon/32$.

    We then take a new sample set $A$ consisting of
    $n_1 = \Theta(\nu^{-1}\log(\nu^{-1}\beta^{-1}))$ samples.
    With probability $1-\beta/16$, the set contains all $x$ such that
    $\Pru{D}{x} \ge \nu'$.
    We then estimate the probability of all elements in $A$ up to additive error $\epsilon/32$ and failure probability $\frac{\beta}{8|A|}$. This takes
    $O(\log(4n_1/\beta)/\epsilon^2)$ samples in total. By union bound, all estimates are correct with probability $1 - \beta/8$. If there is some $x \in A$ for which the estimate is at least $\nu - \epsilon/16$ but $x \notin L$, then we reject.
    We again note that a valid solution to $\cT_{\nu', \epsilon'}$ cannot be rejected; if it's rejected because of some $x \in A \backslash L$, then that $x$ must have probability at least $\nu - \epsilon/16 - \epsilon/32 > \nu'$, which means $L$ was not a valid solution after all.

    We next claim that any invalid solution to $\cT_{\nu, \epsilon}$ will be
    rejected assuming the failure events mentioned above do not occur. Consider
    some list $L$. If $L$ contains some $x$ with probability less than $\nu -
    \epsilon$, then the estimate for the probability will be at most $\nu -
    \epsilon + \epsilon/32$ which means it will get rejected. Let us assume
    that there is some element $x$ with probability more than $\nu$ that does
    not appear in $L$. The probability will appear in the set $A$ if the
    failure events do not occur. Its probability estimate will be at least $\nu
    - \epsilon/32$ which is larger than $\nu - \epsilon/16$. Therefore, since
    $x \notin L$, the set $A$ will get rejected.

    Since we can assume without loss of generality that $|L| \le \Theta(\nu^{-1})$, the total sample complexity of the algorithm is 
    \begin{align*}
        O\rbr{
            \log(4|L|/\beta)/\epsilon^2 + 
            \nu^{-1}\log(\nu^{-1}\beta^{-1})
            + 
            \log(\nu^{-1}\beta^{-1})/\epsilon^2
        }
        = 
        O\rbr{
            (\log(\nu^{-1}) + \log(\beta^{-1}))
            (\epsilon^{-2} + \nu^{-1})
        }
        .
    \end{align*}
    This implies the bound in the lemma statement since $\epsilon \le \nu$.
    The failure probability is also at most $\beta$ by union bound.
\end{proof}
We next consider a non-replicable algorithm for the problem. We start by
sampling a set $A$ consisting of $n_1 = \Theta(\nu^{-1} \log(\nu^{-1}
\beta^{-1}))$ elements. We claim that, with probability at least $1-\beta/8$, the
set contains all items with probability at least $\nu$. This is because for any
such $x$, the probability that it does not appear in $A$ is at most
\begin{align*}
    (1-\Pru{\cD}{x})^{n_1} \le e^{-n_1\nu} \le 
    \nu\beta/16.
\end{align*}
Since the number of such elements is a most $\nu^{-1}$, a union bound implies the claim. 
Next, for all elements in $A$, we estimate the probability of each value with additive error $\epsilon/32$.
To achieve this, we take $n_2$ samples where
$n_2 = O(\log(n_1/\beta)/\epsilon^2)$. This implies that each estimate is correct with probability at least $1- \frac{\beta}{8n_1}$. By union bound, all estimates are correct with probability at least $1- \frac{\beta}{8}$. For any element, if the estimate is above $\nu - \epsilon/16$, we put in the output list $L$ and if the estimate is below this threshold we do not.
The total sample complexity of the algorithm is at most
\begin{align*}
    n_1 + n_2 \le 
    O\rbr{
        \frac{\log(\nu^{-1}\beta^{-1})}{\nu}
        +
        \frac{\log(\nu^{-1}\beta^{-1})}{\epsilon^2}
    }
    \le O\rbr{
        \frac{\log(\nu^{-1}\beta^{-1})}{\epsilon^2}
    }
\end{align*}

As shown by \cite{esfandiari2024replicable} however, we have
\begin{align*}
    n_{\textnormal{replic}}(\rho, \nu, \epsilon, \beta)
    \le 
    \wtilde{O}(
    \frac{\log(1/\beta)}{\nu\epsilon^2\rho^2}
    ).
\end{align*}
Applying \Cref{thm:boost_beta}, the theorem follows.

\subsection{Application: Best Arm Problem}
\label{sec:app_best_arm_boost}
We next consider the \emph{best arm problem}.
Here, we are given a set of arms $A$. Each arm $a \in A$ has a corresponding distribution over $[0, 1]$ with mean $\mu_a$. The goal is to find an arm $a$ satisfying $\mu_a \ge \max_{a'}\mu_{a'} - \alpha$.
As shown by \cite{hopkins2025generative}, the problem can be solved with
\begin{math}
    n_{\textnormal{replic}}(\rho, \alpha, \beta) = 
    O(\frac{1}{\rho^2 \alpha^2} \log^3(|A|/\beta))
\end{math}
samples per arm (see Theorem 3.3 in their paper). 

We next implement a tester for this problem. 
To achieve this, we first estimate the mean of each arm with additive error $\alpha/16$ and failure probability $\frac{\beta}{16|A|}$. This can be done with $O(\log(|A|/\beta)/\alpha^2)$ samples per arm by Hoeffding's inequality. 
Let $\hat{\mu}_{a}$ denote the obtained estimates.
A union bound implies that, with probability at least $\beta/8$ we have
$\abs{\mu_{a} - \hat{\mu}_{a}} \le \alpha/16$ for all $a$.
To check if an arm $a$ is indeed a correct solution, we check if it satisfies
$\hat{\mu}_{a} \ge \max_{a'} \hat{\mu}_{a'} - \frac{3}{2}\alpha$.
If the solution is indeed valid for $\cT_{\alpha}$, then we have
$\mu_{a} \ge \max_{a'} \mu_{a'} - \alpha$. Given the bound on $\abs{\mu_{a} - \hat{\mu}_{a}}$, this implies that the tester will accept. Conversely, if a solution $a$ is invalid for $\cT_{2\alpha}$, then for some arm $a'$ we have
$\mu_{a'} \ge \mu_{a} + 2\alpha$. Given the same bound on differences, we conclude that
$\hat{\mu}_{a'} \ge \hat{\mu}_{a} + \frac{3}{2}\alpha$ and as such the tester will reject.

The sample complexity of the non-replicable version of the problem can be similarly bounded. We estimate each $\mu_{a}$ with additive error $\alpha/16$ and failure probability $\frac{\beta}{16|A|}$ using 
$O(\log(|A|/\beta)/\alpha^2)$ samples per arm.
We then output the arm with maximum estimated mean.
A union bound implies that with probability at least $\beta/8$, all estimates are correct up to the desired error correct. Therefore, the chosen arm is a solution to $\cT_{\alpha}$ as required. 

Combining the above arguments and invoking \Cref{thm:boost_beta}, we obtain the following result.
\begin{theorem}
    Let $\beta \le \rho \le 1/2$. There exists a $\rho$-replicable algorithm for solving the best-arm problem over a set $A$ with additive error $\alpha$ and failure probability $\beta$, using
    \begin{align*}
        n_{\textnormal{new}}(\rho, \alpha, \beta)
         = 
         O\rbr{
         \frac{1}{\rho^2 \alpha^2} \log^3(\frac{|A|}{\rho})
         + 
         \frac{\log(\frac{|A|}{\beta})}{\alpha^2}
         }
    \end{align*}
    samples per arm.
\end{theorem}

\section{Omitted Proofs}
\label{sec:omitted}
\subsection{Analysis of Naive Composition}
\label{app:naive_is_tight}
In this section we demonstrate that the standard union bound analysis for naive composition is, up to constants, tight.
Consider an algorithm $\cA$ for some problem $\cT$.
The choice of problem is arbitrary and all we require is that $\cA$ is not $\rho$-replicable with $\rho=0$; e.g., one can consider the mean estimation problem. 
Let $\rho$ be the \enquote{optimal} replicability parameter for $\cA$; i.e.,
\begin{align*}
    \rho = \inf\cbr{\rho': \text{$\cA$ is $\rho'$-replicable}}.
\end{align*}
Assume that $\rho > 0$. 
By definition of the infimum, the algorithm $\cA$ is $(2\rho)$-replicable.
However, there exist some distribution $D$ such that for two independent samples $S^{(1)}, S^{(2)}$ from $D$ we have
\begin{align*}
    \Pr{\cA(S^{(1)}; r) \ne \cA(S^{(2)};r)}
    \ge \rho/2.
\end{align*}

Let $\cDcomp$ denote the product distribution of $k$ independent copies from $D$. 
Let $\cT_i$ denote the problem $\cT$ over the $i$-th coordinate of the input and define
$\cTcomp = (\cT_1, \dots, \cT_k)$. Let $\cA_i$ denote the algorithm that runs $\cA$ on the $i$-th coordinate of the input samples. Define $\cAcomp = (\cA_1, \dots, \cA_k)$ to be the composition of $\cA_i$.
Given the guarantees on $\cA$, it is clear that $\cA_i$ is $(2\rho)$-replicable.
We will show however that the replicability parameter of $\cAcomp$ is at least $\Omega(\min\{\rho k, 1\})$.

Define the event $\err_i$ as $\err_i = \cbr{\cA_i(S^{(1)}; r) \ne \cA_i(S^{(2)};r)}$ 
Given the guarantees on $\cA$, for any two samples $S^{(1)}, S^{(2)}$ from $\cDcomp$ we have
\begin{align*}
    \Pr{\err_i}
    \ge \rho/2.
\end{align*}
Observe that $\err_i$ depends only on the coordinate $i$ in the samples and the random bits used by $\cA_i$. Since $\cDcomp$ was assumed to be a product distributions, this implies that
the events $\cbr{\err_i}_{i=1}^{k}$ are independent.
It is easy to show however that given $k$ independent events with probability at least $p$, the probability of their union is at least $\min\{kp, 1\}/2$:
\begin{align*}
    \Pr{\text{Union event}}
    \ge
    1 - (1-p)^{k}
    \ge 1-e^{-kp}
    \ge 1 - \max\{1/2, 1 - kp/2\},
\end{align*}
where the last inequality follows from
$e^{-x} \le \max\{1/2, 1 - x/2\}$ which holds for all $x > 0$.
It follows that
\begin{align*}
    \Pr{\cAcomp(S^{(1)}; r) \ne \cAcomp(S^{(2)};r)}
    &=
    \Pr{\cup_{i=1}^k \err_i}
    \ge 
    \frac{1}{2}\min(1, k\rho/2)
    .
\end{align*}
\subsection{Tightness of Composition Bounds}
\label{sec:tight}
A natural question is whether the $\wtilde{O}(nk)$ bound obtained above is tight or one can improve it in terms of the dependence on either $n$ or $k$.
The following theorem shows that the result is indeed tight; in particular, the theorem implies that if we fix either of the parameters $n$ or $k$, we cannot obtain a composition that is truly sublinear in the other parameter.
\begin{theorem}\label{thm:hardness}
    Suppose there exists a constant $c > 0$
    and a mechanism with the following property.
    For any $\beta > 0$ and any collection of $k$ algorithms
    $\cA_1, \dots, \cA_k$ that are each $c$-replicable with sample
    complexity $n$ and failure probability $\beta$, the mechanism produces a
    $0.0001$-replicable algorithm $\cAcomp$ solving the composed problem with
    failure probability as high as $O(\poly(k)\poly(\beta))$ and sample complexity
    \begin{math} \wtilde{O}(n^{\chi}k^{\psi} \polylog(1/\beta)) \end{math} for some constants
        $\chi, \psi$. Then $\chi \ge 1$ and $\psi \ge 1$.
\end{theorem}

In the above statement, since $\beta < 1$, the term $\poly(\beta)$ refers to any function of the form $\beta^{r}$ for some $r > 0$ (e.g., $\sqrt{\beta}$). The theorem shows that one cannot, for instance, obtain a better than linear dependence on $n$ and $k$ by allowing the failure probability of the composition to be $k^{10}\beta^{1/10}$.

The proof of the above theorem follows from the hardness of the \emph{one-way} marginals problem, which is a special case of the multi-dimensional mean estimation problem considered in \Cref{thm:multi_dim_mean}.

\begin{definition}[One-way marginals]\label{def:one-way}
    Let $\cX = \{0, 1\}^{d}$ and set $D$
    to be the product of Bernoulli distributions with parameters $p=(p_1, \dots, p_d)$ on $\cX$;
    i.e.,
    $\Pru{X\sim D}{X=x} = \prod_{i=1}^{d} p_i^{x_i} (1-p_i)^{1-x_i}$.
    We say $v \in \R^d$ is an $\alpha$-accurate solution for the one-way marginal problem if
    $\norm{p - v}_{\infty} \le \alpha$.
\end{definition}
The following lemma lower bounds the sample complexity of this problem.

\begin{lemma}[Theorem 5.3 in \cite{bun2023stability}]\label{lm:marginal_neg}
    There exist constants $C > 0$ with the following property. For any $d >
    C$ any $0.0001$-replicable algorithm that outputs
    an $0.01$-accurate solution to the one-way marginals problem with failure
    probability at most $0.001$ must have sample complexity at least
    $\wtilde{\Omega}(d)$. \end{lemma}

Combining the above lemma with the positive result from \Cref{thm:multi_dim_mean}, we obtain the following lemma which we will then use to prove \Cref{thm:hardness}.

\begin{lemma}\label{lm:hard_instance_exists}
    For any $n, k > 0$, there exists
    statistical problems $\cT_1, \dots, \cT_k$ such that
    the following holds.
    \begin{itemize}
        \item For any $\rho, \beta > 0$, each $\cT_i$ can be solved with a $\rho$-replicable algorithm with failure probability $\beta$ and complexity $\wtilde{O}(n \rho^{-2}\polylog(1/\beta))$.
        \item Solving the composed problem $\cTcomp = (\cT_1, \dots, \cT_k)$
            with a $0.0001$-replicable algorithm and failure probability $0.001$
            requires sample complexity at least $\wtilde{\Omega}(nk)$.
    \end{itemize}
\end{lemma}
\begin{proof}
    Set $d=n$, $\dcomp = nk$, and $\alpha=0.01$.
    Define $\cX = \{0, 1\}^{\dcomp}$ and
    let $D$ be a product of Bernoulli distributions on $\cX$ with parameters $p=(p_1, \dots, p_{\dcomp})$. 
    For any $i\in [k]$, define
    $p^{(i)} = (p_{(i-1)d + 1},\dots, p_{id})$ and let $\cT_i$ be the problem of finding a vector $v \in [0, 1]^{d}$ satisfying $\norm{v - p^{(i)}}_{\infty} \le \alpha$.
    It is clear that each $\cT_i$ is the problem of obtaining an $\alpha$-accurate solution to the one-way marginal problem which has sample complexity $\wtilde{O}(d\rho^{-2}\log(\beta^{-1}))$ by \Cref{thm:multi_dim_mean}.
    
    The composition $\cTcomp$ is
    the problem of finding a set of vectors $v^{(1)}, \dots, v^{(i)}$ such that $\norm{v^{(i)} - p^{(i)}}_{\infty} \le \alpha$ for all $i$.
    If we glue these vectors $v^{(i)}$ into a single vector however, it is clear that this problem is equivalent to finding a vector $v$ satisfying
    $\norm{v - p}_{\infty} \le \alpha$. This is again an instance of the one-way marginal problem with $\dcomp$ instead of $d$ and, by \Cref{lm:marginal_neg}, requires
    sample complexity $\wtilde{\Omega}(\dcomp)$.
    The lemma now follows from the definition of $d$ and $\dcomp$.
\end{proof}

We now use the above lemma to prove \Cref{thm:hardness}.

\begin{proof}[Proof of \Cref{thm:hardness}]
    Assume that such a mechanism exists and that
    if $\cA_i$ have failure probability $\beta$, then
    the output algorithm has failure probability $\Theta(k^{r}\beta^{\ell})$ for some constants $r, \ell > 0$.
    Let $n, k > 0$ be arbitrary values
    and consider the instances $\cT_1, \dots, \cT_k$
    from \Cref{lm:hard_instance_exists}.
    
    Let $\beta = \Theta(\frac{1}{k^{r/\ell}})$ for small enough constant hidden
    under $\Theta(.)$.
    By assumption, the mechanism produces an algorithm for $\cT_i$ with failure probability
    $\Theta(k^{r} \beta^{\ell}) = \Theta(1)$ which, picking the constant for $\beta$ small enough, is at most $0.001$.
    Since $r, \ell$ are constants, we have $\log(1/\beta) = O(\log k)$.
    Therefore, the sample complexity of the composed problem is bounded by
    \begin{align*}
        \wtilde{O}(n^{\chi}k^{\psi} \polylog(1/\beta)) = \wtilde{O}(n^{\chi}k^{\psi} \polylog(k)).
    \end{align*}
    Since the algorithm $\cAcomp$ is $0.0001$-replicable and has failure probability $0.01$ however,
    by the assumption in \Cref{lm:hard_instance_exists}, its sample complexity is at least
    $\wtilde{\Omega}(nk)$. 
    Therefore,
    $\wtilde{O}(n^{\chi} k^{\psi} \polylog(k)) \ge \wtilde{\Omega}(nk)$.
    Since this holds for all $n, k$ we must have $\chi \ge 1$ and $\psi \ge 1$ as claimed.
\end{proof}

\subsection{Proof of \Cref{lm:replic_to_pg}}
\label{app:replic_to_pg}
We use the following two lemmas \cite{bun2023stability}; the first lemma relates replicability to the concept of \emph{sample perfectly generalizing} algorithms, while the second shows that any sample perfectly generalizing is also perfectly generalizing with slightly worse parameters.
\begin{lemma}[Theorem 3.19 in \cite{bun2023stability}]\label{lm:replic_to_sample_pg}
    Fix sufficiently small $\delta, \beta> 0$ and $\eps \in (0, 1]$.
      Every $(0.01, 0.01)$-replicable algorithm $\cA$ with sample complexity $n$ and failure probability $\beta^2$
      can be converted into a $(2\delta, \eps, 2\delta)$-sample perfectly generalizing algorithm $\cA'$ with failure probability
      $O(\delta) + \beta \log(1/\delta)$ and sample complexity
      \begin{math}
        \frac{n \log(1/\eps)}{\eps^2}\polylog(1/\delta)
      \end{math}
      .
\end{lemma}
\begin{lemma}[Lemma 3.6 in \cite{bun2023stability}]\label{lm:sample_pg_is_pg}
  Any $(\gamma, \eps, \delta)$-sample perfectly generalizing algorithm is also
  $(\sqrt{\gamma}, \eps, \delta + \sqrt{\gamma})$-perfectly generalizing.
\end{lemma}

\Cref{lm:replic_to_sample_pg} considers the notion of $(\eta, \nu)$-replicable algorithms which is slightly different from standard $\rho$-replicability. As shown by \cite{impagliazzo2022reproducibility} however (see Appendix A and definition A.1 in their paper), any $\rho$-replicable algorithm is also $(\rho/\nu, \nu)$-replicable for any $\nu \in [\rho, 1)$. In particular, this means that a $0.0001$-replicable algorithm is
also $(0.01, 0.01)$-replicable.

\begin{proof}[Proof of \Cref{lm:replic_to_pg}]
  Set $\delta' = \delta^2/8$ and $\beta'=\sqrt{\beta}$. 
  We invoke
  \Cref{lm:replic_to_sample_pg} with parameters $\delta',  \eps$ to transform the replicable algorithm
 into a 
$(\delta^2/4, \eps, \delta^2/4)$-sample PG algorithm with failure probability $O(\delta^2) + \sqrt{\beta}\log(8/\delta^2)$ and sample complexity
  $\frac{n\log(1/\eps)}{\eps^2} \polylog(1/\delta)$.
  Note that the preconditions of the lemma hold; the $0.0001$-replicable algorithm is $(0.01, 0.01)$-replicable  as explained above and, if $\delta, \beta$ are sufficiently small, then $\delta', \beta'$ can be made sufficiently small as well. 
  By \Cref{lm:sample_pg_is_pg},
  the algorithm is $(\sqrt{\delta^2/4}, \eps, \delta^2/4 + \sqrt{\delta^2/4})$-PG which means it is $(\delta, \eps, \delta)$-PG, finishing the proof.
\end{proof}

\newpage

\bibliographystyle{alpha}
\bibliography{ref}

\newpage
\appendix
\section{Challenge of adaptivity in perfect generalization}
\subsection{Issues in the proof of adaptive composition for PG}
\label{sec:pg_error_explain}
The fact that one can obtain an adaptive composition theorem for perfect generalization
(Theorem~4.8 in \cite{bassily2016typical} and our extension in \Cref{lm:pg_compose_het} to the heterogeneous case)
is at first surprising given that the related notion
of Max-information under product distributions does not compose as shown by \cite{RRST16}.

We briefly recall the relevant definitions.
Given jointly distributed random variables $X$ and $Z$,
we write $X \otimes Z$ for an independent copy, i.e., a pair drawn from the product of the marginals of $X$ and $Z$.
The \emph{$\beta$-approximate max-information} between $X$ and $Z$ is defined as
$
I_\infty^{\beta}(X ; Z) = \log \sup_{\cO} \frac{\Pr{(X, Z) \in \cO} - \beta}{\Pr{X \otimes Z \in \cO}}
$
where the supremum is over all events $\cO$ with $\Pr{(X,Z) \in \cO} > \beta$.
An algorithm $\cA: \cX^n \to \cY$ has \emph{$\beta$-approximate max-information of $k$ under product distributions}
if $I_\infty^{\beta}(X ; \cA(X)) \le k$ for every product distribution $\cP$ over $\cX$ when $X \sim \cP^n$.

\paragraph{Why max-information fails to compose.}
Theorem~4.1 in \cite{RRST16} shows that max-information under product distributions does not compose.
They construct algorithms
$\cA: \cX^n \to \{0, 1\}^r$ with $r \in \Theta(\log(n))$ and $\cB: \cX^n \times \{0, 1\}^r \to (\cX \cup \{\perp\})$,
where $\cX = \{0,1\}$ and $\perp$ is a special error token, such that the following holds.
Let $X$ be sampled uniformly from $\cX^n$.
The algorithm $\cA$ has small max-information under product distributions and,
for any fixed $z \in \{0, 1\}^r$,
$\cB(\cdot, z)$ also has small max-information.
However, $\cB(X, \cA(X))$ recovers $X$ with high probability,
meaning the composition has very high max-information.
The key reason max-information under product distribution does not compose is that once we condition on the output of
$\cA$, the distribution of $X$ is no longer a product distribution.
As such, even though $\cB(\cdot, z)$ has low max-information under product distributions for any fixed
$z$,
this does not yield a bound for the max-information of $\cB(X, \cA(X))$ for random $X$.

\paragraph{How perfect generalization sidesteps this issue.}
Intuitively, the same issue should arise for perfect generalization:
conditioned on $\cA(X)$, the distribution of $X$ is no longer i.i.d.
The key idea that makes composition work for PG is to define the oracle transcript
$\widetilde{Z}$ \emph{independently} of $X$, using the oracle's own prefix rather than the algorithm's.
To see why this helps, consider
the above example of the two algorithms $\cA$ and $\cB$.
Define $Z=\cA(X)$ and let $\widetilde{Z}$ denote a sample from the oracle corresponding to
$\cA(\cdot)$.
The key problem is that we cannot immediately claim $\cB(X, Z) \approx W(Z)$ because,
conditioned on $Z$, the distribution of $X$ is no longer i.i.d.
However we can still show $\cB(X, Z) \approx W(Z)$ using a more involved argument.
Concretely, observe that $X$ and $\widetilde{Z}$ are independent.
As such, $\cB(X, \widetilde{Z})$ has the desired perfect generalization guarantees:
with high probability over the randomness of $X$, the distribution of $\cB(X, \widetilde{Z})$ is close
to some fixed $W(\widetilde{Z})$.
Letting $\hat{C}$ denote the set of all pairs $(x, z)$ such that
$\cB(x, z) \approx W(z)$, this means that
$(X, \widetilde{Z}) \in \hat{C}$ with high probability.
Since $\cA(\cdot)$ is PG, this implies that
$(X, Z) \in \hat{C}$ with high probability as well.
Therefore,
$\cB(X, Z) \approx W(Z)$ with high probability.
The proof of \Cref{lm:pg_compose_het} follows this approach; we refer to the high-level sketch
after the statement of \Cref{lm:pg_compose_het} for a more detailed explanation.

\paragraph{An issue in the existing proof and its fix.}
The proof of \cite{bassily2016typical} as written has a small issue related to the above discussion: it implicitly assumes
that $\cA(X)$ and $X$ are independent.
For simplicity, we discuss the proof of Theorem~4.2 in their paper, which handles
the composition of pure typically stable algorithms (i.e., $\delta=0$); the same issue exists for the main result.
The key issue is in the definition of $W_i(P)$, which is defined
as the oracle corresponding to $\cA_i(\cdot, Z^{i-1})$ (see the discussion after Theorem~4.2 on page~10).
The random variables $\widetilde{Z}_i$ are then sampled from the distribution $W_i(P)$.
Since $W_i(P)$ is defined in terms of $Z^{i-1}$, which itself depends on the input $X$,
the variables $\widetilde{Z}_i$ depend on $X$ as well. This is problematic because at multiple points (e.g., right
after Equation~14, as well as right before the last equation on page~13) the paper assumes that $X$ and $\widetilde{Z}^i$ are independent.

The fix is simple: define the oracle transcript recursively using its own prefix rather than the algorithm's.
Concretely, it suffices to define $W_i(P)$ as the oracle corresponding to
$\cA_i(\cdot, \widetilde{Z}^{i-1})$ where $\widetilde{Z}_i \sim W_i(P)$. In other words,
we first define $W_1(P)$ as the oracle corresponding to $\cA_1(\cdot)$. Note that,
by definition, we have
$\cA_1(X) \approx_\eps W_1(P)$ with probability at least $1-\nu$,
where we write $U \approx_{\eps} V$ to denote
$\abs{\ln\rbr{\frac{\Pr{U=x}}{\Pr{V=x}}}} \le \eps$ for all $x$.
We then sample $\widetilde{Z}_1$ from $W_1(P)$ and define
$W_2(P)$ as
the oracle corresponding to $\cA_2(\cdot, \widetilde{Z}^{1})$, and repeat.
Since each $\widetilde{Z}_i$ is defined only in terms of $\widetilde{Z}^{i-1}$ (and not $X$), the variables $\widetilde{Z}^k$ are independent of $X$.
We now verify that the remaining steps of their proof go through unchanged.

Lemma~4.3 (in their paper) is unaffected since it does not involve $\widetilde{Z}$.
Equation~15 holds as well: while the equation follows from applying Lemma~4.3 to
$W^j(P)$, the lemma does not at any point use the fact that $\cA_j(\cdot, \widetilde{Z}^{j-1}) \approx W_j(P)$.
Indeed, the way in which closeness of the distributions is leveraged is by
bounding $\Pr{(X, \widetilde{Z}^j) \notin \widetilde{G}_j}$ under the assumption that $(X, \widetilde{Z}^{j-1}) \in \widehat{C}_j$;
the latter assumption is what ensures the closeness of distributions.
This assumption holds for $(X, \widetilde{Z}^{j-1})$ because of the following.
For any fixed $z^{j-1}$, we have
$\cA_j(X, z^{j-1}) \approx_{\eps} W_j(P)$ with probability at least
$1-\nu$.
While $\widetilde{Z}^{j-1}$ is random, since $X$ is independent of $\widetilde{Z}^{j-1}$, we have
$\cA_j(X, \widetilde{Z}^{j-1}) \approx_{\eps} W_j(P)$ with probability at least $1-\nu$ as well.
It follows that, for any fixed $j$, $\Pr{(X, \widetilde{Z}^{j-1}) \notin C_{j}} \le \nu$, which by a union bound
implies $\Pr{(X, \widetilde{Z}^{j-1}) \notin \widehat{C}_{j}} \le j\nu$.

\end{document}